\documentclass[13pt]{article}
\usepackage{latexsym}
\usepackage{geometry}
\usepackage{graphicx}
\usepackage{amsmath, amssymb, amsthm}
\usepackage{booktabs}
\usepackage{algorithm}
\usepackage{algorithmic}

\usepackage[utf8]{inputenc} % allow utf-8 input
\usepackage[T1]{fontenc}    % use 8-bit T1 fonts

\usepackage{url}
\usepackage{natbib}
\usepackage{appendix}

\usepackage{amsmath}
\usepackage{amssymb}
\usepackage{mathtools}
\usepackage{amsthm}
\usepackage{placeins}
\usepackage{pifont}
\usepackage{multirow}
\usepackage{subcaption}

\usepackage{amsfonts}
\usepackage{multirow}
\usepackage{multicol}

\usepackage{color}

\usepackage{xcolor,colortbl}
\definecolor{LightCyan}{rgb}{0.88,1,1}

\usepackage{hyperref}
\usepackage{tcolorbox}

\newcommand{\R}{\mathbb{R}}
\newcommand{\E}{\mathbb{E}}
\newcommand{\norm}[1]{\lVert #1 \rVert}

\newtheorem{theorem}{Theorem}
\newtheorem{lemma}{Lemma}

\newtheorem{assumption}{Assumption}
\newtheorem{remark}{Remark}

\begin{document}
	
\title{ Federated Compositional Muon Optimizer for \\ Matrix-Wise Models }

\author{
Wang Yan\thanks{Wang Yan is with College of Computer Science and Technology, Nanjing University of Aeronautics and Astronautics, Nanjing, China.}, \	
Feihu Huang\thanks{Feihu Huang is with College of Computer Science and Technology,
Nanjing University of Aeronautics and Astronautics, Nanjing, China;
and also with MIIT Key Laboratory of Pattern Analysis and Machine Intelligence, Nanjing, China. Email: huangfeihu2018@gmail.com}
 }

\date{}
\maketitle

\begin{abstract}
 Muon, a more recently developed optimizer, is useful for matrix-wise models in AI areas.
 Although many works have studied Muon and its variants, these methods are still not particularly well-suited for hierarchical structured problems. 
 To fill this gap, we propose an effective federated compositional Muon (FedCoMuon) optimizer to solve distributed matrix-wise compositional optimization problems. Specifically, our FedCoMuon optimizer builds on 
 compositional gradient tracking and orthogonalized momentum. 
 Moreover, we propose a variance reduced variant of FedCoMuon (FedCoMuon-VR)
 based on a momentum-based variance reduced technique. 
 In theory, we analyze the convergence properties of our algorithms under the non-i.i.d. and non-convex settings. 
 In particular, we prove that our FedCoMuon-VR obtains a lower sample complexity of $O(\epsilon^{-3})$ for finding an $\epsilon$-stationary solution than the existing FedMuon algorithms.   
 Extensive numerical experiments on robust federated learning and task-distributed risk-sensitive meta learning show that our proposed methods are competitive with existing compositional baselines and achieve the best reported accuracy in several settings.
\end{abstract}

\section{Introduction}
\label{sec:intro}

Federated learning (FL)~\citep{mcmahan2017,kairouz2019advances,li2021survey} 
enables multiple clients to collaboratively train a shared model without 
centralizing their local data, making it attractive for privacy-sensitive 
and resource-constrained applications such as mobile intelligence, healthcare, 
recommendation systems, and edge learning. The first federated learning algorithm, i.e., FedAvg~\citep{mcmahan2017},
established the standard paradigm of alternating local stochastic updates
with periodic model aggregation, while subsequent methods such as 
FedProx~\citep{li2020federated}, SCAFFOLD~\citep{karimireddy2020}, and 
FedOpt~\citep{reddi2020adaptive} further improved federated training under
client heterogeneity and limited communication. Despite their success, 
these methods were designed for standard federated empirical risk minimization, 
where the global objective is typically a single-level weighted average of client losses. 

Beyond this standard single-level formulation, many federated applications 
involve more complex nested objectives~\citep{yang2026compositional}, especially when the learned model is 
expected to perform reliably across heterogeneous or adverse client distributions. 
In distributionally robust FL, for example, a nonlinear risk measure is 
applied to client-level expected losses, resulting in a nested stochastic 
compositional objective~\citep{huang2021compositional}. Similar nested 
structures arise in the task-distributed meta learning problem~\citep{huang2021compositional,wang2021memory} 
and federated reinforcement learning~\citep{jin2022fedrl,tarzanagh2022fednest}. 
Since the outer gradient depends on an inner stochastic expectation, 
directly applying local SGD generally yields biased gradient estimates. 
Several federated compositional optimization methods correct 
this bias, including ComFedL~\citep{huang2021compositional}, 
Local-MOML~\citep{wang2021memory}, Local-SCGDM~\citep{gao2022convergence}, 
and FedNest~\citep{tarzanagh2022fednest}.
Although these methods address the nested stochastic structure, their 
algorithmic formulations do not explicitly exploit the matrix structure of model parameters.

In fact, modern neural networks such as Convolutional Neural Networks (CNN)~\citep{lecun2015deep} and Transformers~\citep{vaswani2017attention}
are matrix-wise models with numerous matrix parameter blocks in attention and linear layers.
Since training such models is computationally expensive, designing optimizers that 
can effectively exploit their matrix structure has attracted increasing attention. 
Muon~\citep{jordan2024muon} updates the matrix parameter blocks of matrix-wise models using an
orthogonalized momentum direction, which is efficiently computed through 
Newton--Schulz iterations. It has demonstrated competitive empirical 
performance across a variety of large language models (LLMs)~\citep{liu2025scalable}. 
More recently, several studies have 
extended Muon to federated learning~\citep{takezawa2025fedmuon,liu2025fedmuon,zhang2025provable}. 
However, the existing Muon-type federated algorithms are primarily developed for 
single-level stochastic objectives, and it remains unclear how Muon 
can be 
effectively applied to federated stochastic compositional optimization, where the nested expectation introduces additional gradient estimation errors under heterogeneous client data. This raises a natural question: 
\textit{Could we develop effective Muon-type federated algorithms for
	matrix-wise compositional optimization under heterogeneous data?}

In this paper, we provide an affirmative answer to the above question and develop two effective Muon-type federated learning algorithms (i.e., FedCoMuon and FedCoMuon-VR) 
to solve distributed matrix-wise stochastic compositional problems. In summary, our main contributions are summarized as follows:

\begin{itemize}
	\item We develop a class of effective Muon-type federated compositional algorithms (i.e., FedCoMuon and FedCoMuon-VR) for 
	distributed matrix-wise stochastic compositional optimization. 
	Specifically, our FedCoMuon combines compositional gradient tracking with orthogonalized matrix momentum, 
	while FedCoMuon-VR further incorporates a momentum-based variance-reduction technique~\citep{cutkosky2019momentum}.
	
	\item We provide a solid convergence analysis for the proposed algorithms under non-convex 
	and non-i.i.d. settings. To find an $\epsilon$-stationary point, our FedCoMuon requires 
	a sample complexity of $\mathcal{O}(\epsilon^{-4})$ and a communication 
	complexity of $\mathcal{O}(\epsilon^{-3})$, while our FedCoMuon-VR has a lower sample complexity of  $\mathcal{O}(\epsilon^{-3})$ while retaining 
	the same communication complexity. In particular, our FedCoMuon-VR has a lower sample complexity than the existing Federated Muon algorithms~\citep{takezawa2025fedmuon,liu2025fedmuon,zhang2025provable,compressedgluon}.
	
	\item Experiments on the task-distributed meta learning problem and robust federated learning 
	demonstrate effectiveness of the proposed algorithms.
\end{itemize}

\section{Related Work}
\label{sec:related_work}

\subsection{Federated Learning}
Federated learning~\citep{mcmahan2017} is a popular distributed learning paradigm in machine learning. 
FedAvg~\citep{mcmahan2017} is the first federated learning algorithm, which reduces communication frequency by allowing 
each client to perform multiple local SGD steps before server aggregation. 
Under heterogeneous client data, however, repeated local updates may drift 
from the global descent direction and slow convergence. Subsequently, some effective variants of FedAvg have been developed. For example, FedProx~\citep{li2020federated} 
and SCAFFOLD~\citep{karimireddy2020} mitigate this issue through proximal 
regularization and control variances, respectively. Meanwhile, momentum-based, adaptive, and variance-reduced extensions~\citep{khanduri2021,cheng2024,padamfed2025} further improve
convergence and communication efficiency. 

\subsection{Compositional Optimization}
Compositional optimization~\citep{wang2017stochastic} is a class of effective nested structural optimization problems in machine learning. 
Since compositional optimization is widely used in many machine learning tasks such as robust learning and 
federated learning, many algorithms~\citep{wang2017stochastic,ghadimi2020single,chen2021solving,zhang2019multi} have recently been developed. 
For example, SCGD~\citep{wang2017stochastic} 
controls this bias by tracking the inner mapping with a moving average. 
Subsequently, several accelerated algorithms based on momentum or
variance-reduction techniques have been developed for non-convex
stochastic compositional optimization~\citep{
	ghadimi2020single,chen2021solving,zhang2019multi,jiang2022optimal}.

To solve distributed compositional optimization, 
some federated compositional algorithms have been developed. 
For example, 
ComFedL~\citep{huang2021compositional} introduced a federated
compositional framework for distributionally robust learning
and meta learning. Subsequently,   Local-MOML~\citep{wang2021memory} and
Local-SCGDM~\citep{gao2022convergence} improved sample and communication efficiency through local updates and 
momentum-based tracking.  FedDRO~\citep{khanduri2023feddro} 
further studies the interaction between compositional gradient 
bias and client heterogeneity in distributionally robust federated 
learning.  \cite{tarzanagh2022fednest,huang2026faster,gao2024doubly} proposed variance-reduced federated compositional algorithms for distributed non-convex stochastic composition optimization based on variance-reduction techniques. 
In fact, the federated setting is more challenging because compositional gradient 
bias and client drift should be controlled simultaneously. 

\subsection{Muon-Based Optimization}
In the last two years, Muon~\citep{jordan2024muon} has emerged as a promising optimizer, which directly updates the matrix parameter blocks of matrix-wise models. 
Specifically, it uses orthogonalization of the matrix-valued momentum by a few Newton--Schulz 
iterations, which can be interpreted as an LMO-based optimizer over a 
spectral-norm ball~\citep{pethick2025}. More recently, its convergence properties have been
studied in non-convex stochastic optimization~\citep{shen2025,li2025note,riabinin2025gluon,kovalev2025understanding,kim2026convergence}. Subsequently, its variance-reduced variants~\citep{sfyraki2025lions,huang2025limuon,qian2025muon,chang2025convergence} have also been studied. 

More recently, several works~\cite{takezawa2025fedmuon,liu2025fedmuon,zhang2025provable} have begun to study Federated Muon (i.e., FedMuon) algorithms for distributed matrix-wise optimization. 
Specifically, \cite{takezawa2025fedmuon} proposed the FedMuon algorithm based on bias-correction mechanism to 
address the bias induced by local linear minimization oracles.  \cite{liu2025fedmuon} presented the FedMuon algorithm by using  
momentum aggregation and local-global alignment to 
mitigate client drift under heterogeneous data. 
Meanwhile, \cite{zhang2025provable} developed the FedMuon algorithm with hyper-parameter 
choices independent of problem-specific constants, and established its convergence under 
both bounded-variance and heavy-tailed stochastic noise. Subsequently, \cite{compressedgluon} proposed communication-efficient federated Gluon algorithm based on gradient compression and error feedback. 
However, these methods focus on standard single-level objectives and do not 
provide algorithms or convergence guarantees for stochastic compositional optimization under heterogeneous client data.

\paragraph{Notation.}
Let $[K]=\{1,2,\ldots,K\}$ denote the set of clients. 
$\|\cdot\|$ denotes Euclidean and 
spectral norm for vector and matrix, respectively. For matrices
$A,B\in\mathbb R^{m\times n}$, we use
$\langle A,B\rangle=\operatorname{tr}(A^\top B)$ to denote the Frobenius
inner product, and use $\|A\|_F$ to denote the Frobenius norm. $A\otimes B$ denotes 
the Kronecker product of matrices $A$ and $B$.

\section{Preliminaries}
\label{sec:preliminaries}
In this paper, we study the Muon optimizer to solve the following distributed
matrix-wise compositional optimization problem:
\begin{align}\label{eq:main-problem}
	\min_{W\in\R^{m\times n}}
	\frac{1}{K}\sum_{k=1}^K
	\E_{\zeta\sim\mathcal D_f^k}
	\left[
	f^k
	\left(
	\E_{\xi\sim\mathcal D_g^k}[g^k(W;\xi)];
	\zeta
	\right)
	\right],
\end{align}
where the inner and outer expected mappings are defined as
$g^k(W)\triangleq \E_{\xi\sim\mathcal D_g^k}[g^k(W;\xi)]:
\R^{m\times n}\to\R^d$ and
$f^k(y)\triangleq \E_{\zeta\sim\mathcal D_f^k}[f^k(y;\zeta)]:
\R^d\to\R$, respectively. We write
$F(W)=K^{-1}\sum_{k=1}^K F^k(W)$, where
$F^k(W)=f^k(g^k(W))$. The inner and outer data distributions,
$\mathcal D_g^k$ and $\mathcal D_f^k$, may differ across clients. This formulation therefore captures both the nested compositional structure and the data heterogeneity inherent in federated settings.

Next, we introduce several mild assumptions for Problem~\eqref{eq:main-problem}.

\begin{assumption}[Lower-bounded objective]\label{assm:lower}
	The global objective $F(W)$ has a lower bound, i.e.,
	$F_*:=\inf_{W\in\R^{m\times n}}F(W)>-\infty$.
\end{assumption}

\begin{assumption}[Unbiased stochastic oracles and bounded variances]
	\label{assm:oracle}
	For any client $k\in[K]$, the stochastic oracles are unbiased:
	\begin{align}
		\E_{\xi^k}[g^k(W;\xi^k)]=g^k(W), \ \E_{\xi^k}[\nabla g^k(W;\xi^k)]=\nabla g^k(W), \ \E_{\zeta^k}[\nabla_y f^k(y;\zeta^k)]=\nabla_y f^k(y). \nonumber 
	\end{align}
	The samples are independent across clients and iterations, and $\xi^k$ is
	independent of $\zeta^k$. In addition, there exist constants
	$\sigma_g,\sigma_{\nabla g},\sigma_f>0$ such that, for any
	$W\in\R^{m\times n}$ and $y\in\R^d$,
	\begin{align}
		&\E_{\xi^k}\|g^k(W;\xi^k)-g^k(W)\|^2\le \sigma_g^2, \nonumber \\
		&\E_{\xi^k}\|\nabla g^k(W;\xi^k)-\nabla g^k(W)\|_F^2
		\le \sigma_{\nabla g}^2, \nonumber \\ 
		&\E_{\zeta^k}\|\nabla_y f^k(y;\zeta^k)-\nabla_y f^k(y)\|^2\le \sigma_f^2. \nonumber 
	\end{align}
\end{assumption}

\begin{assumption}[Bounded gradient moments]\label{assm:bounded-grad}
	For any client $k\in[K]$, there exist constants $C_g,C_f>0$ such that, for any
	$W\in\R^{m\times n}$ and $y\in\R^d$, the stochastic Jacobians and outer
	gradients have bounded second moments:
	\begin{align}
		\E_{\xi^k}\|\nabla g^k(W;\xi^k)\|_F^2\le C_g^2,  \ \E_{\zeta^k}\|\nabla_y f^k(y;\zeta^k)\|^2\le C_f^2. \nonumber
	\end{align}
\end{assumption}

Assumptions~\ref{assm:lower}--\ref{assm:bounded-grad} have been commonly used in the 
convergence analysis of stochastic compositional optimization algorithms~\citep{tarzanagh2022fednest,gao2022convergence,jiang2022optimal,gao2024doubly}.

\begin{algorithm}[t]
	\caption{FedCoMuon Algorithm}
	\label{alg:fedcomuon}
	\begin{algorithmic}[1]
		\STATE \textbf{Input:} $\eta>0$, $\alpha\in[0,1)$, $\beta\in[0,1)$, and $\tau>0$.
		\STATE \begin{minipage}[t]{0.92\linewidth}
			\textbf{Initialize:} For all $k\in[K]$, set
			$W_0^k=W_0\in\R^{m\times n}$,
			$u_0^k=g^k(W_0^k;\xi_0^k)$, and
			$M_0^k=\nabla g^k(W_0^k;\xi_0^k)
			\bigl(\nabla_y f^k(u_0^k;\zeta_0^k)\otimes I_n\bigr)$ for
			$\xi_0^k\sim\mathcal D_g^k$ and $\zeta_0^k\sim\mathcal D_f^k$;
		\end{minipage}
		\FOR{$t = 0, 1, \dots, T-1$}
		\FOR{each client $k \in [K]$ (\textbf{in parallel})}
		\STATE \begin{minipage}[t]{0.92\linewidth}
			$(U_t^k,\Sigma_t^k,V_t^k)=\mathrm{SVD}(M_t^k)$;
			\enspace // Orthonormalize
			$M_t^k$ with the Newton--Schulz approach
		\end{minipage}
		\STATE $W_{t+1}^k=W_t^k-\eta U_t^k(V_t^k)^\top$;
		\STATE Draw two independent samples $\xi_{t+1}^k\sim\mathcal D_g^k$ and $\zeta_{t+1}^k\sim\mathcal D_f^k$;
		\STATE $u_{t+1}^k=\alpha g^k(W_{t+1}^k;\xi_{t+1}^k)+(1-\alpha)u_t^k$;
		\STATE $Z_{t+1}^k=\nabla g^k(W_{t+1}^k;\xi_{t+1}^k)
		\bigl(\nabla_y f^k(u_{t+1}^k;\zeta_{t+1}^k)\otimes I_n\bigr)$;
		\STATE $M_{t+1}^k=\beta Z_{t+1}^k+(1-\beta)M_t^k$;
		\ENDFOR
		\IF{$\operatorname{mod}(t+1,\tau)=0$}
		\STATE Receive $\{W_{t+1}^k,M_{t+1}^k\}_{k=1}^K$ from all clients;
		\STATE $\bar W_{t+1}=K^{-1}\sum_k W_{t+1}^k$;
		\STATE $\bar M_{t+1}=K^{-1}\sum_k M_{t+1}^k$;
		\STATE Send $\bar W_{t+1}$ and $\bar M_{t+1}$ to each client;
		\ENDIF
		\ENDFOR
		\STATE \textbf{Output:} Sampling uniformly from $\{\bar W_t\}_{t=0}^{T-1}$ (in theory), and $\bar W_T$ (in practice).
	\end{algorithmic}
\end{algorithm}

\section{Federated Compositional Muon Methods}
\label{sec:algorithms}

In this section, we propose a class of efficient federated compositional Muon algorithms (i.e., FedCoMuon and FedCoMuon-VR) for large matrix-valued models. Specifically, our FedCoMuon builds on compositional gradient tracking and Muon, while our FedCoMuon-VR further builds on a momentum-based variance-reduction technique.

\begin{algorithm}[!t]
	\caption{FedCoMuon-VR Algorithm}
	\label{alg:fedcomuon-vr}
	\begin{algorithmic}[1]
		\STATE \textbf{Input:} $\eta>0$, $\alpha,\beta,\gamma,\rho\in[0,1)$,
		$\tau>0$, and $b>0$.
		\STATE \begin{minipage}[t]{0.92\linewidth}
			\textbf{Initialize:} For all $k\in[K]$, set
			$W_0^k=W_0\in\R^{m\times n}$, and draw
			$b$ i.i.d.\ samples $\{\xi_{0,j}^k\}_{j=1}^{b}$ from $\mathcal D_g^k$ and
			$b$ i.i.d.\ samples $\{\zeta_{0,j}^k\}_{j=1}^{b}$ from $\mathcal D_f^k$, and set
			$u_0^k=\frac1b\sum_{j=1}^{b}g^k(W_0^k;\xi_{0,j}^k)$,
			$H_0^k=\frac1b\sum_{j=1}^{b}\nabla g^k(W_0^k;\xi_{0,j}^k)$,
			and
			$v_0^k=\frac1b\sum_{j=1}^{b}\nabla_y f^k(u_0^k;\zeta_{0,j}^k)$,
			$M_0^k=H_0^k(v_0^k\otimes I_n)$.
			Set $\bar M_0=K^{-1}\sum_{k=1}^K M_0^k$ and $M_0^k=\bar M_0$ for all
			$k\in[K]$;
		\end{minipage}
		\FOR{$t=0,1,\ldots,T-1$}
		\FOR{each client $k\in[K]$ \textbf{(in parallel)}}
		\STATE $(U_t^k,\Sigma_t^k,V_t^k)=\operatorname{SVD}(M_t^k)$;
		\STATE $W_{t+1}^k=W_t^k-\eta U_t^k(V_t^k)^\top$;
		\ENDFOR
		\IF{$\operatorname{mod}(t+1,\tau)=0$}
		\STATE Receive $\{W_{t+1}^k\}_{k=1}^K$ from all clients;
		\STATE \begin{minipage}[t]{0.92\linewidth}
			$\bar W_{t+1}=\frac{1}{K}\sum_{k=1}^K W_{t+1}^k$;
		\end{minipage}
		\STATE Send $\bar W_{t+1}$ to each client;
		\ENDIF
		\FOR{each client $k\in[K]$ \textbf{(in parallel)}}
		\STATE Draw $\xi_{t+1}^k\sim\mathcal D_g^k$ and
		$\zeta_{t+1}^k\sim\mathcal D_f^k$;
		\STATE \begin{minipage}[t]{0.92\linewidth}
			$u_{t+1}^k=g^k(W_{t+1}^k;\xi_{t+1}^k)
			+(1-\alpha)\bigl(u_t^k-g^k(W_t^k;\xi_{t+1}^k)\bigr)$;
		\end{minipage}
		\STATE \begin{minipage}[t]{0.92\linewidth}
			$v_{t+1}^k=\Pi_{C_f}\bigl[\nabla_y f^k(u_{t+1}^k;\zeta_{t+1}^k)
			+(1-\beta)\bigl(v_t^k-\nabla_y f^k(u_t^k;\zeta_{t+1}^k)\bigr)\bigr]$;
		\end{minipage}
		\STATE \begin{minipage}[t]{0.92\linewidth}
			$H_{t+1}^k=\Pi_{C_g}\bigl[\nabla g^k(W_{t+1}^k;\xi_{t+1}^k)
			+(1-\gamma)\bigl(H_t^k-\nabla g^k(W_t^k;\xi_{t+1}^k)\bigr)\bigr]$;
		\end{minipage}
		\STATE $M_{t+1}^k=(1-\rho)M_t^k+\rho H_{t+1}^k(v_{t+1}^k\otimes I_n)$;
		\ENDFOR
		\IF{$\operatorname{mod}(t+1,\tau)=0$}
		\STATE Receive $\{M_{t+1}^k\}_{k=1}^K$ from all clients;
		\STATE \begin{minipage}[t]{0.92\linewidth}
			$\bar M_{t+1}=\frac{1}{K}\sum_{k=1}^K M_{t+1}^k$;
		\end{minipage}
		\STATE Send $\bar M_{t+1}$ to each client;
		\ENDIF
		\ENDFOR
		\STATE \textbf{Output:} Sampling uniformly from $\{\bar W_t\}_{t=0}^{T-1}$ (in
		theory), and $\bar W_T$ (in practice).
	\end{algorithmic}
\end{algorithm}

\subsection{FedCoMuon Algorithm}
\label{subsec:fedcomuon}
In this subsection, we provide an efficient 
federated compositional Muon (i.e., FedCoMuon) algorithm to solve the problem~(\ref{eq:main-problem}). 
Algorithm~\ref{alg:fedcomuon} shows 
the algorithmic framework for FedCoMuon.

In our FedCoMuon algorithm, at
the $t$-th iteration, each client $k$ uses its local stochastic samples to
update the matrix momentum as
\begin{equation}
	\label{eq:fedcomuon-momentum}
	M_{t+1}^k=\beta Z_{t+1}^k+(1-\beta)M_t^k ,
\end{equation}
where $\beta\in[0,1)$ is the momentum parameter. Here $Z_{t+1}^k$ denotes the
stochastic compositional gradient, defined as
\begin{equation}
	\label{eq:fedcomuon-stochastic-gradient}
	Z_{t+1}^k
	=
	\nabla g^k(W_{t+1}^k;\xi_{t+1}^k)
	\bigl(\nabla_y f^k(u_{t+1}^k;\zeta_{t+1}^k)\otimes I_n\bigr),
\end{equation}
where $\xi_{t+1}^k$ and $\zeta_{t+1}^k$ are fresh stochastic samples. In the above
estimator, $u_{t+1}^k$ tracks the inner function value $g^k(W_{t+1}^k)$,
thereby enabling $Z_{t+1}^k$ to form a matrix-form stochastic
compositional gradient estimate. The parameter $\alpha\in[0,1)$ controls the
moving-average update of $u_{t+1}^k$.

Following the Muon algorithm~\citep{jordan2024muon}, we use
Newton--Schulz iterations instead of an expensive exact SVD to obtain
the orthogonalized momentum direction. 

Every $\tau$ local iterations, clients communicate with the server. In each
communication round, the server receives
$\{W_{t+1}^k,M_{t+1}^k\}_{k=1}^K$, averages them to obtain $\bar W_{t+1}$ and
$\bar M_{t+1}$, and sends the averaged variables back to all clients.

\subsection{FedCoMuon-VR Algorithm}
\label{subsec:fedcomuon-vr}

In this subsection, we provide an efficient 
federated compositional variance-reduced Muon (i.e., FedCoMuon-VR) algorithm to solve the problem~(\ref{eq:main-problem}) based on the momentum-based variance-reduction technique. 
Algorithm~\ref{alg:fedcomuon-vr} provides 
the algorithmic framework for FedCoMuon-VR.

In Algorithm~\ref{alg:fedcomuon-vr}, after obtaining
$W_{t+1}^k$, each client $k$ draws fresh samples
$\xi_{t+1}^k$ and $\zeta_{t+1}^k$, and recursively updates
$u_{t+1}^k$, $v_{t+1}^k$, and $H_{t+1}^k$ to track the inner function
value, the outer gradient, and the inner Jacobian, respectively. Each 
recursive estimator evaluates its function on the same fresh sample 
at two consecutive iterates, which reduces the stochastic estimation error.

Let
$\mathcal B_f:=\{v:\|v\|\le C_f\}$ and
$\mathcal B_g:=\{H:\|H\|_F\le C_g\}$, and let $\Pi_{C_f}$ and
$\Pi_{C_g}$ denote the Euclidean projections onto $\mathcal B_f$ and
$\mathcal B_g$, respectively. These projection operators keep the
outer-gradient estimator and the Jacobian estimator bounded.
The resulting estimators form the variance-reduced compositional gradient
estimate $H_{t+1}^k(v_{t+1}^k\otimes I_n)$, which is used to update the
matrix momentum with parameter $\rho\in[0,1)$. Every $\tau$ local
iterations, the server averages the local model parameters and momentum
matrices, and broadcasts the averaged variables back to all clients.

\section{Convergence Analysis}
\label{sec:theory}
In this section, we study the convergence properties of our FedCoMuon and FedCoMuon-VR algorithm under some mild assumptions. All related proofs are provided in the Appendix.

\subsection{Convergence Properties of Our FedCoMuon}

\begin{assumption}[Smoothness]\label{assm:smooth}
	For any client $k\in[K]$, the population mappings $g^k$ and $f^k$ are
	$L_g$- and $L_f$-smooth, respectively. That is, for any
	$W_1,W_2\in\R^{m\times n}$ and $y_1,y_2\in\R^d$,
	\begin{align}
		\|\nabla g^k(W_1)-\nabla g^k(W_2)\|_F
		&\le	L_g\|W_1-W_2\|_F, \ 
		\|\nabla_y f^k(y_1)-\nabla_y f^k(y_2)\|\le
		L_f\|y_1-y_2\|. \nonumber
	\end{align}
\end{assumption}

\begin{assumption}[Gradient heterogeneity]\label{assm:hetero}
	There exists a constant $\delta\ge0$ such that, for all $W\in\R^{m\times n}$,
	\begin{align}
		\frac1K\sum_{k=1}^K
		\|\nabla F^k(W)-\nabla F(W)\|_F^2
		\le
		\delta^2. \nonumber
	\end{align}
\end{assumption}

\begin{theorem}\label{thm:fedcomuon}
	Suppose Assumptions~\ref{assm:lower}--\ref{assm:hetero} hold. Let
	$\{W_t^k,u_t^k,M_t^k\}$ be generated by
	Algorithm~\ref{alg:fedcomuon} with $0<\alpha,\beta<1$, and
	$\alpha\tau\le1$. Then, we have
	{\small
		\begin{align}
			\frac1T\sum_{t=0}^{T-1}
			\mathbb E\|\nabla F(\bar W_t)\|_F
			&\le
			\frac{F(\bar W_0)-F_*}{\eta T} +
			\frac{2\sqrt n C_gL_f\sigma_g}{\beta T}
			+ 2\sqrt n
			\bigl(\frac{1}{\beta T}+2\beta\tau\bigr)
			\bigl(C_g\sigma_f+C_f\sigma_{\nabla g}\bigr)
			\notag\\
			&\quad
			+
			2\sqrt n C_gL_f(1+2\beta\tau)
			\biggl(
			\frac{4\sigma_g^2}{\alpha T}
			+\frac{86C_g^2n}{\alpha^2}\eta^2
			+86\alpha\sigma_g^2
			\biggr)^{1/2} \notag\\
			&\quad
			+\eta nL_F
			\bigl(
			\frac12+4\tau+\frac{6}{\beta}+8\beta\tau^2
			\bigr)+
			2\sqrt n\,\beta\tau\delta
			+	\frac{2\sqrt{n\beta}}{\sqrt K}
			\sqrt{C_g^2\sigma_f^2+C_f^2\sigma_{\nabla g}^2}.
			\label{eq:fedcomuon-bound}
		\end{align}
	}
\end{theorem}

\begin{remark}[Parameter choice and complexity]
	By choosing $\eta=T^{-3/4}$, $\alpha=\beta=T^{-1/2}$, and
	$\tau=T^{1/4}$, Theorem~\ref{thm:fedcomuon} yields
	$T^{-1}\sum_{t=0}^{T-1}\mathbb E\|\nabla F(\bar W_t)\|_F
	=O(T^{-1/4})$.
	Thus, achieving an $\epsilon$-stationary point requires
	$T=O(\epsilon^{-4})$. Since each local iteration uses a constant
	number of stochastic samples, the total sample complexity per client is
	$O(T)=O(\epsilon^{-4})$. Moreover, the number of communication rounds is
	$T/\tau=T^{3/4}=O(\epsilon^{-3})$.
\end{remark}

\subsection{Convergence Properties of Our FedCoMuon-VR}

\begin{assumption}[Mean-square sample smoothness]\label{assm:vr-sample-smooth}
	For any client $k\in[K]$, the sample mappings are mean-square Lipschitz. That is,
	for any $W_1,W_2\in\R^{m\times n}$ and $y_1,y_2\in\R^d$,
	\begin{align}
		&\E_{\xi^k}\|g^k(W_1;\xi^k)-g^k(W_2;\xi^k)\|^2
		\le C_g^2\|W_1-W_2\|_F^2, \nonumber \\
		&\E_{\xi^k}\|\nabla g^k(W_1;\xi^k)-\nabla g^k(W_2;\xi^k)\|_F^2
		\le L_g^2\|W_1-W_2\|_F^2, \nonumber \\
		&\E_{\zeta^k}\|\nabla_y f^k(y_1;\zeta^k)-\nabla_y f^k(y_2;\zeta^k)\|^2
		\le L_f^2\|y_1-y_2\|^2. \nonumber
	\end{align}
\end{assumption}

\begin{assumption}[Client heterogeneity]\label{assm:vr-hetero}
	There exist constants $\Delta_g,\Delta_{\nabla g},\Delta_f\ge0$ such that, for
	any $k,j\in[K]$, $W\in\R^{m\times n}$ and $y\in\R^d$,
	\begin{align}
		&\|\nabla_y f^k(y)-\nabla_y f^j(y)\|\le\Delta_f, \nonumber \\
		&\|\nabla g^k(W)-\nabla g^j(W)\|_F\le\Delta_{\nabla g}, \nonumber \\
		&\|g^k(W)-g^j(W)\|\le\Delta_g. \nonumber
	\end{align}
\end{assumption}
Assumption~\ref{assm:vr-hetero} imposes a common bounded-heterogeneity condition for compositional federated learning under non-i.i.d. setting ~\citep{tarzanagh2022fednest,huang2026faster}.

\begin{theorem}
	\label{thm:fedcomuon-vr}
	Suppose Assumptions~\ref{assm:lower}--\ref{assm:bounded-grad},
	\ref{assm:vr-sample-smooth}, 
	and~\ref{assm:vr-hetero} hold. Let
	$\{W_t^k,u_t^k,v_t^k,H_t^k,M_t^k\}$ be generated by
	Algorithm~\ref{alg:fedcomuon-vr} with
	$\eta>0$, $0<\alpha,\beta,\gamma,\rho<1$, $b\ge1$, and $\tau>0$.
	Then, for any $T\ge1$, we have
	{\small
		\begin{align}
			\frac1T\sum_{t=0}^{T-1}
			\mathbb E\|\nabla F(\bar W_t)\|_F
			&\le
			\frac{F(\bar W_0)-F_*}{\eta T}
			+\frac{L_Fn\eta}{2}
			+2\sqrt n
			\biggl[
			\frac{4C_gC_f}{\rho T}
			+\frac{L_F\sqrt n\,\eta}{\rho}
			\notag\\
			&\qquad
			+
			(2+\sqrt6\,\rho\tau)
			\biggl(
			C_f^2\sigma_{\nabla g}^2
			\left(\frac1{\gamma Tb}+2\gamma\right)
			+
			2C_g^2\sigma_f^2
			\Bigl(\frac1{\beta Tb}
			\notag\\
			&\qquad
			+2\beta\Bigr)
			+
			2C_g^2L_f^2\sigma_g^2
			\left[
			\frac1{\alpha Tb}
			+2\alpha
			+
			\frac{8\alpha}{\beta Tb}
			+
			\frac{8\alpha^2}{\beta}(1+2\alpha)
			\right]
			\notag\\
			&\qquad
			+
			4n\eta^2\tau
			\Bigl(
			\frac{2C_g^4L_f^2}{\alpha}
			+
			\frac{4C_g^4L_f^2}{\beta}
			+
			\frac{16C_g^4L_f^2\alpha}{\beta}
			\notag\\
			&\qquad
			+
			\frac{C_f^2L_g^2}{\gamma}
			\Bigr)
			+
			n\eta^2\tau^2
			\left(
			C_f^2L_g^2
			+
			C_g^4L_f^2
			+
			\frac{L_F^2}{2}
			\right)
			\biggr)^{1/2}
			\notag\\
			&\qquad
			+
			\sqrt6\,\rho\tau
			\left(
			C_f^2\Delta_{\nabla g}^2
			+
			2C_g^2\Delta_f^2
			+
			2C_g^2L_f^2\Delta_g^2
			\right)^{1/2}
			\biggr].
			\label{eq:fedcomuon-vr-bound}
		\end{align}
	}
\end{theorem}

\begin{remark}[Parameter choice and complexity]
	\label{rem:parameter-complexity}
	By setting $\eta=\alpha=\beta=\gamma=T^{-2/3}$,
	$\rho=T^{-1/3}$, $b=T^{2/3}$, and $\tau=O(1)$,
	Theorem~\ref{thm:fedcomuon-vr} yields
	$T^{-1}\sum_{t=0}^{T-1}\mathbb E\|\nabla F(\bar W_t)\|_F
	=O(T^{-1/3})$. Therefore, FedCoMuon-VR achieves an
	$\epsilon$-stationary point in
	$T=O(\epsilon^{-3})$ iterations.
	The corresponding per-client sample complexity is
	$O(b+T)=O(T)=O(\epsilon^{-3})$, and the communication complexity is
	$O(T/\tau)=O(\epsilon^{-3})$.
\end{remark}

\begin{figure*}[t]
	\centering
	\includegraphics[width=\textwidth]{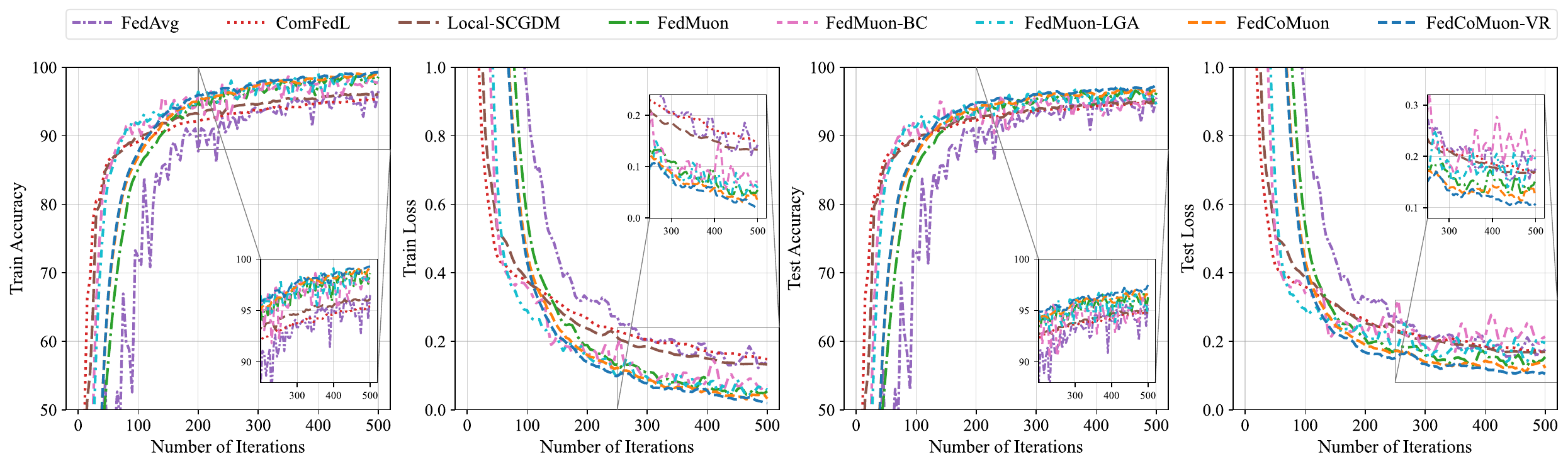}
	\caption{Training and test performance on the imbalanced MNIST
		robust federated learning task.}
	\label{fig:robust-fl}
\end{figure*}

\begin{figure}[t]
	\centering
	\resizebox{0.80\textwidth}{!}{
	\includegraphics[width=\columnwidth]{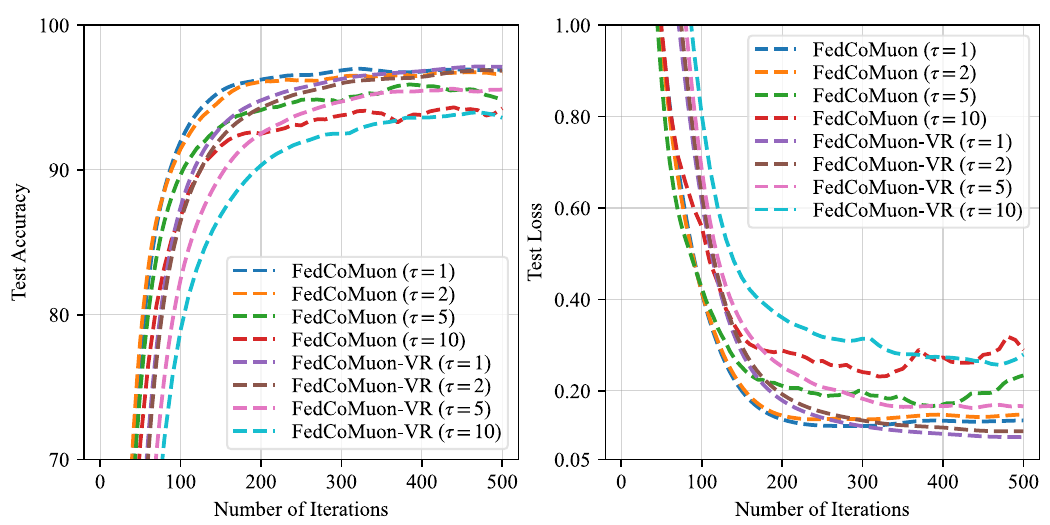}
}
	\caption{Effect of the synchronization gap $\tau$ on our algorithm for the robust federated learning task.}
	\label{fig:robust-sensitivity}
\end{figure}

\section{Numerical Experiments}
\label{sec:experiments}

In the section, we evaluate FedCoMuon and FedCoMuon-VR on robust federated
learning and task-distributed meta learning. In the experiment, we compare our methods with
task-specific standard federated baselines: FedAvg~\citep{mcmahan2017} for
robust federated learning and FedMAML~\citep{fallah2020personalized} for
task-distributed meta learning. We also consider three recent Muon-based federated methods. 
Since all three methods are named FedMuon in their original papers, 
we distinguish them as FedMuon~\citep{zhang2025provable}, FedMuon-LGA~\citep{liu2025fedmuon}, and 
FedMuon-BC~\citep{takezawa2025fedmuon}. In addition, we
include the federated compositional methods
ComFedL~\citep{huang2021compositional} and
Local-SCGDM~\citep{gao2022convergence}. For all methods, the learning rates
and method-specific hyper-parameters are selected via grid search, and we
report the best-performing configurations. 
% All experiments are run on a machine with a 24-vCPU 13th Gen Intel Core i9-13900KF CPU and two NVIDIA RTX 4090 GPUs.

\subsection{Robust Federated Learning}
\label{sec:exp-robust}

In this experiment, we evaluate FedCoMuon and FedCoMuon-VR on robust
federated learning, which can be formulated as the following distributed
compositional optimization problem:
\begin{equation}
	\label{eq:robust-fl}
	\min_{W\in\mathbb{R}^{m\times n}}
	\frac{1}{K}\sum_{k=1}^K
	f\left(g^k(W)/\lambda\right),
\end{equation}
where $f(\cdot)=\exp(\cdot)$ and $\lambda>0$ is a regularization parameter.
Other monotonically increasing functions may also be used as $f$.
We implement image classification on the MNIST~\citep{lecun1998gradient} dataset 
and language modeling on the WikiText-2~\citep{merity2016pointer} dataset. 
Specifically, we train a 4-layer CNN on MNIST and a Transformer on WikiText-2.

\begin{figure}[t]
	\centering
	\resizebox{0.80\textwidth}{!}{
	\includegraphics[width=\columnwidth]{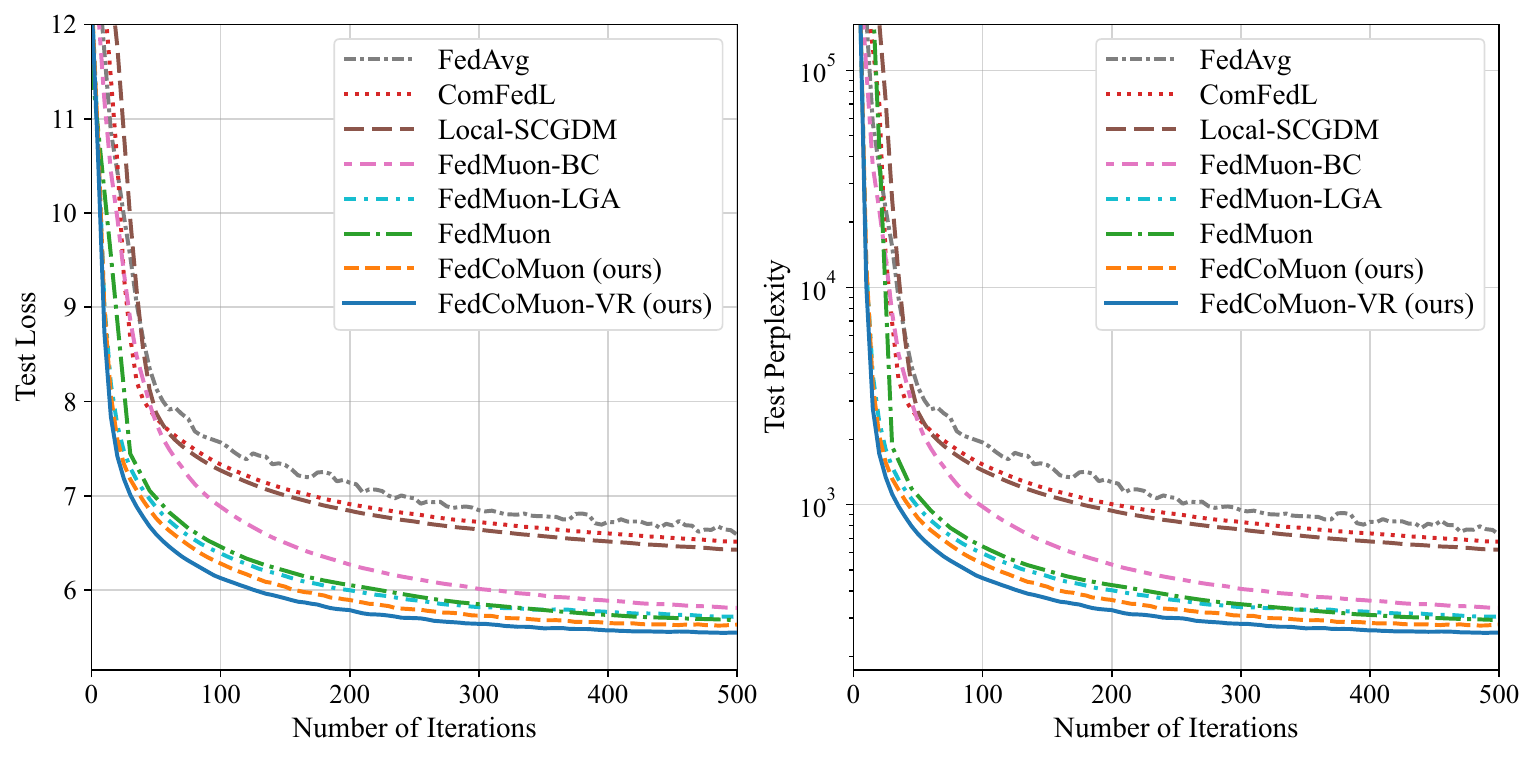}
}
	\caption{Test loss and perplexity of each method on the WikiText-2 language-modeling task.}
	\label{fig:wikitext}
\end{figure}

\subsubsection{Image Classification on MNIST}
\label{sec:exp-robust-mnist}

For image classification, we train a 4-layer CNN on MNIST in a
federated system with 10 clients. To create an imbalanced partition,
we assign 5000 training images to one client and 20 images to each of
the remaining clients. For FedCoMuon and FedCoMuon-VR, the learning
rate is set to $0.01$. The total number of training iterations is set
to $500$, and the synchronization gap is set to $\tau=5$ unless
otherwise specified. Additional implementation details and complete hyperparameter settings are
provided in the Appendix~\ref{app:experimental-settings}.

As shown in Figure~\ref{fig:robust-fl}, FedCoMuon and FedCoMuon-VR
outperform the baseline methods in terms of both training and test
performance under the highly imbalanced data partition. Our methods outperform 
the federated compositional baselines, demonstrating the effectiveness of Muon-based 
updates for robust compositional optimization.
In particular, FedCoMuon-VR exhibits more stable convergence and achieves
the best overall performance, showing the benefit of the variance-reduction
mechanism. Figure~\ref{fig:robust-sensitivity} further shows that our methods
perform consistently under different synchronization gaps, with $\tau=1$
yielding the best performance.

\subsubsection{Language Modeling on WikiText-2}
\label{sec:exp-robust-wikitext}

For language modeling, we conduct experiments on WikiText-2 with an 8-layer
Transformer language model. The Transformer has a hidden dimension of 768,
8 attention heads, a feed-forward dimension of 1024, sinusoidal positional
encodings, and a sequence length of 128. We split the training data across
10 clients in an imbalanced manner, where one client holds about 50\% of the
training blocks and the remaining clients equally share the rest. For FedCoMuon 
and FedCoMuon-VR, the learning rates are set to \(0.02\) and \(0.03\), respectively. 
The total
number of training iterations is set to $500$. We report test loss and
perplexity (PPL) as the evaluation metrics. Additional implementation details
and complete hyper-parameter settings are provided in
Appendix~\ref{app:experimental-settings}.

As shown in Figure~\ref{fig:wikitext}, both FedCoMuon and FedCoMuon-VR
converge faster and achieve lower test loss and perplexity than the baseline
methods. These results demonstrate the effectiveness of the proposed
Muon-based compositional updates for Transformer language models with many
matrix-valued parameters. In particular, FedCoMuon-VR achieves the best
overall performance.

\begin{figure}[t]
	\centering
	\begin{subfigure}[b]{0.48\textwidth}
	\includegraphics[width=\columnwidth]{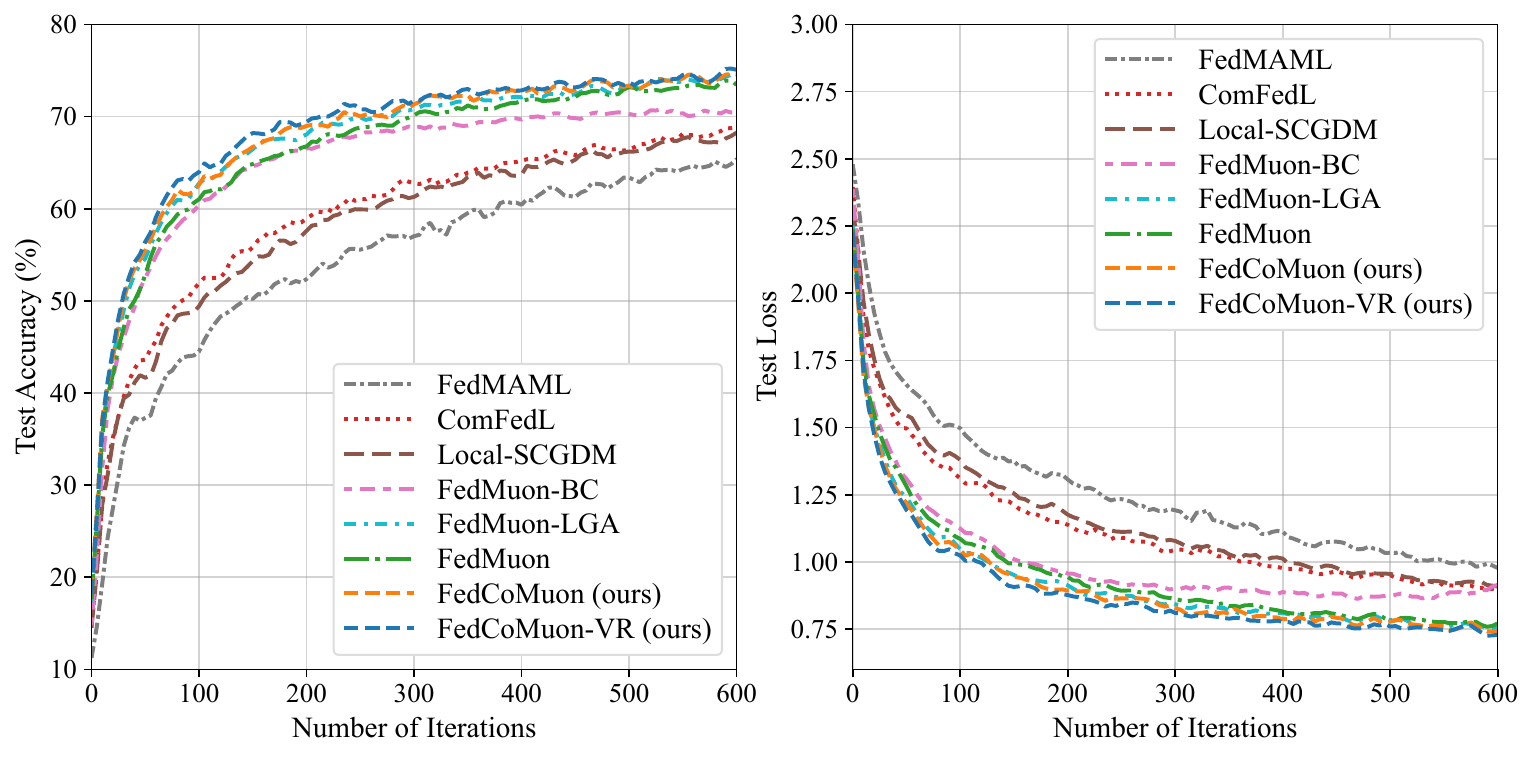}
	\caption{$\chi=0.3$}
	\end{subfigure}
	\begin{subfigure}[b]{0.48\textwidth}
	\includegraphics[width=\columnwidth]{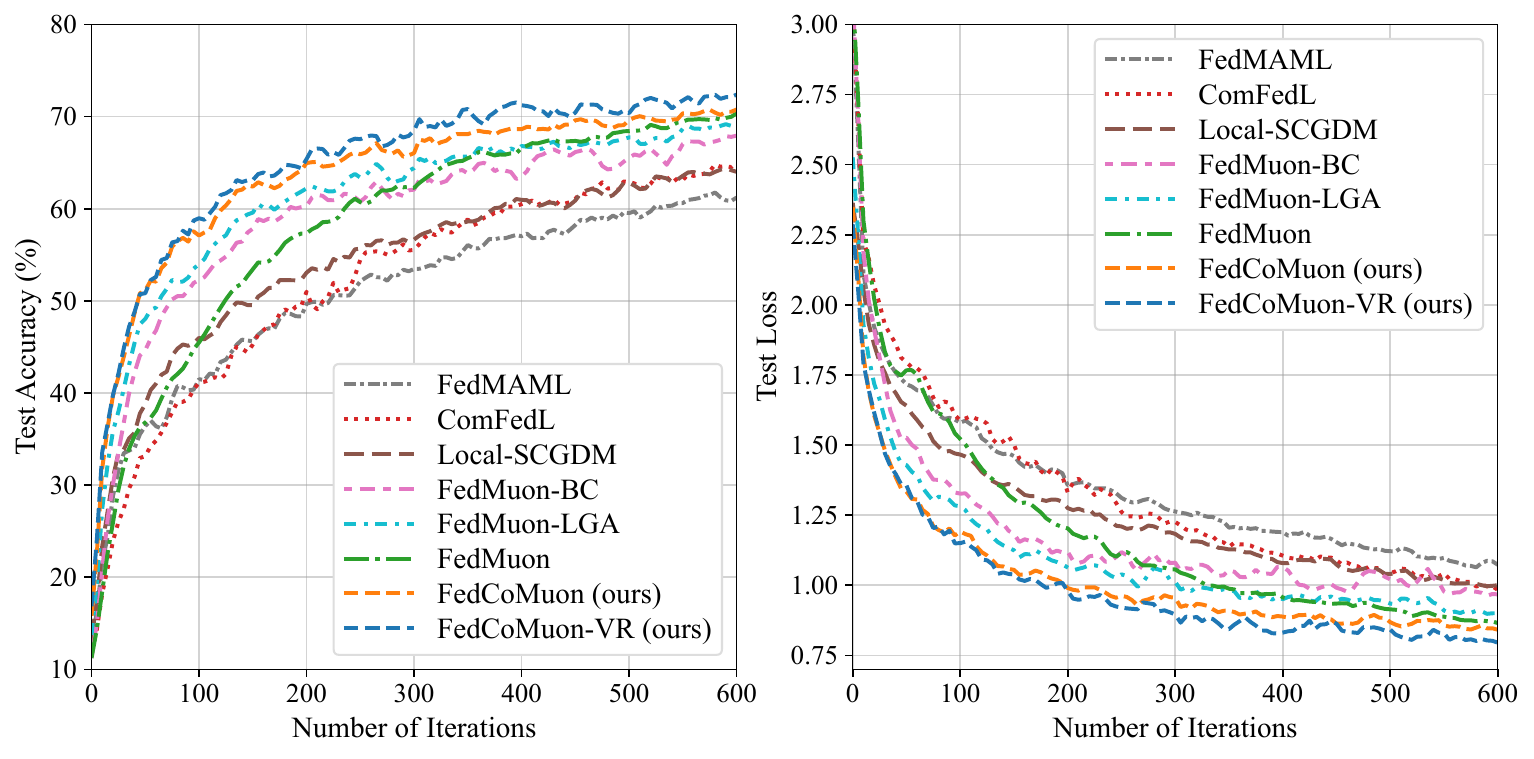}
	\caption{$\chi=0.5$}
	\end{subfigure}
	\begin{subfigure}[b]{0.48\textwidth}
	\includegraphics[width=\columnwidth]{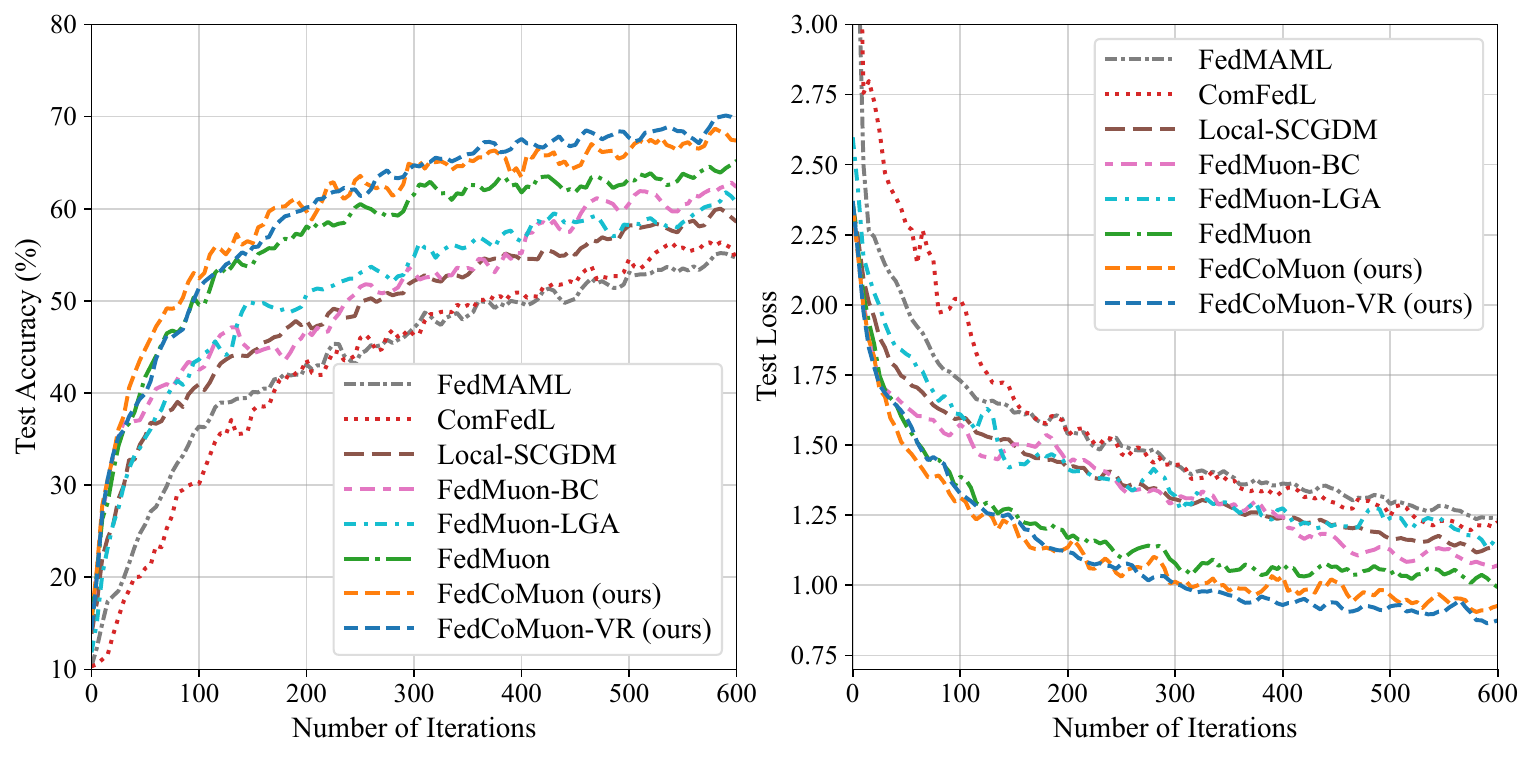}	
	\caption{$\chi=0.7$}
	\end{subfigure}
	\caption{Test accuracy and loss of each method on the task-distributed CNN
		meta-learning task with heterogeneous CIFAR-10 data for
		$\chi\in\{0.3,0.5,0.7\}$.}
	\label{fig:maml-cnn}
\end{figure}

\begin{figure}[t]
	\centering
	\resizebox{0.80\textwidth}{!}{
	\includegraphics[width=\columnwidth]{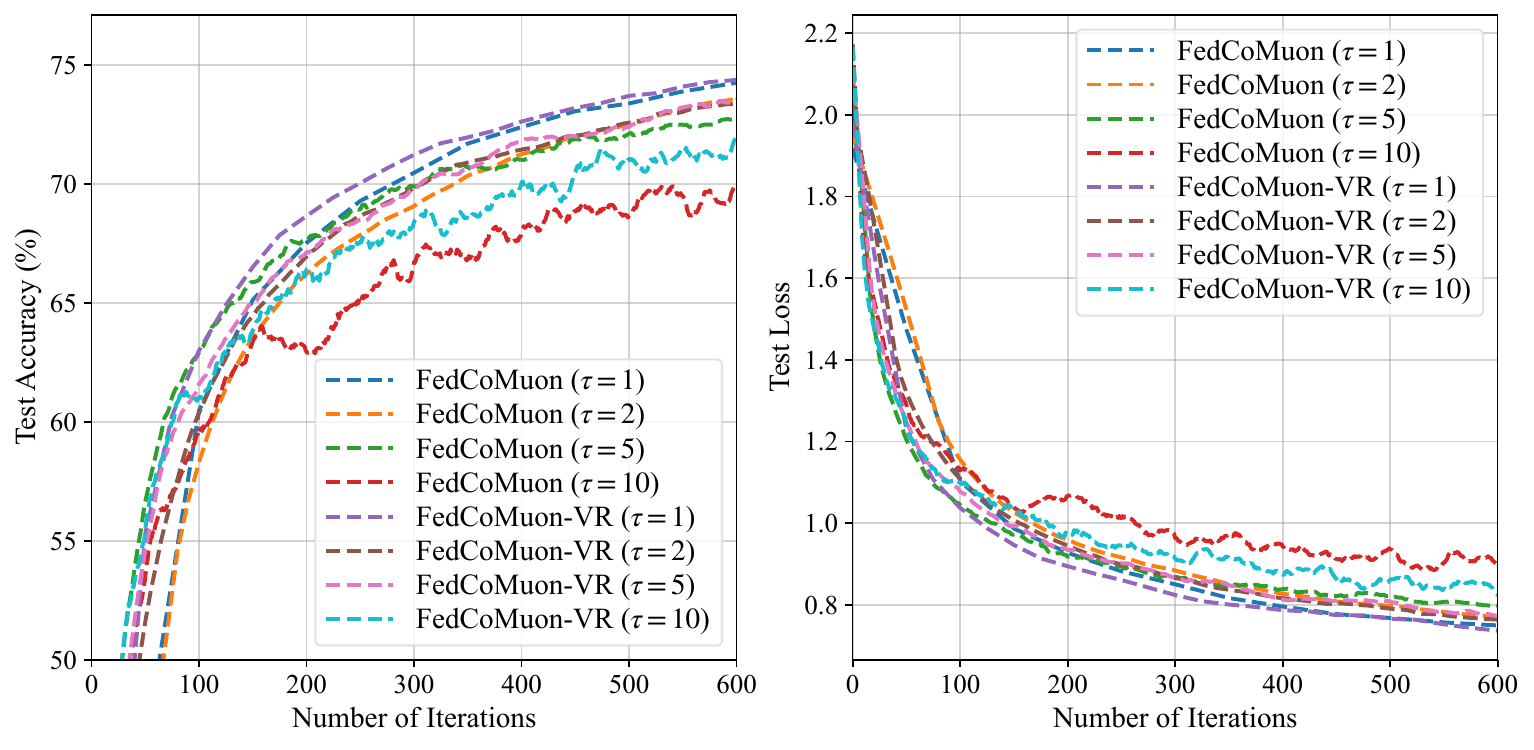}
}
	\caption{Effect of the synchronization gap $\tau$ on task-distributed CNN meta learning with heterogeneous CIFAR-10 data.}
	\label{fig:maml-tau}
\end{figure}

\begin{figure}[t]
	\centering
	\resizebox{0.80\textwidth}{!}{
	\includegraphics[width=\columnwidth]{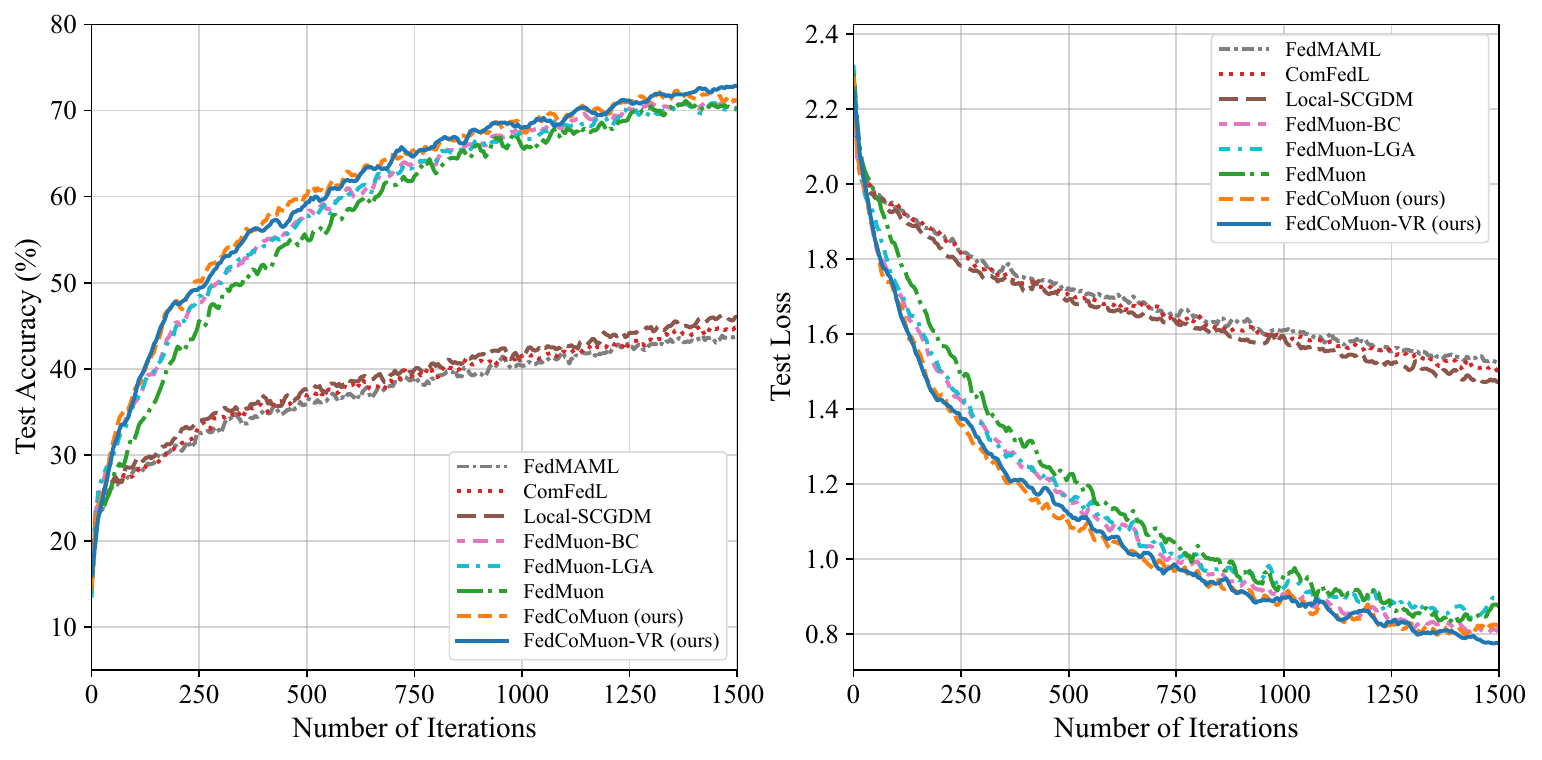}
}
	\caption{Test accuracy and loss of each method on the task-distributed ViT-Tiny
		meta learning task with heterogeneous CIFAR-10 for $\chi=0.3$.}
	\label{fig:maml-vit}
\end{figure}

\subsection{Task-Distributed Meta Learning Problem}
\label{sec:exp-maml}

In this experiment, we conduct task-distributed meta learning experiments 
on the CIFAR-10~\citep{krizhevsky2009learning} dataset. We first consider a 7-layer CNN
as a standard setting for comparing federated compositional optimization
methods under different levels of data heterogeneity. We further extend the
evaluation to ViT-Tiny~\citep{touvron2021training}, a matrix-wise vision
model based on the Transformer architecture and containing many
matrix-valued parameter blocks. Specifically, we optimize the following
risk-sensitive compositional MAML objective:
\begin{equation}
	\label{eq:maml-prelim}
	\min_{W \in \mathbb{R}^{m\times n}}
	\frac{1}{K}\sum_{k=1}^K
	\exp\left(
	\ell^k\left(W-\eta\nabla \ell^k(W)\right)/\lambda
	\right),
\end{equation}
where $\ell^k$ denotes the loss function on client~$k$, $\eta>0$ is the
inner-loop learning rate, and $\lambda>0$ is a regularization parameter
controlling the degree of risk sensitivity.

\subsubsection{CNN-Based Meta Learning}
\label{sec:exp-maml-cnn}

For CNN-based meta learning, we conduct experiments on CIFAR-10 using a
7-layer CNN in a federated system with 10 clients and one central server.
Each client is assigned a distinct dominant class, where a $\chi$ fraction
of its local samples belongs to the dominant class and a $(1-\chi)/9$
fraction belongs to each of the remaining classes. We evaluate
$\chi\in\{0.3,0.5,0.7\}$. For FedCoMuon and FedCoMuon-VR, the inner- and
outer-loop learning rates are set to $0.03$ and $0.1$ respectively. The
regularization parameter is set to $\lambda=0.5$, the synchronization gap
is set to $\tau=5$, and the total number of training iterations is set to
$600$. Additional implementation details and complete hyperparameter
settings are provided in the Appendix~\ref{app:experimental-settings}.

As shown in Figure~\ref{fig:maml-cnn}, FedCoMuon and FedCoMuon-VR
achieve better overall performance than the baseline methods under
different levels of data heterogeneity. As $\chi$ increases, the
optimization problem becomes more challenging, while the relative
advantage of our methods over the baseline methods becomes more
pronounced. Figure~\ref{fig:maml-tau} further
shows the sensitivity of our methods to different synchronization gaps,
with $\tau=1$ yielding the best performance.

\subsubsection{ViT-Tiny-Based Meta Learning}
\label{sec:exp-maml-vit}

For task-distributed meta learning with ViT-Tiny, we conduct experiments on
CIFAR-10 using a ViT-Tiny model with 12 Transformer blocks, a hidden
dimension of 192, three attention heads, and a patch size of 4. We adopt the
same federated setting and dominant-class data partition as in the CNN
experiments and set $\chi=0.3$. For both FedCoMuon and FedCoMuon-VR, 
the inner- and outer-loop learning rates are set to $0.005$ and $0.01$ respectively. 
The total number of training iterations is
set to $1500$. Additional implementation details and complete hyperparameter
settings are provided in the Appendix~\ref{app:experimental-settings}.

As shown in Figure~\ref{fig:maml-vit}, FedCoMuon and FedCoMuon-VR achieve 
substantially higher test accuracy and lower test loss than FedMAML and the 
federated compositional baselines. They also remain competitive with the three 
FedMuon baselines throughout training. These results demonstrate the effectiveness 
of our compositional Muon methods for task-distributed meta learning with the 
Transformer-based ViT-Tiny model.

\section{Conclusion}
\label{sec:conclusion}
In this paper, we studied the matrix-wise composition optimization, and proposed a class of effective federated compositional Muon algorithms (i.e., FedCoMuon and FedCoMuon-VR), which build on  compositional gradient tracking and
orthogonalized momentum. In theory, we established convergence guarantees
under non-i.i.d. and non-convex settings. In particular, our FedCoMuon-VR algorithm achieves
a lower sample complexity of $O(\epsilon^{-3})$ for finding an
$\epsilon$-stationary solution than the existing FedMuon algorithms. Extensive experiments on robust federated
learning and the task-distributed meta learning problem demonstrate the
effectiveness of our proposed algorithms.

\small

\bibliographystyle{plainnat}

\bibliography{FeCoMuon}

\newpage

\appendix
\onecolumn

\setcounter{theorem}{0}
\renewcommand{\thetheorem}{\arabic{theorem}}

\section{Convergence Analysis of our FedCoMuon Algorithm}
\label{app:fedcomuon-proof}

In this section, we provide the detailed convergence analysis of
FedCoMuon under the assumptions stated in the main paper. We first
introduce the following notation: $\bar W_t=\frac{1}{K}\sum_{k=1}^K W_t^k$
and $\bar M_t=\frac{1}{K}\sum_{k=1}^K M_t^k$. Moreover,
$F(W)=\frac{1}{K}\sum_{k=1}^K f^k(g^k(W))$,
and
\[
\nabla F(W)=\frac{1}{K}\sum_{k=1}^K
\nabla g^k(W)
\bigl(\nabla_y f^k(g^k(W))\otimes I_n\bigr).
\]
We next establish several auxiliary lemmas used in the convergence
analysis.

\begin{lemma}
	\label{lem:smooth-compositional}
	Given Assumptions~\ref{assm:bounded-grad}, \ref{assm:oracle},
	and~\ref{assm:smooth}, for each
	client $k\in[K]$, the local compositional objective
	$F^k(W)=f^k(g^k(W))$ is $L_F$-smooth, i.e., for any
	$W_1,W_2\in\mathbb R^{m\times n}$,
	\begin{equation}
		\norm{
			\nabla F^k(W_1)-\nabla F^k(W_2)
		}_F
		\le
		L_F\|W_1-W_2\|_F,
	\end{equation}
	where $L_F=C_fL_g+C_g^2L_f$. Consequently,
	$F(W)=K^{-1}\sum_{k=1}^K F^k(W)$ is also $L_F$-smooth.
\end{lemma}

\begin{proof}
	Jensen's inequality and Assumptions~\ref{assm:oracle}
	and~\ref{assm:bounded-grad} imply
	\(\|\nabla g^k(W)\|_F\le C_g\) and
	\(\|\nabla_y f^k(y)\|\le C_f\).
	Consequently, $g^k$ is $C_g$-Lipschitz, i.e.,
	\(\|g^k(W_1)-g^k(W_2)\|
	\le C_g\|W_1-W_2\|_F\).
	For any $W_1,W_2\in\mathbb R^{m\times n}$, we have
	\begin{align}
		&\norm{
			\nabla F^k(W_1)-\nabla F^k(W_2)
		}_F \notag\\
		&=
		\norm{
			\nabla g^k(W_1)
			\left(
			\nabla_y f^k(g^k(W_1))\otimes I_n
			\right)
			-
			\nabla g^k(W_2)
			\left(
			\nabla_y f^k(g^k(W_2))\otimes I_n
			\right)
		}_F \notag\\
		&=
		\norm{
			\nabla g^k(W_1)
			\left(
			\nabla_y f^k(g^k(W_1))\otimes I_n
			\right)
			-
			\nabla g^k(W_1)
			\left(
			\nabla_y f^k(g^k(W_2))\otimes I_n
			\right)
			\notag\\
			&\quad+
			\nabla g^k(W_1)
			\left(
			\nabla_y f^k(g^k(W_2))\otimes I_n
			\right)
			-
			\nabla g^k(W_2)
			\left(
			\nabla_y f^k(g^k(W_2))\otimes I_n
			\right)
		}_F \notag\\
		&\le
		\norm{
			\nabla g^k(W_1)
			\left(
			\left(
			\nabla_y f^k(g^k(W_1))
			-
			\nabla_y f^k(g^k(W_2))
			\right)\otimes I_n
			\right)
		}_F \notag\\
		&\quad+
		\norm{
			\left(
			\nabla g^k(W_1)-\nabla g^k(W_2)
			\right)
			\left(
			\nabla_y f^k(g^k(W_2))\otimes I_n
			\right)
		}_F \notag\\
		&\le
		\|\nabla g^k(W_1)\|_F
		\norm{
			\nabla_y f^k(g^k(W_1))
			-
			\nabla_y f^k(g^k(W_2))
		} \notag\\
		&\quad+
		\|\nabla g^k(W_1)-\nabla g^k(W_2)\|_F
		\norm{
			\nabla_y f^k(g^k(W_2))
		} \notag\\
		&\le
		C_gL_f\|g^k(W_1)-g^k(W_2)\|
		+
		C_fL_g\|W_1-W_2\|_F \notag\\
		&\le
		\left(C_g^2L_f+C_fL_g\right)
		\|W_1-W_2\|_F .
	\end{align}
	
	Moreover,
	\begin{align}
		\|\nabla F(W_1)-\nabla F(W_2)\|_F
		&=
		\norm{
			\frac{1}{K}\sum_{k=1}^K
			\left(
			\nabla F^k(W_1)-\nabla F^k(W_2)
			\right)
		}_F \notag\\
		&\le
		\frac{1}{K}\sum_{k=1}^K
		\norm{
			\nabla F^k(W_1)-\nabla F^k(W_2)
		}_F \notag\\
		&\le
		L_F\|W_1-W_2\|_F ,
	\end{align}
	where we used Assumptions~\ref{assm:bounded-grad} and~\ref{assm:smooth},
	the triangle inequality, and the submultiplicativity of matrix norms.
\end{proof}

\begin{lemma}
	\label{lem:muon-properties}
	Let $M\in\mathbb R^{m\times n}$ have compact singular value decomposition
	$M=U\Sigma V^\top$. Then
	\begin{equation}
		\|UV^\top\|_F\le \sqrt n,
		\qquad
		\langle M,UV^\top\rangle=\|M\|_* .
	\end{equation}
	Moreover, for any $A,B\in\mathbb R^{m\times n}$, if
	$B=U_B\Sigma_BV_B^\top$ is a compact singular value decomposition, then
	\begin{equation}
		\langle A,U_BV_B^\top\rangle
		\ge
		\|A\|_F-2\sqrt n\,\|A-B\|_F .
	\end{equation}
\end{lemma}

\begin{proof}
	Let $r=\operatorname{rank}(M)$. Then
	\begin{align}
		\|UV^\top\|_F^2
		&=
		\operatorname{tr}\!\left((UV^\top)^\top UV^\top\right)
		=
		\operatorname{tr}\!\left(VU^\top UV^\top\right)
		=
		\operatorname{tr}\!\left(VV^\top\right)
		=
		\operatorname{tr}\!\left(V^\top V\right)
		=
		r
		\le n,
	\end{align}
	and
	\begin{align}
		\langle M,UV^\top\rangle
		&=
		\operatorname{tr}\!\left(M^\top UV^\top\right) \notag\\
		&=
		\operatorname{tr}\!\left((U\Sigma V^\top)^\top UV^\top\right) \notag\\
		&=
		\operatorname{tr}\!\left(V\Sigma U^\top UV^\top\right)
		=
		\operatorname{tr}\!\left(V\Sigma V^\top\right) \notag\\
		&=
		\operatorname{tr}\!\left(\Sigma V^\top V\right)
		=
		\operatorname{tr}(\Sigma)
		=
		\|M\|_* .
	\end{align}
	For any $A,B\in\mathbb R^{m\times n}$,
	\begin{align}
		\langle A,U_BV_B^\top\rangle
		&=
		\langle B,U_BV_B^\top\rangle
		+
		\langle A-B,U_BV_B^\top\rangle \notag\\
		&\ge
		\|B\|_*
		-
		\|A-B\|_F\|U_BV_B^\top\|_F \notag\\
		&\ge
		\|B\|_F
		-
		\sqrt n\,\|A-B\|_F \notag\\
		&\ge
		\|A\|_F
		-
		\|A-B\|_F
		-
		\sqrt n\,\|A-B\|_F \notag\\
		&=
		\|A\|_F
		-
		(1+\sqrt n)\|A-B\|_F \notag\\
		&\ge
		\|A\|_F
		-
		2\sqrt n\,\|A-B\|_F ,
	\end{align}
	where we used the compact SVD, the Cauchy--Schwarz inequality, and the
	standard nuclear--Frobenius norm relations.
\end{proof}

\begin{lemma}
	\label{lem:app-inner-tracking}
	Let $\{W_t^k,u_t^k,M_t^k\}$ be generated by
	Algorithm~\ref{alg:fedcomuon}.
	Given Assumptions~\ref{assm:oracle} and~\ref{assm:bounded-grad}, for any
	$T,\tau>0$ and $0<\alpha<1$ satisfying $\alpha\tau\le1$, we have
	\begin{align}
		&\frac1T\sum_{t=0}^{T-1}\frac1K\sum_{k=1}^K
		\mathbb E\|u_t^k-g^k(W_t^k)\|^2 \le
		\frac{4}{\alpha T}\frac1K\sum_{k=1}^K
		\mathbb E\|u_0^k-g^k(W_0^k)\|^2
		+\frac{86C_g^2n}{\alpha^2}\eta^2
		+86\alpha\sigma_g^2.
		\label{eq:app-inner-average}
	\end{align}
\end{lemma}

\begin{proof}
	Within each communication block, the update rule,
	Assumptions~\ref{assm:oracle} and~\ref{assm:bounded-grad},
	Lemma~\ref{lem:muon-properties}, and Young's inequality give
	\begin{align}
		&\mathbb E\norm{u_{t+1}^k-g^k(W_{t+1}^k)}^2 \notag\\
		&=
		\mathbb E\norm{
			(1-\alpha)
			\left(u_t^k-g^k(W_t^k)\right)
			+(1-\alpha)
			\left(g^k(W_t^k)-g^k(W_{t+1}^k)\right) +
			\alpha
			\left(g^k(W_{t+1}^k;\xi_{t+1}^k)-g^k(W_{t+1}^k)\right)
		}^2 \notag\\
		&\le
		(1-\alpha)\mathbb E\norm{u_t^k-g^k(W_t^k)}^2
		+\frac1\alpha
		\mathbb E\norm{g^k(W_t^k)-g^k(W_{t+1}^k)}^2
		+\alpha^2\sigma_g^2 \notag\\
		&\le
		(1-\alpha)\mathbb E\norm{u_t^k-g^k(W_t^k)}^2
		+\frac{C_g^2}{\alpha}
		\mathbb E\norm{W_{t+1}^k-W_t^k}_F^2
		+\alpha^2\sigma_g^2 \notag\\
		&=
		(1-\alpha)\mathbb E\norm{u_t^k-g^k(W_t^k)}^2
		+\frac{C_g^2\eta^2}{\alpha}
		\mathbb E\norm{U_t^k(V_t^k)^\top}_F^2
		+\alpha^2\sigma_g^2 \notag\\
		&\le
		(1-\alpha)\mathbb E\norm{u_t^k-g^k(W_t^k)}^2
		+\frac{C_g^2n}{\alpha}\eta^2
		+\alpha^2\sigma_g^2 .
		\label{eq:app-inner-one-step}
	\end{align}
	Let $s=q\tau$ be the beginning of a block. Iterating
	\eqref{eq:app-inner-one-step} within the block gives
	\begin{align}
		&\sum_{\ell=0}^{\tau-1}\frac1K\sum_{k=1}^K
		\mathbb E\|u_{s+\ell}^k-g^k(W_{s+\ell}^k)\|^2 \notag\\
		&\le
		\sum_{\ell=0}^{\tau-1}
		\Bigg[
		(1-\alpha)^\ell\frac1K\sum_{k=1}^K
		\mathbb E\|u_s^k-g^k(W_s^k)\|^2
		+\sum_{i=0}^{\ell-1}(1-\alpha)^i
		\left(
		\frac{C_g^2n}{\alpha}\eta^2+\alpha^2\sigma_g^2
		\right)
		\Bigg] \notag\\
		&\le
		\tau\frac1K\sum_{k=1}^K
		\mathbb E\|u_s^k-g^k(W_s^k)\|^2
		+\tau^2
		\left(
		\frac{C_g^2n}{\alpha}\eta^2+\alpha^2\sigma_g^2
		\right),
		\label{eq:app-inner-block-sum}
		\\[0.3em]
		&\frac1K\sum_{k=1}^K
		\mathbb E\|u_{s+\tau}^k-g^k(W_{s+\tau}^k)\|^2 \notag\\
		&\le
		(1-\alpha)^\tau
		\frac1K\sum_{k=1}^K
		\mathbb E\|u_s^k-g^k(W_s^k)\|^2
		+\sum_{i=0}^{\tau-1}(1-\alpha)^i
		\left(
		\frac{C_g^2n}{\alpha}\eta^2+\alpha^2\sigma_g^2
		\right) \notag\\
		&\le
		(1-\alpha)^\tau
		\frac1K\sum_{k=1}^K
		\mathbb E\|u_s^k-g^k(W_s^k)\|^2
		+\tau
		\left(
		\frac{C_g^2n}{\alpha}\eta^2+\alpha^2\sigma_g^2
		\right).
		\label{eq:app-inner-block-end}
	\end{align}
	At the block end, the server replaces $W_{s+\tau}^k$ with
	$\bar W_{s+\tau}$ while leaving $u_{s+\tau}^k$ unchanged. Young's inequality
	with parameter $\alpha\tau/2$ gives
	\begin{align}
		&\frac1K\sum_{k=1}^K
		\mathbb E\|u_{s+\tau}^k-g^k(\bar W_{s+\tau})\|^2 \notag\\
		&\le
		\left(1+\frac{\alpha\tau}{2}\right)
		\frac1K\sum_{k=1}^K
		\mathbb E\|u_{s+\tau}^k-g^k(W_{s+\tau}^k)\|^2 +
		\left(1+\frac{2}{\alpha\tau}\right)
		\frac1K\sum_{k=1}^K
		\mathbb E\|g^k(W_{s+\tau}^k)-g^k(\bar W_{s+\tau})\|^2
		\notag\\
		&\le
		\left(1+\frac{\alpha\tau}{2}\right)(1-\alpha)^\tau
		\frac1K\sum_{k=1}^K
		\mathbb E\|u_s^k-g^k(W_s^k)\|^2
		+\frac32\tau
		\left(
		\frac{C_g^2n}{\alpha}\eta^2+\alpha^2\sigma_g^2
		\right)
		\notag\\
		&\quad+
		\left(1+\frac{2}{\alpha\tau}\right)C_g^2
		\frac{\eta^2}{K}\sum_{k=1}^K
		\mathbb E\norm{
			\sum_{i=s}^{s+\tau-1}
			\left(
			U_i^k(V_i^k)^\top-\frac1K\sum_{j=1}^K U_i^j(V_i^j)^\top
			\right)
		}_F^2
		\notag\\
		&\le
		\left(1+\frac{\alpha\tau}{2}\right)(1-\alpha)^\tau
		\frac1K\sum_{k=1}^K
		\mathbb E\|u_s^k-g^k(W_s^k)\|^2
		+\frac32\tau
		\left(
		\frac{C_g^2n}{\alpha}\eta^2+\alpha^2\sigma_g^2
		\right)
		+4\left(1+\frac{2}{\alpha\tau}\right)
		C_g^2\eta^2\tau^2n
		\notag\\
		&\le
		\left(1+\frac{\alpha\tau}{2}\right)e^{-\alpha\tau}
		\frac1K\sum_{k=1}^K
		\mathbb E\|u_s^k-g^k(W_s^k)\|^2
		+\frac32\tau
		\left(
		\frac{C_g^2n}{\alpha}\eta^2+\alpha^2\sigma_g^2
		\right)
		+\frac{12\tau C_g^2n}{\alpha}\eta^2
		\notag\\
		&\le
		\left(1-\frac{\alpha\tau}{3}\right)
		\frac1K\sum_{k=1}^K
		\mathbb E\|u_s^k-g^k(W_s^k)\|^2
		+14\tau
		\left(
		\frac{C_g^2n}{\alpha}\eta^2+\alpha^2\sigma_g^2
		\right),
		\label{eq:app-inner-block-recursion}
	\end{align}
	Let $Q=\left\lfloor T/\tau\right\rfloor$ and
	$R=T-Q\tau<\tau$. Since
	$\sum_{q=0}^{Q-1}\left(1-\frac{\alpha\tau}{3}\right)^q
	\le \frac{3}{\alpha\tau}$,
	combining \eqref{eq:app-inner-block-sum},
	\eqref{eq:app-inner-block-recursion}, and
	\eqref{eq:app-inner-one-step}, we have
	\begin{align}
		&\frac1T\sum_{t=0}^{T-1}\frac1K\sum_{k=1}^K
		\mathbb E\|u_t^k-g^k(W_t^k)\|^2
		\notag\\
		&=
		\frac1T
		\sum_{q=0}^{Q-1}\sum_{\ell=0}^{\tau-1}
		\frac1K\sum_{k=1}^K
		\mathbb E
		\|u_{q\tau+\ell}^k-g^k(W_{q\tau+\ell}^k)\|^2
		+
		\frac1T
		\sum_{\ell=0}^{R-1}
		\frac1K\sum_{k=1}^K
		\mathbb E
		\|u_{Q\tau+\ell}^k-g^k(W_{Q\tau+\ell}^k)\|^2
		\notag\\
		&\le
		\frac1T\Bigg[
		\tau\sum_{q=0}^{Q-1}
		\frac1K\sum_{k=1}^K
		\mathbb E\|u_{q\tau}^k-g^k(W_{q\tau}^k)\|^2
		+
		Q\tau^2
		\left(
		\frac{C_g^2n}{\alpha}\eta^2+\alpha^2\sigma_g^2
		\right)
		\notag\\
		&\qquad\qquad+
		R\frac1K\sum_{k=1}^K
		\mathbb E\|u_{Q\tau}^k-g^k(W_{Q\tau}^k)\|^2
		+
		R^2
		\left(
		\frac{C_g^2n}{\alpha}\eta^2+\alpha^2\sigma_g^2
		\right)
		\Bigg]
		\notag\\
		&\le
		\frac1T\Bigg[
		\tau\Bigg(
		\frac{3}{\alpha\tau}
		\frac1K\sum_{k=1}^K
		\mathbb E\|u_0^k-g^k(W_0^k)\|^2
		+
		\frac{42Q}{\alpha}
		\left(
		\frac{C_g^2n}{\alpha}\eta^2+\alpha^2\sigma_g^2
		\right)
		\Bigg)
		+
		Q\tau^2
		\left(
		\frac{C_g^2n}{\alpha}\eta^2+\alpha^2\sigma_g^2
		\right)
		\notag\\
		&\qquad\qquad+
		R\Bigg(
		\frac1K\sum_{k=1}^K
		\mathbb E\|u_0^k-g^k(W_0^k)\|^2
		+
		\frac{42}{\alpha}
		\left(
		\frac{C_g^2n}{\alpha}\eta^2+\alpha^2\sigma_g^2
		\right)
		\Bigg)
		+
		R^2
		\left(
		\frac{C_g^2n}{\alpha}\eta^2+\alpha^2\sigma_g^2
		\right)
		\Bigg]
		\notag\\
		&\le
		\frac1T\Bigg[
		\frac3\alpha
		\frac1K\sum_{k=1}^K
		\mathbb E\|u_0^k-g^k(W_0^k)\|^2
		+
		43T
		\left(
		\frac{C_g^2n}{\alpha^2}\eta^2+\alpha\sigma_g^2
		\right)
		\notag\\
		&\qquad\qquad+
		\frac1\alpha
		\frac1K\sum_{k=1}^K
		\mathbb E\|u_0^k-g^k(W_0^k)\|^2
		+
		43T
		\left(
		\frac{C_g^2n}{\alpha^2}\eta^2+\alpha\sigma_g^2
		\right)
		\Bigg]
		\notag\\
		&=
		\frac{4}{\alpha T}\frac1K\sum_{k=1}^K
		\mathbb E\|u_0^k-g^k(W_0^k)\|^2
		+
		\frac{86C_g^2n}{\alpha^2}\eta^2
		+
		86\alpha\sigma_g^2.
	\end{align}
	The second inequality follows by iterating
	\eqref{eq:app-inner-block-recursion}, while the last inequality uses
	$Q\tau\le T$, $R<\tau\le1/\alpha$, and $R\le T$.
	
\end{proof}

\begin{lemma}
	\label{lem:app-oracle-error}
	Let $\{W_t^k,u_t^k,M_t^k\}$ be generated by Algorithm~\ref{alg:fedcomuon}.
	For each client $k\in[K]$ and iteration $t\ge0$, define the stochastic
	compositional gradient estimator
	\begin{equation}
		Z_t^k
		=
		\nabla g^k(W_t^k;\xi_t^k)
		\left(
		\nabla_y f^k(u_t^k;\zeta_t^k)\otimes I_n
		\right).
		\label{eq:app-Z-definition}
	\end{equation}
	Given Assumptions~\ref{assm:smooth}, \ref{assm:bounded-grad},
	and~\ref{assm:oracle}, for any $k\in[K]$ and $t\ge0$, we have
	\begin{equation}
		\mathbb E
		\norm{
			Z_t^k-\nabla F^k(W_t^k)
		}_F
		\le
		C_gL_f
		\sqrt{
			\mathbb E
			\norm{
				u_t^k-g^k(W_t^k)
			}^2
		}
		+C_g\sigma_f+C_f\sigma_{\nabla g}.
		\label{eq:app-oracle-error}
	\end{equation}
\end{lemma}

\begin{proof}
	For any $k\in[K]$, we have
	\begin{align}
		&\mathbb E
		\norm{
			Z_t^k-\nabla F^k(W_t^k)
		}_F \notag\\
		&=
		\mathbb E
		\norm{
			\nabla g^k(W_t^k;\xi_t^k)
			\left(
			\nabla_y f^k(u_t^k;\zeta_t^k)\otimes I_n
			\right)
			\notag
			-
			\nabla g^k(W_t^k)
			\left(
			\nabla_y f^k(g^k(W_t^k))\otimes I_n
			\right)
		}_F \notag\\
		&\le
		\mathbb E
		\norm{
			\nabla g^k(W_t^k;\xi_t^k)
			\left(
			\left(
			\nabla_y f^k(u_t^k;\zeta_t^k)
			-
			\nabla_y f^k(u_t^k)
			\right)\otimes I_n
			\right)
		}_F \notag\\
		&\quad+
		\mathbb E
		\norm{
			\nabla g^k(W_t^k;\xi_t^k)
			\left(
			\left(
			\nabla_y f^k(u_t^k)
			-
			\nabla_y f^k(g^k(W_t^k))
			\right)\otimes I_n
			\right)
		}_F \notag\\
		&\quad+
		\mathbb E
		\norm{
			\left(
			\nabla g^k(W_t^k;\xi_t^k)
			-
			\nabla g^k(W_t^k)
			\right)
			\left(
			\nabla_y f^k(g^k(W_t^k))\otimes I_n
			\right)
		}_F \notag\\
		&\le
		\sqrt{
			\mathbb E
			\norm{
				\nabla g^k(W_t^k;\xi_t^k)
			}_F^2
		}
		\sqrt{
			\mathbb E
			\norm{
				\nabla_y f^k(u_t^k;\zeta_t^k)
				-
				\nabla_y f^k(u_t^k)
			}^2
		}
		\notag\\
		&\quad+
		L_f
		\sqrt{
			\mathbb E
			\norm{
				\nabla g^k(W_t^k;\xi_t^k)
			}_F^2
		}
		\sqrt{
			\mathbb E
			\norm{
				u_t^k-g^k(W_t^k)
			}^2
		}
		\notag\\
		&\quad+
		C_f
		\sqrt{
			\mathbb E
			\norm{
				\nabla g^k(W_t^k;\xi_t^k)
				-
				\nabla g^k(W_t^k)
			}_F^2
		}
		\notag\\
		&\le
		C_g\sigma_f
		+
		C_gL_f
		\sqrt{
			\mathbb E
			\norm{
				u_t^k-g^k(W_t^k)
			}^2
		}
		+
		C_f\sigma_{\nabla g},
	\end{align}
	where we used Assumptions~\ref{assm:smooth}, \ref{assm:oracle},
	and~\ref{assm:bounded-grad}, together with the Cauchy--Schwarz inequality.
\end{proof}

\begin{lemma}
	\label{lem:app-drift}
	Let $\bar W_t=K^{-1}\sum_{k=1}^K W_t^k$ and
	$\bar M_t=K^{-1}\sum_{k=1}^K M_t^k$, and let
	$s_t=\tau\lfloor t/\tau\rfloor$ denote the beginning of the
	communication block containing iteration $t$. Then, for every $t\ge0$,
	\begin{equation}
		\frac1K\sum_{k=1}^K
		\|W_t^k-\bar W_t\|_F
		\le
		2\eta\tau\sqrt n .
		\label{eq:app-model-drift}
	\end{equation}
	Moreover, under Assumptions~\ref{assm:smooth}, \ref{assm:oracle},
	\ref{assm:bounded-grad}, and~\ref{assm:hetero}, for any $T\ge1$,
	\begin{align}
		&\frac1T\sum_{t=0}^{T-1}\frac1K\sum_{k=1}^K
		\mathbb E\|M_t^k-\bar M_t\|_F \notag\\
		&\le
		\beta\tau
		\Bigg[
		2C_gL_f
		\left(
		\frac1T\sum_{t=0}^{T-1}\frac1K\sum_{k=1}^K
		\mathbb E\|u_t^k-g^k(W_t^k)\|^2
		\right)^{1/2}
		+2C_g\sigma_f+2C_f\sigma_{\nabla g}
		+4\eta\tau\sqrt n\,L_F+\delta
		\Bigg].
		\label{eq:app-momentum-drift}
	\end{align}
\end{lemma}

\begin{proof}
	For any $t\ge0$, since the last synchronization occurs at
	$s_t=\tau\lfloor t/\tau\rfloor$, we have $W_{s_t}^k=\bar W_{s_t}$ and
	\begin{align}
		\frac1K\sum_{k=1}^K
		\|W_t^k-\bar W_t\|_F
		&=
		\frac1K\sum_{k=1}^K
		\norm{
			\bar W_{s_t}
			-\eta\sum_{i=s_t}^{t-1}\frac1K\sum_{j=1}^K U_i^j(V_i^j)^\top
			-W_{s_t}^k
			+\eta\sum_{i=s_t}^{t-1}U_i^k(V_i^k)^\top
		}_F \notag\\
		&\le
		\eta\frac1K\sum_{k=1}^K
		\norm{
			\sum_{i=s_t}^{t-1}U_i^k(V_i^k)^\top
			-
			\sum_{i=s_t}^{t-1}\frac1K\sum_{j=1}^K U_i^j(V_i^j)^\top
		}_F \notag\\
		&\le
		\eta\frac1K\sum_{k=1}^K
		\norm{
			\sum_{i=s_t}^{t-1}U_i^k(V_i^k)^\top
		}_F
		+
		\eta\frac1K\sum_{k=1}^K
		\norm{
			\sum_{i=s_t}^{t-1}\frac1K\sum_{j=1}^K U_i^j(V_i^j)^\top
		}_F \notag\\
		&\le
		\eta\frac1K\sum_{k=1}^K
		\sum_{i=s_t}^{t-1}\|U_i^k(V_i^k)^\top\|_F
		+
		\eta\frac1K\sum_{k=1}^K
		\sum_{i=s_t}^{t-1}\frac1K\sum_{j=1}^K\|U_i^j(V_i^j)^\top\|_F \notag\\
		&\le
		2\eta(t-s_t)\sqrt n
		\le
		2\eta\tau\sqrt n .
	\end{align}
	
	Under the synchronization convention above,
	$M_{s_t}^k=\bar M_{s_t}$. For $t=s_t$, the momentum disagreement is zero.
	For $t>s_t$, averaging the momentum recursion over the clients and unrolling
	it from $s_t$ give
	\begin{align}
		M_t^k-\bar M_t
		&=
		(1-\beta)^{t-s_t}
		\left(M_{s_t}^k-\bar M_{s_t}\right)
		+
		\sum_{i=s_t+1}^{t}
		\beta(1-\beta)^{t-i}
		\left(
		Z_i^k-\frac1K\sum_{j=1}^K Z_i^j
		\right)
		\notag\\
		&=
		\sum_{i=s_t+1}^{t}
		\beta(1-\beta)^{t-i}
		\left(
		Z_i^k-\frac1K\sum_{j=1}^K Z_i^j
		\right).
	\end{align}
	Consequently, the triangle inequality gives, for every $t\ge0$,
	\begin{align}
		&\frac1K\sum_{k=1}^K
		\mathbb E\|M_t^k-\bar M_t\|_F
		\notag\\
		&\le
		\frac1K\sum_{k=1}^K
		\mathbb E\norm{
			\sum_{i=s_t+1}^{t}
			\beta(1-\beta)^{t-i}
			\left(
			Z_i^k-\frac1K\sum_{j=1}^K Z_i^j
			\right)
		}_F
		\notag\\
		&\le
		\sum_{i=s_t+1}^{t}
		\beta(1-\beta)^{t-i}
		\frac1K\sum_{k=1}^K
		\mathbb E\norm{
			Z_i^k-\frac1K\sum_{j=1}^K Z_i^j
		}_F .
		\label{eq:app-momentum-block}
	\end{align}
	For any $i$,
	\begin{align}
		&\frac1K\sum_{k=1}^K
		\mathbb E
		\norm{
			Z_i^k-\frac1K\sum_{j=1}^K Z_i^j
		}_F \notag\\
		&\le
		\frac2K\sum_{k=1}^K
		\mathbb E
		\norm{
			Z_i^k-\nabla F^k(W_i^k)
		}_F 
		+
		\frac1K\sum_{k=1}^K
		\mathbb E
		\norm{
			\nabla F^k(W_i^k)
			-
			\frac1K\sum_{j=1}^K
			\nabla F^j(W_i^j)
		}_F \notag\\
		&\le
		2C_gL_f
		\frac1K\sum_{k=1}^K
		\sqrt{
			\mathbb E
			\|u_i^k-g^k(W_i^k)\|^2
		}
		+
		2C_g\sigma_f
		+
		2C_f\sigma_{\nabla g} 
		+
		\frac1K\sum_{k=1}^K
		\mathbb E
		\norm{
			\nabla F^k(W_i^k)
			-
			\frac1K\sum_{j=1}^K
			\nabla F^j(W_i^j)
		}_F \notag\\
		&\le
		2C_gL_f
		\frac1K\sum_{k=1}^K
		\sqrt{
			\mathbb E
			\|u_i^k-g^k(W_i^k)\|^2
		}
		+
		2C_g\sigma_f
		+
		2C_f\sigma_{\nabla g} 
		+
		\frac1K\sum_{k=1}^K
		\mathbb E
		\norm{
			\nabla F^k(W_i^k)-\nabla F^k(\bar W_i)
		}_F 
		\notag\\
		&\quad+
		\frac1K\sum_{k=1}^K
		\mathbb E
		\norm{
			\nabla F^k(\bar W_i)-\nabla F(\bar W_i)
		}_F 
		+
		\mathbb E
		\norm{
			\nabla F(\bar W_i)
			-
			\frac1K\sum_{j=1}^K
			\nabla F^j(W_i^j)
		}_F \notag\\
		&\le
		2C_gL_f
		\frac1K\sum_{k=1}^K
		\sqrt{
			\mathbb E
			\|u_i^k-g^k(W_i^k)\|^2
		}
		+
		2C_g\sigma_f
		+
		2C_f\sigma_{\nabla g} +
		L_F
		\frac1K\sum_{k=1}^K
		\mathbb E\|W_i^k-\bar W_i\|_F
		\notag\\
		&\quad+
		\left(
		\frac1K\sum_{k=1}^K
		\mathbb E
		\|\nabla F^k(\bar W_i)-\nabla F(\bar W_i)\|_F^2
		\right)^{1/2} +
		L_F
		\frac1K\sum_{j=1}^K
		\mathbb E\|\bar W_i-W_i^j\|_F \notag\\
		&\le
		2C_gL_f
		\frac1K\sum_{k=1}^K
		\sqrt{
			\mathbb E
			\|u_i^k-g^k(W_i^k)\|^2
		}
		+
		2C_g\sigma_f
		+
		2C_f\sigma_{\nabla g} \notag\\
		&\quad+
		2L_F
		\frac1K\sum_{k=1}^K
		\mathbb E\|W_i^k-\bar W_i\|_F
		+
		\left(
		\frac1K\sum_{k=1}^K
		\mathbb E
		\|\nabla F^k(\bar W_i)-\nabla F(\bar W_i)\|_F^2
		\right)^{1/2} \notag\\
		&\le
		2C_gL_f
		\frac1K\sum_{k=1}^K
		\sqrt{
			\mathbb E
			\|u_i^k-g^k(W_i^k)\|^2
		}
		+
		2C_g\sigma_f
		+
		2C_f\sigma_{\nabla g}
		+
		4\eta\tau\sqrt n\,L_F
		+
		\delta .
		\label{eq:app-Z-drift}
	\end{align}
	Combining \eqref{eq:app-momentum-block} and \eqref{eq:app-Z-drift}, we obtain
	\begin{align}
		&\frac1T\sum_{t=0}^{T-1}\frac1K\sum_{k=1}^K
		\mathbb E\|M_t^k-\bar M_t\|_F \notag\\
		&\le
		\frac1T\sum_{t=0}^{T-1}
		\sum_{i=s_t+1}^{t}
		\beta(1-\beta)^{t-i}
		\Bigg[
		2C_gL_f
		\frac1K\sum_{k=1}^K
		\sqrt{
			\mathbb E
			\|u_i^k-g^k(W_i^k)\|^2
		}
		+2C_g\sigma_f+2C_f\sigma_{\nabla g}
		+4\eta\tau\sqrt n\,L_F+\delta
		\Bigg] \notag\\
		&\le
		\frac{\beta}{T}\sum_{t=0}^{T-1}
		\sum_{i=s_t+1}^{t}
		2C_gL_f
		\frac1K\sum_{k=1}^K
		\sqrt{
			\mathbb E
			\|u_i^k-g^k(W_i^k)\|^2
		}
		+
		\frac{\beta}{T}\sum_{t=0}^{T-1}
		\sum_{i=s_t+1}^{t}
		\left(
		2C_g\sigma_f+2C_f\sigma_{\nabla g}
		+4\eta\tau\sqrt n\,L_F+\delta
		\right) \notag\\
		&\le
		\frac{2\beta C_gL_f}{T}
		\sum_{i=0}^{T-1}
		\sum_{t=i}^{\min\{T-1,\tau\lfloor i/\tau\rfloor+\tau-1\}}
		\frac1K\sum_{k=1}^K
		\sqrt{
			\mathbb E
			\|u_i^k-g^k(W_i^k)\|^2
		} 
		+\frac{\beta}{T}\sum_{t=0}^{T-1}
		(t-s_t)
		\left(
		2C_g\sigma_f+2C_f\sigma_{\nabla g}
		+4\eta\tau\sqrt n\,L_F+\delta
		\right) \notag\\
		&\le
		\frac{2\beta\tau C_gL_f}{T}
		\sum_{i=0}^{T-1}
		\frac1K\sum_{k=1}^K
		\sqrt{
			\mathbb E
			\|u_i^k-g^k(W_i^k)\|^2
		} 
		+\beta\tau
		\left(
		2C_g\sigma_f+2C_f\sigma_{\nabla g}
		+4\eta\tau\sqrt n\,L_F+\delta
		\right) \notag\\
		&=
		\beta\tau
		\Bigg[
		\frac{2C_gL_f}{T}
		\sum_{t=0}^{T-1}
		\frac1K\sum_{k=1}^K
		\sqrt{
			\mathbb E
			\|u_t^k-g^k(W_t^k)\|^2
		}
		+2C_g\sigma_f+2C_f\sigma_{\nabla g}
		+4\eta\tau\sqrt n\,L_F+\delta
		\Bigg] \notag\\
		&\le
		\beta\tau
		\Bigg[
		2C_gL_f
		\left(
		\frac1T\sum_{t=0}^{T-1}
		\left(
		\frac1K\sum_{k=1}^K
		\sqrt{
			\mathbb E
			\|u_t^k-g^k(W_t^k)\|^2
		}
		\right)^2
		\right)^{1/2}
		+2C_g\sigma_f+2C_f\sigma_{\nabla g}
		+4\eta\tau\sqrt n\,L_F+\delta
		\Bigg] \notag\\
		&\le
		\beta\tau
		\Bigg[
		2C_gL_f
		\left(
		\frac1T\sum_{t=0}^{T-1}
		\frac1K\sum_{k=1}^K
		\mathbb E\|u_t^k-g^k(W_t^k)\|^2
		\right)^{1/2}
		+
		2C_g\sigma_f
		+
		2C_f\sigma_{\nabla g}
		+
		4\eta\tau\sqrt n\,L_F
		+
		\delta
		\Bigg],
	\end{align}
	where we used Lemmas~\ref{lem:app-oracle-error} and~\ref{lem:app-drift},
	Assumption~\ref{assm:hetero}, and the Cauchy--Schwarz inequality.
\end{proof}

\begin{lemma}
	\label{lem:app-descent}
	Let $\bar W_t=K^{-1}\sum_{k=1}^K W_t^k$
	and $\bar M_t=K^{-1}\sum_{k=1}^K M_t^k$. Given
	Assumptions~\ref{assm:smooth}, \ref{assm:oracle}, and~\ref{assm:bounded-grad},
	for every $t\ge0$, we have
	\begin{align}
		F(\bar W_{t+1})
		&\le
		F(\bar W_t)
		-\eta\|\nabla F(\bar W_t)\|_F
		+\frac{\eta^2nL_F}{2} 
		+2\eta\sqrt n\,\frac1K\sum_{k=1}^K\|M_t^k-\bar M_t\|_F
		\notag\\
		&\quad
		+2\eta\sqrt n\,L_F\frac1K\sum_{k=1}^K\|W_t^k-\bar W_t\|_F 
		+2\eta\sqrt n
		\norm{
			\frac1K\sum_{k=1}^K
			\left(\nabla F^k(W_t^k)-M_t^k\right)
		}_F .
		\label{eq:app-descent}
	\end{align}
\end{lemma}

\begin{proof}
	\begin{align}
		F(\bar W_{t+1})
		&\le
		F(\bar W_t)
		+
		\left\langle
		\nabla F(\bar W_t),
		\bar W_{t+1}-\bar W_t
		\right\rangle
		+
		\frac{L_F}{2}
		\|\bar W_{t+1}-\bar W_t\|_F^2 \notag\\
		&=
		F(\bar W_t)
		-\eta
		\left\langle
		\nabla F(\bar W_t),
		\frac1K\sum_{k=1}^K U_t^k(V_t^k)^\top
		\right\rangle
		+
		\frac{\eta^2L_F}{2}
		\norm{
			\frac1K\sum_{k=1}^K U_t^k(V_t^k)^\top
		}_F^2 \notag\\
		&\le
		F(\bar W_t)
		-\eta\frac1K\sum_{k=1}^K
		\left\langle
		\nabla F(\bar W_t),U_t^k(V_t^k)^\top
		\right\rangle
		+
		\frac{\eta^2nL_F}{2} \notag\\
		&\le
		F(\bar W_t)
		-\eta\|\nabla F(\bar W_t)\|_F
		+2\eta\sqrt n\frac1K\sum_{k=1}^K
		\|\nabla F(\bar W_t)-M_t^k\|_F
		+
		\frac{\eta^2nL_F}{2} \notag\\
		&\le
		F(\bar W_t)
		-\eta\|\nabla F(\bar W_t)\|_F
		+
		\frac{\eta^2nL_F}{2} 
		+2\eta\sqrt n
		\left(
		\frac1K\sum_{k=1}^K
		\|\bar M_t-M_t^k\|_F
		+
		\norm{
			\frac1K\sum_{k=1}^K
			\left(
			\nabla F^k(\bar W_t)-M_t^k
			\right)
		}_F
		\right) \notag\\
		&\le
		F(\bar W_t)
		-\eta\|\nabla F(\bar W_t)\|_F
		+
		\frac{\eta^2nL_F}{2}
		+2\eta\sqrt n
		\frac1K\sum_{k=1}^K
		\|\bar M_t-M_t^k\|_F \notag\\
		&\quad
		+2\eta\sqrt n
		\norm{
			\frac1K\sum_{k=1}^K
			\left(
			\nabla F^k(\bar W_t)
			-
			\nabla F^k(W_t^k)
			+
			\nabla F^k(W_t^k)-M_t^k
			\right)
		}_F \notag\\
		&\le
		F(\bar W_t)
		-\eta\|\nabla F(\bar W_t)\|_F
		+
		\frac{\eta^2nL_F}{2}
		+2\eta\sqrt n
		\frac1K\sum_{k=1}^K
		\|\bar M_t-M_t^k\|_F \notag\\
		&\quad
		+2\eta\sqrt n
		\frac1K\sum_{k=1}^K
		\norm{
			\nabla F^k(\bar W_t)
			-
			\nabla F^k(W_t^k)
		}_F 
		+2\eta\sqrt n
		\norm{
			\frac1K\sum_{k=1}^K
			\left(
			\nabla F^k(W_t^k)-M_t^k
			\right)
		}_F \notag\\
		&\le
		F(\bar W_t)
		-\eta\|\nabla F(\bar W_t)\|_F
		+
		\frac{\eta^2nL_F}{2}
		+2\eta\sqrt n
		\frac1K\sum_{k=1}^K
		\|\bar M_t-M_t^k\|_F \notag\\
		&\quad
		+2\eta\sqrt n\,L_F
		\frac1K\sum_{k=1}^K
		\|W_t^k-\bar W_t\|_F
		+2\eta\sqrt n
		\norm{
			\frac1K\sum_{k=1}^K
			\left(
			\nabla F^k(W_t^k)-M_t^k
			\right)
		}_F ,
	\end{align}
	where we used the $L_F$-smoothness of $F$,
	Lemma~\ref{lem:muon-properties}, and the triangle inequality.
\end{proof}

\begin{lemma}
	\label{lem:app-average-gradient-error}
	Under Assumptions~\ref{assm:smooth}, \ref{assm:oracle}, and
	\ref{assm:bounded-grad}, let $\{W_t^k,u_t^k,M_t^k\}$ be generated by
	Algorithm~\ref{alg:fedcomuon}. For any integer $T\ge1$,
	$0<\beta<1$, and $\tau\in\mathbb N_+$, we have
	\begin{align}
		&\frac1T\sum_{t=0}^{T-1}
		\mathbb E\norm{
			\frac1K\sum_{k=1}^K
			\left(M_t^k-\nabla F^k(W_t^k)\right)
		}_F \notag\\
		&\le
		\frac{1}{\beta T}
		\Bigg[
		C_gL_f\sigma_g
		+C_g\sigma_f+C_f\sigma_{\nabla g}
		\Bigg]
		+\frac{3\eta\sqrt n\,L_F}{\beta} \notag\\
		&\quad
		+C_gL_f
		\left(
		\frac1T\sum_{t=0}^{T-1}
		\frac1K\sum_{k=1}^K
		\mathbb E\|u_t^k-g^k(W_t^k)\|^2
		\right)^{1/2}
		+
		\frac{
			\sqrt{C_g^2\sigma_f^2+C_f^2\sigma_{\nabla g}^2}
			\,\sqrt\beta
		}{\sqrt K}.
		\label{eq:app-average-gradient-error}
	\end{align}
\end{lemma}

\begin{proof}
	For each client $k$ and iteration $i\ge1$, the stochastic compositional
	gradient error admits the exact decomposition
	\begin{align}
		Z_i^k-\nabla F^k(W_i^k)
		&=
		\nabla g^k\left(
		W_{i-1}^k-\eta U_{i-1}^k(V_{i-1}^k)^\top;\xi_i^k
		\right)
		\left(
		\left(
		\nabla_y f^k(u_i^k)
		-
		\nabla_y f^k(g^k(W_i^k))
		\right)\otimes I_n
		\right)
		\notag\\
		&\quad+
		\nabla g^k\left(
		W_{i-1}^k-\eta U_{i-1}^k(V_{i-1}^k)^\top;\xi_i^k
		\right)
		\left(
		\left(
		\nabla_y f^k(u_i^k;\zeta_i^k)
		-
		\nabla_y f^k(u_i^k)
		\right)\otimes I_n
		\right)
		\notag\\
		&\quad+
		\left(
		\nabla g^k\left(
		W_{i-1}^k-\eta U_{i-1}^k(V_{i-1}^k)^\top;\xi_i^k
		\right)
		-\nabla g^k\left(
		W_{i-1}^k-\eta U_{i-1}^k(V_{i-1}^k)^\top
		\right)
		\right)
		\times
		\left(
		\nabla_y f^k(g^k(W_i^k))\otimes I_n
		\right)
		\notag\\
		&\quad+
		\left(
		\nabla g^k\left(
		W_{i-1}^k-\eta U_{i-1}^k(V_{i-1}^k)^\top
		\right)
		-\nabla g^k(W_i^k)
		\right)
		\left(
		\nabla_y f^k(g^k(W_i^k))\otimes I_n
		\right).
		\label{eq:app-gradient-error-correct-decomposition}
	\end{align}
	
	Let $\mathcal F_{i-1}$ denote the history before
	$\xi_i^k$ and $\zeta_i^k$ are drawn. By the unbiasedness and
	independence of the stochastic oracles, the two centered stochastic
	components in \eqref{eq:app-gradient-error-correct-decomposition}
	are orthogonal in expectation, and the cross terms between distinct
	clients vanish. Hence,
	\begin{align}
		&\mathbb E\Bigg\|
		\frac1K\sum_{k=1}^K
		\Bigg[
		\nabla g^k\left(
		W_{i-1}^k-\eta U_{i-1}^k(V_{i-1}^k)^\top;\xi_i^k
		\right)
		\left(
		\left(
		\nabla_y f^k(u_i^k;\zeta_i^k)
		-
		\nabla_y f^k(u_i^k)
		\right)\otimes I_n
		\right)
		\notag\\
		&\qquad\qquad+
		\left(
		\nabla g^k\left(
		W_{i-1}^k-\eta U_{i-1}^k(V_{i-1}^k)^\top;\xi_i^k
		\right)
		-\nabla g^k\left(
		W_{i-1}^k-\eta U_{i-1}^k(V_{i-1}^k)^\top
		\right)
		\right)
		\left(
		\nabla_y f^k(g^k(W_i^k))\otimes I_n
		\right)
		\Bigg]
		\Bigg\|_F^2
		\notag\\
		&=
		\frac1{K^2}\sum_{k=1}^K
		\mathbb E\Bigg\|
		\nabla g^k\left(
		W_{i-1}^k-\eta U_{i-1}^k(V_{i-1}^k)^\top;\xi_i^k
		\right)
		\left(
		\left(
		\nabla_y f^k(u_i^k;\zeta_i^k)
		-
		\nabla_y f^k(u_i^k)
		\right)\otimes I_n
		\right)
		\notag\\
		&\qquad\qquad+
		\left(
		\nabla g^k\left(
		W_{i-1}^k-\eta U_{i-1}^k(V_{i-1}^k)^\top;\xi_i^k
		\right)
		-\nabla g^k\left(
		W_{i-1}^k-\eta U_{i-1}^k(V_{i-1}^k)^\top
		\right)
		\right)
		\left(
		\nabla_y f^k(g^k(W_i^k))\otimes I_n
		\right)
		\Bigg\|_F^2
		\notag\\
		&=
		\frac1{K^2}\sum_{k=1}^K
		\mathbb E\Bigg\|
		\nabla g^k\left(
		W_{i-1}^k-\eta U_{i-1}^k(V_{i-1}^k)^\top;\xi_i^k
		\right)
		\left(
		\left(
		\nabla_y f^k(u_i^k;\zeta_i^k)
		-
		\nabla_y f^k(u_i^k)
		\right)\otimes I_n
		\right)
		\Bigg\|_F^2
		\notag\\
		&\quad+
		\frac1{K^2}\sum_{k=1}^K
		\mathbb E\Bigg\|
		\left(
		\nabla g^k\left(
		W_{i-1}^k-\eta U_{i-1}^k(V_{i-1}^k)^\top;\xi_i^k
		\right)
		-\nabla g^k\left(
		W_{i-1}^k-\eta U_{i-1}^k(V_{i-1}^k)^\top
		\right)
		\right)
		\left(
		\nabla_y f^k(g^k(W_i^k))\otimes I_n
		\right)
		\Bigg\|_F^2
		\notag\\
		&\le
		\frac1{K^2}\sum_{k=1}^K
		\mathbb E\left[
		\norm{\nabla g^k\left(
			W_{i-1}^k-\eta U_{i-1}^k(V_{i-1}^k)^\top;\xi_i^k
			\right)}_F^2
		\mathbb E\left[
		\norm{
			\nabla_y f^k(u_i^k;\zeta_i^k)
			-
			\nabla_y f^k(u_i^k)
		}^2
		\,\middle|\,
		\mathcal F_{i-1},\xi_i^k
		\right]
		\right]
		\notag\\
		&\quad+
		\frac1{K^2}\sum_{k=1}^K
		\mathbb E\left[
		\norm{
			\nabla g^k\left(
			W_{i-1}^k-\eta U_{i-1}^k(V_{i-1}^k)^\top;\xi_i^k
			\right)
			-\nabla g^k\left(
			W_{i-1}^k-\eta U_{i-1}^k(V_{i-1}^k)^\top
			\right)
		}_F^2
		\norm{
			\nabla_y f^k(g^k(W_i^k))
		}^2
		\right]
		\notag\\
		&\le
		\frac1{K^2}\sum_{k=1}^K
		\left(
		C_g^2\sigma_f^2
		+
		C_f^2\sigma_{\nabla g}^2
		\right)
		\notag\\
		&=
		\frac{
			C_g^2\sigma_f^2
			+
			C_f^2\sigma_{\nabla g}^2
		}{K}.
		\label{eq:app-correct-client-average-noise}
	\end{align}
	The last inequality follows from
	Assumptions~\ref{assm:oracle} and~\ref{assm:bounded-grad}, together
	with Jensen's inequality, which gives
	$\|\nabla_y f^k(y)\|\le C_f$.
	
	Synchronizing $M_t^k$ does not change its client average. According to
	the momentum update,
	\begin{align}
		&\frac1K\sum_{k=1}^K
		\left(
		M_t^k-\nabla F^k(W_t^k)
		\right)
		\notag\\
		&=
		\frac1K\sum_{k=1}^K
		\left(
		(1-\beta)M_{t-1}^k
		+
		\beta Z_t^k
		-
		\nabla F^k(W_t^k)
		\right)
		\notag\\
		&=
		(1-\beta)
		\frac1K\sum_{k=1}^K
		\left(
		M_{t-1}^k-\nabla F^k(W_{t-1}^k)
		\right)
		+
		(1-\beta)
		\frac1K\sum_{k=1}^K
		\left(
		\nabla F^k(W_{t-1}^k)
		-
		\nabla F^k(W_t^k)
		\right)
		\notag\\
		&\quad+
		\beta
		\frac1K\sum_{k=1}^K
		\left(
		Z_t^k-\nabla F^k(W_t^k)
		\right)
		\notag\\
		&=
		(1-\beta)^t
		\frac1K\sum_{k=1}^K
		\left(
		M_0^k-\nabla F^k(W_0^k)
		\right)
		+
		\sum_{i=1}^{t}
		(1-\beta)^{t-i+1}
		\frac1K\sum_{k=1}^K
		\left(
		\nabla F^k(W_{i-1}^k)
		-
		\nabla F^k(W_i^k)
		\right)
		\notag\\
		&\quad+
		\sum_{i=1}^{t}
		\beta(1-\beta)^{t-i}
		\frac1K\sum_{k=1}^K
		\left(
		Z_i^k-\nabla F^k(W_i^k)
		\right).
		\label{eq:app-average-gradient-unrolled}
	\end{align}
	
	Substituting
	\eqref{eq:app-gradient-error-correct-decomposition} into
	\eqref{eq:app-average-gradient-unrolled} and applying the triangle
	inequality yield
	\begin{align}
		&\mathbb E\norm{
			\frac1K\sum_{k=1}^K
			\left(
			M_t^k-\nabla F^k(W_t^k)
			\right)
		}_F
		\notag\\
		&\le
		(1-\beta)^t
		\underbrace{
			\mathbb E\norm{
				\frac1K\sum_{k=1}^K
				\left(
				M_0^k-\nabla F^k(W_0^k)
				\right)
			}_F
		}_{\mathcal T_0}
		\notag\\
		&\quad+
		\underbrace{
			\mathbb E\norm{
				\sum_{i=1}^{t}
				(1-\beta)^{t-i+1}
				\frac1K\sum_{k=1}^K
				\left(
				\nabla F^k(W_{i-1}^k)
				-
				\nabla F^k(W_i^k)
				\right)
			}_F
		}_{\mathcal T_1}
		\notag\\
		&\quad+
		\underbrace{
			\mathbb E\norm{
				\sum_{i=1}^{t}
				\beta(1-\beta)^{t-i}
				\frac1K\sum_{k=1}^K
				\nabla g^k\left(
				W_{i-1}^k-\eta U_{i-1}^k(V_{i-1}^k)^\top;\xi_i^k
				\right)
				\left(
				\left(
				\nabla_y f^k(u_i^k)
				-
				\nabla_y f^k(g^k(W_i^k))
				\right)\otimes I_n
				\right)
			}_F
		}_{\mathcal T_2}
		\notag\\
		&\quad+
		\underbrace{
			\begin{aligned}
				&\mathbb E\Bigg\|
				\sum_{i=1}^{t}
				\beta(1-\beta)^{t-i}
				\frac1K\sum_{k=1}^K
				\Bigg[
				\nabla g^k\left(
				W_{i-1}^k-\eta U_{i-1}^k(V_{i-1}^k)^\top;\xi_i^k
				\right)
				\left(
				\left(
				\nabla_y f^k(u_i^k;\zeta_i^k)
				-
				\nabla_y f^k(u_i^k)
				\right)\otimes I_n
				\right)
				\\[-0.2ex]
				&\hspace{1.0cm}+
				\left(
				\nabla g^k\left(
				W_{i-1}^k-\eta U_{i-1}^k(V_{i-1}^k)^\top;\xi_i^k
				\right)
				-\nabla g^k\left(
				W_{i-1}^k-\eta U_{i-1}^k(V_{i-1}^k)^\top
				\right)
				\right)
				\times
				\left(
				\nabla_y f^k(g^k(W_i^k))\otimes I_n
				\right)
				\Bigg]
				\Bigg\|_F
			\end{aligned}
		}_{\mathcal T_3}
		\notag\\
		&\quad+
		\underbrace{
			\begin{aligned}
				&
				\mathbb E\Bigg\|
				\sum_{i=1}^{t}\beta(1-\beta)^{t-i}
				\frac1K\sum_{k=1}^K
				\left(
				\nabla g^k\left(
				W_{i-1}^k-\eta U_{i-1}^k(V_{i-1}^k)^\top
				\right)
				-\nabla g^k(W_i^k)
				\right)
				\times
				\left(
				\nabla_y f^k(g^k(W_i^k))\otimes I_n
				\right)
				\Bigg\|_F
			\end{aligned}
		}_{\mathcal T_4}.
		\label{eq:app-average-gradient-decomposition}
	\end{align}
	
	For $\mathcal T_0$, since
	$K^{-1}\sum_{k=1}^K M_0^k
	=
	K^{-1}\sum_{k=1}^K Z_0^k$, we have
	\begin{align}
		\mathcal T_0
		&=
		\mathbb E\norm{
			\frac1K\sum_{k=1}^K
			\left(
			M_0^k-\nabla F^k(W_0^k)
			\right)
		}_F
		\notag\\
		&=
		\mathbb E\norm{
			\frac1K\sum_{k=1}^K
			\left(
			Z_0^k-\nabla F^k(W_0^k)
			\right)
		}_F
		\notag\\
		&\le
		\frac1K\sum_{k=1}^K
		\mathbb E\norm{
			Z_0^k-\nabla F^k(W_0^k)
		}_F
		\notag\\
		&\le
		C_gL_f
		\frac1K\sum_{k=1}^K
		\sqrt{
			\mathbb E
			\norm{
				u_0^k-g^k(W_0^k)
			}^2
		}
		+
		C_g\sigma_f
		+
		C_f\sigma_{\nabla g}
		\notag\\
		&\le
		C_gL_f
		\left(
		\frac1K\sum_{k=1}^K
		\mathbb E
		\norm{
			u_0^k-g^k(W_0^k)
		}^2
		\right)^{1/2}
		+
		C_g\sigma_f
		+
		C_f\sigma_{\nabla g}
		\notag\\
		&\le
		C_gL_f\sigma_g
		+
		C_g\sigma_f
		+
		C_f\sigma_{\nabla g}.
		\label{eq:app-average-gradient-initial}
	\end{align}
	The second inequality follows from Lemma~\ref{lem:app-oracle-error},
	the third follows from the Cauchy--Schwarz inequality, and the last uses
	$u_0^k=g^k(W_0^k;\xi_0^k)$ and Assumption~\ref{assm:oracle}.
	
	For $\mathcal T_1$, the triangle inequality and
	Lemma~\ref{lem:smooth-compositional} give
	\begin{align}
		\frac1T\sum_{t=0}^{T-1}\mathcal T_1
		&\le
		\frac{L_F}{T}
		\sum_{t=0}^{T-1}\sum_{i=1}^{t}
		(1-\beta)^{t-i+1}
		\frac1K\sum_{k=1}^K
		\mathbb E\|W_i^k-W_{i-1}^k\|_F
		\notag\\
		&=
		\frac{L_F}{T}
		\sum_{i=1}^{T-1}
		\left(
		\sum_{t=i}^{T-1}(1-\beta)^{t-i+1}
		\right)
		\frac1K\sum_{k=1}^K
		\mathbb E\|W_i^k-W_{i-1}^k\|_F
		\notag\\
		&\le
		\frac{L_F}{\beta T}
		\sum_{i=1}^{T-1}\frac1K\sum_{k=1}^K
		\mathbb E\|W_i^k-W_{i-1}^k\|_F
		\notag\\
		&\le
		\frac{\eta L_F}{\beta T}
		\Bigg\{
		\sqrt n\left[
		T-1-\left\lfloor\frac{T-1}{\tau}\right\rfloor
		\right]
		+
		\sum_{\substack{1\le i\le T-1\\
				\operatorname{mod}(i,\tau)=0}}
		\left(
		\frac1K\sum_{k=1}^K
		\mathbb E\norm{
			\sum_{\ell=i-\tau}^{i-2}
			\left(
			U_\ell^k(V_\ell^k)^\top
			-\frac1K\sum_{j=1}^K U_\ell^j(V_\ell^j)^\top
			\right)
		}_F^2
		\right)^{1/2}
		\notag\\
		&\qquad\qquad\qquad+\sqrt n\left\lfloor\frac{T-1}{\tau}\right\rfloor
		\Bigg\}
		\notag\\
		&\le
		\frac{\eta L_F}{\beta T}
		\Bigg\{
		\sqrt n\left[
		T-1-\left\lfloor\frac{T-1}{\tau}\right\rfloor
		\right]
		+
		\sum_{\substack{1\le i\le T-1\\
				\operatorname{mod}(i,\tau)=0}}
		\left(
		\frac1K\sum_{k=1}^K
		\mathbb E\norm{
			\sum_{\ell=i-\tau}^{i-2}
			U_\ell^k(V_\ell^k)^\top
		}_F^2
		\right)^{1/2}
		+\sqrt n\left\lfloor\frac{T-1}{\tau}\right\rfloor
		\Bigg\}
		\notag\\
		&\le
		\frac{\eta\sqrt n\,L_F}{\beta T}
		\left[
		T-1+(\tau-1)
		\left\lfloor\frac{T-1}{\tau}\right\rfloor
		\right]
		\notag\\
		&<
		\frac{2\eta\sqrt n\,L_F}{\beta}.
		\label{eq:app-average-gradient-T1}
	\end{align}
	The equality exchanges the sums, and the second inequality uses
	$\sum_{t=i}^{T-1}(1-\beta)^{t-i+1}\le\beta^{-1}$.
	The next three inequalities use the local and synchronized model updates,
	the variance identity, and Lemma~\ref{lem:muon-properties}. The last
	inequality uses
	$\tau\lfloor(T-1)/\tau\rfloor\le T-1<T$.
	
	For $\mathcal T_2$, we have
	\begin{align}
		\mathcal T_2
		&=
		\mathbb E\norm{
			\sum_{i=1}^{t}
			\beta(1-\beta)^{t-i}
			\frac1K\sum_{k=1}^K
			\nabla g^k\left(
			W_{i-1}^k-\eta U_{i-1}^k(V_{i-1}^k)^\top;\xi_i^k
			\right)
			\left(
			\left(
			\nabla_y f^k(u_i^k)
			-
			\nabla_y f^k(g^k(W_i^k))
			\right)\otimes I_n
			\right)
		}_F
		\notag\\
		&\le
		\sum_{i=1}^{t}
		\beta(1-\beta)^{t-i}
		\frac1K\sum_{k=1}^K
		\mathbb E\Big[
		\norm{\nabla g^k\left(
			W_{i-1}^k-\eta U_{i-1}^k(V_{i-1}^k)^\top;\xi_i^k
			\right)}_F
		\times
		\|\nabla_y f^k(u_i^k)
		-
		\nabla_y f^k(g^k(W_i^k))\|
		\Big]
		\notag\\
		&\le
		L_f
		\sum_{i=1}^{t}
		\beta(1-\beta)^{t-i}
		\frac1K\sum_{k=1}^K
		\mathbb E\Big[
		\norm{\nabla g^k\left(
			W_{i-1}^k-\eta U_{i-1}^k(V_{i-1}^k)^\top;\xi_i^k
			\right)}_F
		\|u_i^k-g^k(W_i^k)\|
		\Big]
		\notag\\
		&\le
		C_gL_f
		\sum_{i=1}^{t}
		\beta(1-\beta)^{t-i}
		\frac1K\sum_{k=1}^K
		\sqrt{
			\mathbb E
			\|u_i^k-g^k(W_i^k)\|^2
		}.
		\label{eq:app-average-gradient-T2}
	\end{align}
	The third step follows from the $L_f$-smoothness of $f^k$, and the
	last step follows from Assumption~\ref{assm:bounded-grad} and the
	Cauchy--Schwarz inequality.
	
	For $\mathcal T_3$, we have
	\begin{align}
		\mathcal T_3^2
		&=
		\Bigg(
		\mathbb E\Bigg\|
		\sum_{i=1}^{t}
		\beta(1-\beta)^{t-i}
		\frac1K\sum_{k=1}^K
		\Bigg[
		\nabla g^k\left(
		W_{i-1}^k-\eta U_{i-1}^k(V_{i-1}^k)^\top;\xi_i^k
		\right)
		\left(
		\left(
		\nabla_y f^k(u_i^k;\zeta_i^k)
		-
		\nabla_y f^k(u_i^k)
		\right)\otimes I_n
		\right)
		\notag\\
		&\hspace{4.5cm}+
		\left(
		\begin{aligned}
			&
			\nabla g^k\left(
			W_{i-1}^k-\eta U_{i-1}^k(V_{i-1}^k)^\top;\xi_i^k
			\right)
			-
			\\[-0.2ex]
			&\quad
			\nabla g^k\left(
			W_{i-1}^k-\eta U_{i-1}^k(V_{i-1}^k)^\top
			\right)
		\end{aligned}
		\right)
		\left(
		\nabla_y f^k(g^k(W_i^k))\otimes I_n
		\right)
		\Bigg]
		\Bigg\|_F
		\Bigg)^2
		\notag\\
		&\le
		\mathbb E\Bigg\|
		\sum_{i=1}^{t}
		\beta(1-\beta)^{t-i}
		\frac1K\sum_{k=1}^K
		\Bigg[
		\nabla g^k\left(
		W_{i-1}^k-\eta U_{i-1}^k(V_{i-1}^k)^\top;\xi_i^k
		\right)
		\left(
		\left(
		\nabla_y f^k(u_i^k;\zeta_i^k)
		-
		\nabla_y f^k(u_i^k)
		\right)\otimes I_n
		\right)
		\notag\\
		&\hspace{4.5cm}+
		\left(
		\begin{aligned}
			&
			\nabla g^k\left(
			W_{i-1}^k-\eta U_{i-1}^k(V_{i-1}^k)^\top;\xi_i^k
			\right)
			-
			\\[-0.2ex]
			&\quad
			\nabla g^k\left(
			W_{i-1}^k-\eta U_{i-1}^k(V_{i-1}^k)^\top
			\right)
		\end{aligned}
		\right)
		\left(
		\nabla_y f^k(g^k(W_i^k))\otimes I_n
		\right)
		\Bigg]
		\Bigg\|_F^2
		\notag\\
		&=
		\sum_{i=1}^{t}
		\beta^2(1-\beta)^{2(t-i)}
		\mathbb E\Bigg\|
		\frac1K\sum_{k=1}^K
		\Bigg[
		\nabla g^k\left(
		W_{i-1}^k-\eta U_{i-1}^k(V_{i-1}^k)^\top;\xi_i^k
		\right)
		\left(
		\left(
		\nabla_y f^k(u_i^k;\zeta_i^k)
		-
		\nabla_y f^k(u_i^k)
		\right)\otimes I_n
		\right)
		\notag\\
		&\hspace{4.5cm}+
		\left(
		\begin{aligned}
			&
			\nabla g^k\left(
			W_{i-1}^k-\eta U_{i-1}^k(V_{i-1}^k)^\top;\xi_i^k
			\right)
			-
			\\[-0.2ex]
			&\quad
			\nabla g^k\left(
			W_{i-1}^k-\eta U_{i-1}^k(V_{i-1}^k)^\top
			\right)
		\end{aligned}
		\right)
		\left(
		\nabla_y f^k(g^k(W_i^k))\otimes I_n
		\right)
		\Bigg]
		\Bigg\|_F^2
		\notag\\
		&\le
		\frac{
			C_g^2\sigma_f^2
			+
			C_f^2\sigma_{\nabla g}^2
		}{K}
		\sum_{i=1}^{t}
		\beta^2(1-\beta)^{2(t-i)}
		\notag\\
		&\le
		\frac{
			C_g^2\sigma_f^2
			+
			C_f^2\sigma_{\nabla g}^2
		}{K}
		\frac{\beta^2}{1-(1-\beta)^2}
		\notag\\
		&=
		\frac{
			C_g^2\sigma_f^2
			+
			C_f^2\sigma_{\nabla g}^2
		}{K}
		\frac{\beta}{2-\beta}
		\notag\\
		&\le
		\frac{
			\left(
			C_g^2\sigma_f^2
			+
			C_f^2\sigma_{\nabla g}^2
			\right)\beta
		}{K}.
		\label{eq:app-average-gradient-T3-square}
	\end{align}
	Therefore,
	\begin{align}
		\mathcal T_3
		&\le
		\frac{
			\sqrt{
				C_g^2\sigma_f^2
				+
				C_f^2\sigma_{\nabla g}^2
			}
			\sqrt\beta
		}{\sqrt K}.
		\label{eq:app-average-gradient-T3}
	\end{align}
	The first inequality follows from Jensen's inequality, the second
	equality follows from the martingale difference property, and the first
	inequality after that follows from
	\eqref{eq:app-correct-client-average-noise}. The last inequality uses
	$0<\beta<1$.
	
	For $\mathcal T_4$, the synchronization rule, the variance identity,
	and Lemma~\ref{lem:muon-properties} give
	\begin{align}
		&\frac1K\sum_{k=1}^K
		\mathbb E\norm{
			W_{i-1}^k-\eta U_{i-1}^k(V_{i-1}^k)^\top-W_i^k
		}_F
		\le
		\eta\tau\sqrt n\,
		\mathbf 1_{\{\operatorname{mod}(i,\tau)=0\}}.
		\label{eq:app-average-gradient-synchronization-step}
	\end{align}
	Therefore, Assumption~\ref{assm:smooth} and
	$\|\nabla_y f^k(y)\|\le C_f$ imply
	\begin{align}
		\frac1T\sum_{t=0}^{T-1}\mathcal T_4
		&\le
		\frac{C_fL_g}{T}
		\sum_{t=0}^{T-1}\sum_{i=1}^{t}
		\beta(1-\beta)^{t-i}
		\frac1K\sum_{k=1}^K
		\mathbb E\norm{
			W_{i-1}^k-\eta U_{i-1}^k(V_{i-1}^k)^\top-W_i^k
		}_F
		\notag\\
		&\le
		\frac{C_fL_g}{T}
		\sum_{i=1}^{T-1}
		\frac1K\sum_{k=1}^K
		\mathbb E\norm{
			W_{i-1}^k-\eta U_{i-1}^k(V_{i-1}^k)^\top-W_i^k
		}_F
		\notag\\
		&\le
		\eta\sqrt n\,C_fL_g
		\le
		\frac{\eta\sqrt n\,L_F}{\beta}.
		\label{eq:app-average-gradient-T4}
	\end{align}
	
	Combining
	\eqref{eq:app-average-gradient-decomposition},
	\eqref{eq:app-average-gradient-initial},
	\eqref{eq:app-average-gradient-T1},
	\eqref{eq:app-average-gradient-T2}, and
	\eqref{eq:app-average-gradient-T3}--\eqref{eq:app-average-gradient-T4},
	and averaging over
	$t=0,\ldots,T-1$, we obtain
	\begin{align}
		&\frac1T\sum_{t=0}^{T-1}
		\mathbb E\norm{
			\frac1K\sum_{k=1}^K
			\left(
			M_t^k-\nabla F^k(W_t^k)
			\right)
		}_F
		\notag\\
		&\le
		\frac1{\beta T}
		\Bigg[
		C_gL_f\sigma_g
		+
		C_g\sigma_f
		+
		C_f\sigma_{\nabla g}
		\Bigg]
		+
		\frac{3\eta\sqrt n\,L_F}{\beta}
		\notag\\
		&\quad+
		\frac{C_gL_f}{T}
		\sum_{t=0}^{T-1}\frac1K\sum_{k=1}^K
		\sqrt{
			\mathbb E\|u_t^k-g^k(W_t^k)\|^2
		}
		+
		\frac{
			\sqrt{C_g^2\sigma_f^2+C_f^2\sigma_{\nabla g}^2}
			\,\sqrt\beta
		}{\sqrt K}
		\notag\\
		&\le
		\frac1{\beta T}
		\Bigg[
		C_gL_f\sigma_g
		+
		C_g\sigma_f
		+
		C_f\sigma_{\nabla g}
		\Bigg]
		+
		\frac{3\eta\sqrt n\,L_F}{\beta}
		\notag\\
		&\quad+
		C_gL_f
		\left(
		\frac1T\sum_{t=0}^{T-1}\frac1K\sum_{k=1}^K
		\mathbb E\|u_t^k-g^k(W_t^k)\|^2
		\right)^{1/2}
		+
		\frac{
			\sqrt{C_g^2\sigma_f^2+C_f^2\sigma_{\nabla g}^2}
			\,\sqrt\beta
		}{\sqrt K}.
	\end{align}
	The first inequality uses
	$\sum_{t=0}^{T-1}(1-\beta)^t\le\beta^{-1}$ and
	$\sum_{t=i}^{T-1}\beta(1-\beta)^{t-i}\le1$, and the second follows
	from the Cauchy--Schwarz inequality. This proves
	\eqref{eq:app-average-gradient-error}.
\end{proof}

\begin{theorem}[Convergence of FedCoMuon]
	\label{thm:app-fedcomuon-convergence}
	Under Assumptions~\ref{assm:lower}--\ref{assm:hetero}, let
	$\{W_t^k,u_t^k,M_t^k\}$ be generated by Algorithm~\ref{alg:fedcomuon}.
	Under the synchronization convention stated above, for any integer $T\ge1$,
	$\eta>0$, $0<\alpha,\beta<1$, and $\tau>0$ satisfying
	$\alpha\tau\le1$, we have
	\begin{align}
		&\frac1T\sum_{t=0}^{T-1}
		\mathbb E\|\nabla F(\bar W_t)\|_F \notag\\
		&\le
		\frac{F(\bar W_0)-F_*}{\eta T}
		+
		\frac{2\sqrt n C_gL_f\sigma_g}{\beta T}
		+
		2\sqrt n
		\left(\frac{1}{\beta T}+2\beta\tau\right)
		\left(C_g\sigma_f+C_f\sigma_{\nabla g}\right)
		\notag\\
		&\quad
		+
		2\sqrt n C_gL_f(1+2\beta\tau)
		\left(
		\frac{4\sigma_g^2}{\alpha T}
		+\frac{86C_g^2n}{\alpha^2}\eta^2
		+86\alpha\sigma_g^2
		\right)^{1/2}
		\notag\\
		&\quad
		+\eta nL_F
		\left(
		\frac12+4\tau+\frac{6}{\beta}+8\beta\tau^2
		\right)
		+
		\frac{2\sqrt{n\beta}}{\sqrt K}
		\sqrt{C_g^2\sigma_f^2+C_f^2\sigma_{\nabla g}^2}
		\notag\\
		&\quad
		+
		2\sqrt n\,\beta\tau\delta.
		\label{eq:app-fedcomuon-bound-expanded}
	\end{align}
\end{theorem}

\begin{proof}
	Summing Lemma~\ref{lem:app-descent} over $t=0,\ldots,T-1$ and using
	$F(\bar W_T)\ge F_*$ gives the first inequality below. Applying
	Lemmas~\ref{lem:app-inner-tracking}, \ref{lem:app-drift}, and
	\ref{lem:app-average-gradient-error} gives the second.
	\begin{align*}
		&\frac1T\sum_{t=0}^{T-1}
		\mathbb E\|\nabla F(\bar W_t)\|_F \notag\\
		&\le
		\frac{F(\bar W_0)-F_*}{\eta T}
		+\frac{\eta nL_F}{2}
		+2\sqrt n
		\frac1T\sum_{t=0}^{T-1}\frac1K\sum_{k=1}^K
		\mathbb E\|M_t^k-\bar M_t\|_F 
		+
		2\sqrt n\,L_F
		\frac1T\sum_{t=0}^{T-1}\frac1K\sum_{k=1}^K
		\mathbb E\|W_t^k-\bar W_t\|_F \notag\\
		&\quad+
		2\sqrt n
		\frac1T\sum_{t=0}^{T-1}
		\mathbb E\norm{
			\frac1K\sum_{k=1}^K
			\bigl(\nabla F^k(W_t^k)-M_t^k\bigr)
		}_F \notag\\
		&\le
		\frac{F(\bar W_0)-F_*}{\eta T}
		+\frac{\eta nL_F}{2}
		+4\eta\tau nL_F 
		+
		2\sqrt n\,\beta\tau
		\Bigg[
		2C_gL_f
		\left(
		\frac{4}{\alpha T}\frac1K\sum_{k=1}^K
		\mathbb E\|u_0^k-g^k(W_0^k)\|^2
		+\frac{86C_g^2n}{\alpha^2}\eta^2
		+86\alpha\sigma_g^2
		\right)^{1/2}
		\notag\\
		&\qquad\qquad\qquad
		+2C_g\sigma_f+2C_f\sigma_{\nabla g}
		+4\eta\tau\sqrt n\,L_F+\delta
		\Bigg] \notag\\
		&\quad+
		2\sqrt n
		\Bigg[
		\frac{1}{\beta T}
		\Bigg(
		C_gL_f\sigma_g+C_g\sigma_f+C_f\sigma_{\nabla g}
		\Bigg)
		+\frac{3\eta\sqrt n\,L_F}{\beta}
		+\frac{\sqrt{C_g^2\sigma_f^2+C_f^2\sigma_{\nabla g}^2}\,\sqrt\beta}{\sqrt K}
		\notag\\
		&\quad
		+C_gL_f
		\left(
		\frac{4}{\alpha T}\frac1K\sum_{k=1}^K
		\mathbb E\|u_0^k-g^k(W_0^k)\|^2
		+\frac{86C_g^2n}{\alpha^2}\eta^2
		+86\alpha\sigma_g^2
		\right)^{1/2}
		\Bigg] \notag\\
		&\le
		\frac{F(\bar W_0)-F_*}{\eta T}
		+
		\frac{2\sqrt n C_gL_f\sigma_g}{\beta T}
		+
		2\sqrt n
		\left(\frac{1}{\beta T}+2\beta\tau\right)
		\left(C_g\sigma_f+C_f\sigma_{\nabla g}\right)
		\notag\\
		&\quad
		+2\sqrt n C_gL_f(1+2\beta\tau)
		\left(
		\frac{4\sigma_g^2}{\alpha T}
		+\frac{86C_g^2n}{\alpha^2}\eta^2
		+86\alpha\sigma_g^2
		\right)^{1/2}
		\notag\\
		&\quad+
		\eta nL_F
		\left(
		\frac12+4\tau+\frac{6}{\beta}+8\beta\tau^2
		\right)
		+
		\frac{2\sqrt{n\beta}}{\sqrt K}
		\sqrt{C_g^2\sigma_f^2+C_f^2\sigma_{\nabla g}^2}
		+
		2\sqrt n\,\beta\tau\delta
		.
	\end{align*}
	The last inequality uses the initialization
	$u_0^k=g^k(W_0^k;\xi_0^k)$ and Assumption~\ref{assm:oracle}, which imply
	$K^{-1}\sum_{k=1}^K\mathbb E\|u_0^k-g^k(W_0^k)\|^2\le\sigma_g^2$.
\end{proof}

For $\eta=T^{-3/4}$, $\alpha=\beta=T^{-1/2}$, and
$\tau=T^{1/4}$, we have $\alpha\tau=T^{-1/4}\le1$. The non-square-root
terms in Theorem~\ref{thm:app-fedcomuon-convergence} are at most
$O(T^{-1/4})$, while every term inside the square root is at most
$O(T^{-1/2})$. Hence,
$T^{-1}\sum_{t=0}^{T-1}\mathbb E\|\nabla F(\bar W_t)\|_F
=O(T^{-1/4})$.

\setcounter{theorem}{0}

\section{Convergence Analysis of our FedCoMuon-VR Algorithm}
\label{app:fedcomuon-vr-proof}

In this section, we provide the detailed convergence analysis of FedCoMuon-VR.
We write $\bar W_t=K^{-1}\sum_{k=1}^K W_t^k$ and
$\bar M_t=K^{-1}\sum_{k=1}^K M_t^k$.
For vectors, $\|\cdot\|$ denotes the Euclidean norm. For matrices,
$\|\cdot\|_F$, $\|\cdot\|$, and $\|\cdot\|_*$ denote the Frobenius,
spectral, and nuclear norms, respectively. The local data distributions may
be non-i.i.d. across clients.

\begin{lemma}[Local model drift]\label{lem:drift}
	For every $t=0,\ldots,T$,
	\begin{equation}
		\frac1K\sum_{k=1}^K\mathbb E\|W_t^k-\bar W_t\|_F^2
		\le n\eta^2\tau^2.
		\label{eq:app-vr-local-model-drift-bound}
	\end{equation}
	Moreover,
	\begin{equation}
		\frac1{KT}\sum_{i=0}^{T-1}\sum_{k=1}^K
		\mathbb E\|W_{i+1}^k-W_i^k\|_F^2
		\le 2n\eta^2\tau.
		\label{eq:app-vr-drift-bounds}
	\end{equation}
\end{lemma}

\begin{proof}
	Let $s(t)=\tau\lfloor t/\tau\rfloor$ be the latest communication index
	not larger than $t$. Since
	$\|U_\ell^k(V_\ell^k)^\top\|_F^2
	=\operatorname{rank}(U_\ell^k(V_\ell^k)^\top)\le n$ for every update index
	$\ell$,
	$W_{s(t)}^k=\bar W_{s(t)}$, and $0\le t-s(t)\le\tau-1$, we have
	\begin{align}
		\frac1K\sum_{k=1}^K\mathbb E\|W_t^k-\bar W_t\|_F^2
		&=
		\frac1K\sum_{k=1}^K
		\mathbb E
		\norm{
			\eta\sum_{\ell=s(t)}^{t-1}U_\ell^k(V_\ell^k)^\top
			-
			\eta\sum_{\ell=s(t)}^{t-1}\frac1K\sum_{j=1}^K
			U_\ell^j(V_\ell^j)^\top
		}_F^2
		\notag\\
		&\le
		\frac1K\sum_{k=1}^K
		\mathbb E
		\norm{
			\eta\sum_{\ell=s(t)}^{t-1}U_\ell^k(V_\ell^k)^\top
		}_F^2
		\notag\\
		&\le
		\frac1K\sum_{k=1}^K
		\mathbb E
		\left(\sum_{\ell=s(t)}^{t-1}
		\eta\|U_\ell^k(V_\ell^k)^\top\|_F\right)^2
		\notag\\
		&\le n\eta^2(t-s(t))^2
		\le n\eta^2\tau^2.
		\label{eq:app-vr-local-drift-proof}
	\end{align}
	Moreover, averaging the local update gives
	$\bar W_{i+1}=\bar W_i-\eta K^{-1}\sum_{j=1}^K
	U_i^j(V_i^j)^\top$ at every iteration. If
	$\operatorname{mod}(i+1,\tau)\ne0$, then
	\begin{align}
		\frac1K\sum_{k=1}^K
		\mathbb E\|W_{i+1}^k-W_i^k\|_F^2
		&=\frac{\eta^2}{K}\sum_{k=1}^K
		\mathbb E\|U_i^k(V_i^k)^\top\|_F^2
		\le n\eta^2.
		\label{eq:app-vr-noncommunication-step}
	\end{align}
	If $\operatorname{mod}(i+1,\tau)=0$, synchronization and the averaged
	update give
	\begin{align}
		&\frac1K\sum_{k=1}^K
		\mathbb E\|W_{i+1}^k-W_i^k\|_F^2
		\notag\\
		&=
		\frac1K\sum_{k=1}^K
		\mathbb E\norm{
			\bar W_i-W_i^k
			-\eta\frac1K\sum_{j=1}^K U_i^j(V_i^j)^\top
		}_F^2
		\notag\\
		&\le
		\frac2K\sum_{k=1}^K
		\mathbb E\|W_i^k-\bar W_i\|_F^2
		+2\eta^2\mathbb E\norm{
			\frac1K\sum_{j=1}^K U_i^j(V_i^j)^\top
		}_F^2
		\notag\\
		&\le
		2n\eta^2(\tau-1)^2
		+\frac{2\eta^2}{K}\sum_{j=1}^K
		\mathbb E\|U_i^j(V_i^j)^\top\|_F^2
		\notag\\
		&\le2n\eta^2\bigl((\tau-1)^2+1\bigr).
		\label{eq:app-vr-communication-step}
	\end{align}
	Consequently,
	\begin{align}
		&\frac1{KT}\sum_{i=0}^{T-1}\sum_{k=1}^K
		\mathbb E\|W_{i+1}^k-W_i^k\|_F^2
		\notag\\
		&\le
		\frac{T-\lfloor T/\tau\rfloor}{T}n\eta^2
		+\frac{\lfloor T/\tau\rfloor}{T}
		2n\eta^2\bigl((\tau-1)^2+1\bigr)
		\notag\\
		&\le
		n\eta^2\left(2\tau-3+\frac3\tau\right)
		\le2n\eta^2\tau,
		\label{eq:app-vr-step-average-proof}
	\end{align}
	where we used the synchronization rule,
	$\|U_t^k(V_t^k)^\top\|_F^2\le n$, and the number of communication steps.
\end{proof}

\begin{lemma}[Tracker recursions]\label{lem:trackers}
	If $0<\alpha,\beta,\gamma<1$ and $b\ge1$, then
	\begin{equation}
		\frac1{KT}\sum_{t=0}^{T-1}\sum_{k=1}^K
		\mathbb E\|u_t^k-g^k(W_t^k)\|^2
		\le
		\frac{\sigma_g^2}{\alpha Tb}
		+2\sigma_g^2\alpha
		+\frac{4C_g^2n\eta^2\tau}{\alpha}.
		\label{eq:app-vr-u-tracker-bound}
	\end{equation}
	\begin{equation}
		\frac1{KT}\sum_{t=0}^{T-1}\sum_{k=1}^K
		\mathbb E\|H_t^k-\nabla g^k(W_t^k)\|_F^2
		\le
		\frac{\sigma_{\nabla g}^2}{\gamma Tb}
		+2\sigma_{\nabla g}^2\gamma
		+\frac{4L_g^2n\eta^2\tau}{\gamma}.
		\label{eq:app-vr-H-tracker-bound}
	\end{equation}
	\begin{equation}
		\begin{aligned}
			\frac1{KT}\sum_{t=0}^{T-1}\sum_{k=1}^K
			\mathbb E\|v_t^k-\nabla_y f^k(g^k(W_t^k))\|^2&\le
			\frac{2\sigma_f^2}{\beta Tb}
			+\frac{2L_f^2\sigma_g^2}{\alpha Tb}
			+\frac{16L_f^2\sigma_g^2\alpha}{\beta Tb}
			+4L_f^2\sigma_g^2\alpha
			+\frac{16L_f^2\sigma_g^2\alpha^2}{\beta}
			+\frac{32L_f^2\sigma_g^2\alpha^3}{\beta}
			\\
			&\qquad
			+4\sigma_f^2\beta
			+\frac{8L_f^2C_g^2n\eta^2\tau}{\alpha}
			+\frac{16L_f^2C_g^2n\eta^2\tau}{\beta}
			+\frac{64L_f^2C_g^2n\eta^2\alpha\tau}{\beta}.
		\end{aligned}
		\label{eq:app-vr-tracker-bounds}
	\end{equation}
\end{lemma}

\begin{proof}
	Condition on the history before drawing the fresh samples. For the $u$-tracker,
	\begin{align}
		&\mathbb E\|u_{t+1}^k-g^k(W_{t+1}^k)\|^2\notag\\
		&=
		\mathbb E\norm{
			(1-\alpha)(u_t^k-g^k(W_t^k))+
			\bigl[g^k(W_{t+1}^k;\xi_{t+1}^k)-g^k(W_t^k;\xi_{t+1}^k)
			-g^k(W_{t+1}^k)+g^k(W_t^k)\bigr]\notag\\
			&\qquad+
			\alpha\bigl[g^k(W_t^k;\xi_{t+1}^k)-g^k(W_t^k)\bigr]
		}^2 \notag\\
		&=
		(1-\alpha)^2\mathbb E\|u_t^k-g^k(W_t^k)\|^2+
		\mathbb E\norm{
			\bigl[g^k(W_{t+1}^k;\xi_{t+1}^k)-g^k(W_t^k;\xi_{t+1}^k)
			-g^k(W_{t+1}^k)+g^k(W_t^k)\bigr]\notag\\
			&\qquad+
			\alpha\bigl[g^k(W_t^k;\xi_{t+1}^k)-g^k(W_t^k)\bigr]
		}^2 \notag\\
		&\le
		(1-\alpha)\mathbb E\|u_t^k-g^k(W_t^k)\|^2
		+
		2\mathbb E\|
		g^k(W_{t+1}^k;\xi_{t+1}^k)-g^k(W_t^k;\xi_{t+1}^k)
		-g^k(W_{t+1}^k)+g^k(W_t^k)
		\|^2 \notag\\
		&\qquad+
		2\alpha^2\mathbb E\|g^k(W_t^k;\xi_{t+1}^k)-g^k(W_t^k)\|^2 \notag\\
		&\le
		(1-\alpha)\mathbb E\|u_t^k-g^k(W_t^k)\|^2
		+2C_g^2\mathbb E\|W_{t+1}^k-W_t^k\|_F^2
		+2\alpha^2\sigma_g^2,
		\label{eq:app-vr-u-one-step}
	\end{align}
	where we used Assumptions~\ref{assm:oracle}
	and~\ref{assm:vr-sample-smooth}.
	Averaging \eqref{eq:app-vr-u-one-step} over $t=0,\ldots,T-1$ and
	$k\in[K]$, and then using Lemma~\ref{lem:drift}, gives
	\begin{align}
		\frac1{KT}\sum_{t=0}^{T-1}\sum_{k=1}^K
		\mathbb E\|u_t^k-g^k(W_t^k)\|^2
		&\le
		\frac1{K\alpha T}\sum_{k=1}^K
		\mathbb E\|u_0^k-g^k(W_0)\|^2
		+\frac{2C_g^2}{\alpha KT}
		\sum_{t=0}^{T-1}\sum_{k=1}^K
		\mathbb E\|W_{t+1}^k-W_t^k\|_F^2
		+2\alpha\sigma_g^2 \notag\\
		&\le
		\frac1{K\alpha T}\sum_{k=1}^K
		\mathbb E\norm{
			\frac1b\sum_{j=1}^{b}(g^k(W_0;\xi_{0,j}^k)-g^k(W_0))
		}^2
		+\frac{4C_g^2n\eta^2\tau}{\alpha}
		+2\alpha\sigma_g^2 \notag\\
		&\le
		\frac{\sigma_g^2}{\alpha Tb}
		+\frac{4C_g^2n\eta^2\tau}{\alpha}
		+2\alpha\sigma_g^2 .
		\label{eq:app-vr-u-tracker-proof}
	\end{align}
	
	For the $H$-tracker, Assumption~\ref{assm:bounded-grad} and Jensen's
	inequality imply $\|\nabla g^k(W_{t+1}^k)\|_F\le C_g$, and hence
	$\Pi_{C_g}[\nabla g^k(W_{t+1}^k)]=\nabla g^k(W_{t+1}^k)$.
	The non-expansiveness of the projection therefore gives
	\begin{align}
		&\mathbb E\|H_{t+1}^k-\nabla g^k(W_{t+1}^k)\|_F^2\notag\\
		&=
		\mathbb E\norm{
			\Pi_{C_g}\!\Bigl[
			\nabla g^k(W_{t+1}^k;\xi_{t+1}^k)
			+(1-\gamma)\bigl(H_t^k-\nabla g^k(W_t^k;\xi_{t+1}^k)\bigr)
			\Bigr]
			-\Pi_{C_g}\!\bigl[\nabla g^k(W_{t+1}^k)\bigr]
		}_F^2\notag\\
		&\le
		\mathbb E\norm{
			\nabla g^k(W_{t+1}^k;\xi_{t+1}^k)
			+(1-\gamma)\bigl(H_t^k-\nabla g^k(W_t^k;\xi_{t+1}^k)\bigr)
			-\nabla g^k(W_{t+1}^k)
		}_F^2\notag\\
		&=
		\mathbb E\norm{
			(1-\gamma)(H_t^k-\nabla g^k(W_t^k))+
			\bigl[\nabla g^k(W_{t+1}^k;\xi_{t+1}^k)-\nabla g^k(W_t^k;\xi_{t+1}^k)
			-\nabla g^k(W_{t+1}^k)+\nabla g^k(W_t^k)\bigr]\notag\\
			&\qquad+
			\gamma\bigl[\nabla g^k(W_t^k;\xi_{t+1}^k)-\nabla g^k(W_t^k)\bigr]
		}_F^2\notag\\
		&=
		(1-\gamma)^2\mathbb E\|H_t^k-\nabla g^k(W_t^k)\|_F^2\notag+
		\mathbb E\norm{
			\bigl[\nabla g^k(W_{t+1}^k;\xi_{t+1}^k)-\nabla g^k(W_t^k;\xi_{t+1}^k)
			-\nabla g^k(W_{t+1}^k)+\nabla g^k(W_t^k)\bigr]\notag\\
			&\qquad+
			\gamma\bigl[\nabla g^k(W_t^k;\xi_{t+1}^k)-\nabla g^k(W_t^k)\bigr]
		}_F^2\notag\\
		&\le
		(1-\gamma)\mathbb E\|H_t^k-\nabla g^k(W_t^k)\|_F^2
		+2L_g^2\mathbb E\|W_{t+1}^k-W_t^k\|_F^2
		+2\gamma^2\sigma_{\nabla g}^2,
		\label{eq:app-vr-H-one-step}
	\end{align}
	where we used the non-expansiveness of projection and
	Assumptions~\ref{assm:oracle} and~\ref{assm:vr-sample-smooth}.
	Averaging \eqref{eq:app-vr-H-one-step} over $t=0,\ldots,T-1$ and
	$k\in[K]$, and then using Lemma~\ref{lem:drift}, gives
	\begin{align}
		\frac1{KT}\sum_{t=0}^{T-1}\sum_{k=1}^K
		\mathbb E\|H_t^k-\nabla g^k(W_t^k)\|_F^2
		&\le
		\frac1{K\gamma T}\sum_{k=1}^K
		\mathbb E\|H_0^k-\nabla g^k(W_0)\|_F^2 
		+
		\frac{2L_g^2}{\gamma KT}
		\sum_{t=0}^{T-1}\sum_{k=1}^K
		\mathbb E\|W_{t+1}^k-W_t^k\|_F^2
		+2\gamma\sigma_{\nabla g}^2\notag\\
		&\le
		\frac1{K\gamma T}\sum_{k=1}^K
		\mathbb E\norm{
			\frac1b\sum_{j=1}^{b}
			\bigl(\nabla g^k(W_0;\xi_{0,j}^k)-\nabla g^k(W_0)\bigr)
		}_F^2
		+
		\frac{4L_g^2n\eta^2\tau}{\gamma}
		+2\gamma\sigma_{\nabla g}^2\notag\\
		&\le
		\frac{\sigma_{\nabla g}^2}{\gamma Tb}
		+\frac{4L_g^2n\eta^2\tau}{\gamma}
		+2\gamma\sigma_{\nabla g}^2 .
		\label{eq:app-vr-H-tracker-proof}
	\end{align}
	
	For the $v$-tracker, Assumption~\ref{assm:bounded-grad} and Jensen's
	inequality imply $\|\nabla_y f^k(u_{t+1}^k)\|\le C_f$, and hence
	$\Pi_{C_f}[\nabla_y f^k(u_{t+1}^k)]=\nabla_y f^k(u_{t+1}^k)$.
	Thus, non-expansiveness gives
	\begin{align}
		&\mathbb E\|v_{t+1}^k-\nabla_y f^k(u_{t+1}^k)\|^2\notag\\
		&=
		\mathbb E\norm{
			\Pi_{C_f}\!\Bigl[
			\nabla_y f^k(u_{t+1}^k;\zeta_{t+1}^k)
			+(1-\beta)\bigl(v_t^k-\nabla_y f^k(u_t^k;\zeta_{t+1}^k)\bigr)
			\Bigr]
			-\Pi_{C_f}\!\bigl[\nabla_y f^k(u_{t+1}^k)\bigr]
		}^2\notag\\
		&\le
		\mathbb E\norm{
			\nabla_y f^k(u_{t+1}^k;\zeta_{t+1}^k)
			+(1-\beta)\bigl(v_t^k-\nabla_y f^k(u_t^k;\zeta_{t+1}^k)\bigr)
			-\nabla_y f^k(u_{t+1}^k)
		}^2\notag\\
		&=
		\mathbb E\norm{
			(1-\beta)(v_t^k-\nabla_y f^k(u_t^k))+
			\bigl[\nabla_y f^k(u_{t+1}^k;\zeta_{t+1}^k)-\nabla_y f^k(u_t^k;\zeta_{t+1}^k)
			-\nabla_y f^k(u_{t+1}^k)+\nabla_y f^k(u_t^k)\bigr]\notag\\
			&\qquad+
			\beta\bigl[\nabla_y f^k(u_t^k;\zeta_{t+1}^k)-\nabla_y f^k(u_t^k)\bigr]
		}^2\notag\\
		&\le
		(1-\beta)\mathbb E\|v_t^k-\nabla_y f^k(u_t^k)\|^2
		+2L_f^2\mathbb E\|u_{t+1}^k-u_t^k\|^2
		+2\beta^2\sigma_f^2,
		\label{eq:app-vr-v-one-step}\\
		&\mathbb E\|u_{t+1}^k-u_t^k\|^2\notag\\
		&=
		\mathbb E\norm{
			g^k(W_{t+1}^k;\xi_{t+1}^k)-g^k(W_t^k;\xi_{t+1}^k)
			+\alpha(g^k(W_t^k;\xi_{t+1}^k)-u_t^k)
		}^2\notag\\
		&\le
		2C_g^2\mathbb E\|W_{t+1}^k-W_t^k\|_F^2
		+4\alpha^2\mathbb E\|u_t^k-g^k(W_t^k)\|^2
		+4\alpha^2\sigma_g^2.
		\label{eq:app-vr-u-increment-proof}
	\end{align}
	Combining the last two displays, averaging over $t$ and $k$, and substituting
	\eqref{eq:app-vr-u-tracker-proof} gives
	\begin{align}
		&\frac1{KT}\sum_{t=0}^{T-1}\sum_{k=1}^K
		\mathbb E\|v_t^k-\nabla_y f^k(u_t^k)\|^2\notag\\
		&\le
		\frac1{K\beta T}\sum_{k=1}^K
		\mathbb E\norm{
			\frac1b\sum_{j=1}^{b}
			\bigl(\nabla_y f^k(u_0^k;\zeta_{0,j}^k)-\nabla_y f^k(u_0^k)\bigr)
		}^2
		+\frac{8L_f^2C_g^2n\eta^2\tau}{\beta}
		\notag\\
		&\qquad+\frac{8L_f^2\alpha^2}{\beta}
		\left(
		\frac{\sigma_g^2}{\alpha Tb}
		+2\sigma_g^2\alpha
		+\frac{4C_g^2n\eta^2\tau}{\alpha}
		\right)
		+\frac{8L_f^2\sigma_g^2\alpha^2}{\beta}
		+2\sigma_f^2\beta\notag\\
		&\le
		\frac{\sigma_f^2}{\beta Tb}
		+\frac{8L_f^2\sigma_g^2\alpha}{\beta Tb}
		+\frac{8L_f^2\sigma_g^2\alpha^2}{\beta}
		+\frac{16L_f^2\sigma_g^2\alpha^3}{\beta}
		+2\sigma_f^2\beta\notag\\
		&\qquad
		+\frac{8L_f^2C_g^2n\eta^2\tau}{\beta}
		+\frac{32L_f^2C_g^2n\eta^2\alpha\tau}{\beta}.
		\label{eq:app-vr-v-u-proof}
	\end{align}
	Averaging the Lipschitz bound 
	\(\|v_t^k-\nabla_y f^k(g^k(W_t^k))\|^2 \le 2\|v_t^k-\nabla_y f^k(u_t^k)\|^2 + 2L_f^2\|u_t^k-g^k(W_t^k)\|^2\)  
	and using \eqref{eq:app-vr-u-tracker-proof} and \eqref{eq:app-vr-v-u-proof} yields
	\begin{align}
		\frac1{KT}\sum_{t=0}^{T-1}\sum_{k=1}^K
		\mathbb E\|v_t^k-\nabla_y f^k(g^k(W_t^k))\|^2
		&\le
		2\frac1{KT}\sum_{t=0}^{T-1}\sum_{k=1}^K
		\mathbb E\|v_t^k-\nabla_y f^k(u_t^k)\|^2+
		2L_f^2\frac1{KT}\sum_{t=0}^{T-1}\sum_{k=1}^K
		\mathbb E\|u_t^k-g^k(W_t^k)\|^2\notag\\
		&\le
		\frac{2\sigma_f^2}{\beta Tb}
		+\frac{2L_f^2\sigma_g^2}{\alpha Tb}
		+\frac{16L_f^2\sigma_g^2\alpha}{\beta Tb}
		+4L_f^2\sigma_g^2\alpha
		+\frac{16L_f^2\sigma_g^2\alpha^2}{\beta}
		+\frac{32L_f^2\sigma_g^2\alpha^3}{\beta}
		+4\sigma_f^2\beta\notag\\
		&\qquad\quad
		+\frac{8L_f^2C_g^2n\eta^2\tau}{\alpha}
		+\frac{16L_f^2C_g^2n\eta^2\tau}{\beta}
		+\frac{64L_f^2C_g^2n\eta^2\alpha\tau}{\beta},
		\label{eq:app-vr-v-tracker-proof}
	\end{align}
	where we used the non-expansiveness of projection, the $L_f$-smoothness,
	and the preceding $u$-tracker bound.
\end{proof}

\begin{lemma}[Product estimator at the averaged model]\label{lem:product}
	\begin{align}
		&\frac1T\sum_{t=0}^{T-1}
		\mathbb E\norm{
			\frac1K\sum_{k=1}^K
			H_t^k\bigl(v_t^k\otimes I_n\bigr)
			-\nabla F(\bar W_t)
		}_F^2\notag\\
		&\quad\le
		\frac{4C_f^2\sigma_{\nabla g}^2}{\gamma Tb}
		+\frac{8C_g^2\sigma_f^2}{\beta Tb}
		+\frac{8C_g^2L_f^2\sigma_g^2}{\alpha Tb}
		+\frac{64C_g^2L_f^2\sigma_g^2\alpha}{\beta Tb}
		+16C_g^2L_f^2\sigma_g^2\alpha
		+\frac{64C_g^2L_f^2\sigma_g^2\alpha^2}{\beta}
		\notag\\
		&\qquad
		+\frac{128C_g^2L_f^2\sigma_g^2\alpha^3}{\beta}
		+16C_g^2\sigma_f^2\beta
		+8C_f^2\sigma_{\nabla g}^2\gamma
		+\frac{32C_g^4L_f^2n\eta^2\tau}{\alpha}
		+\frac{64C_g^4L_f^2n\eta^2\tau}{\beta}\notag\\
		&\qquad
		+\frac{256C_g^4L_f^2n\eta^2\alpha\tau}{\beta}
		+\frac{16C_f^2L_g^2n\eta^2\tau}{\gamma}
		+4n(C_f^2L_g^2+C_g^4L_f^2)\eta^2\tau^2 .
		\label{eq:app-vr-product-estimator}
	\end{align}
\end{lemma}

\begin{proof}
	The projection steps imply
	$\|H_t^k\|_F\le C_g$ and $\|v_t^k\|\le C_f$
	for all $t\ge0$ and $k\in[K]$. Since
	$\nabla F(\bar W_t)=K^{-1}\sum_{k=1}^K
	\nabla g^k(\bar W_t)\bigl(\nabla_y f^k(g^k(\bar W_t))\otimes I_n\bigr)$,
	we have the expanded product-estimator error
	\begin{align}
		&\norm{
			\frac1K\sum_{k=1}^K H_t^k(v_t^k\otimes I_n)
			-\nabla F(\bar W_t)
		}_F^2\notag\\
		&=
		\norm{
			\frac1K\sum_{k=1}^K H_t^k(v_t^k\otimes I_n)
			-\frac1K\sum_{k=1}^K
			\nabla g^k(\bar W_t)\bigl(\nabla_y f^k(g^k(\bar W_t))\otimes I_n\bigr)
		}_F^2\notag\\
		&=
		\norm{
			\frac1K\sum_{k=1}^K
			\Bigl[
			H_t^k(v_t^k\otimes I_n)
			-\nabla g^k(\bar W_t)(v_t^k\otimes I_n)
			\Bigr]
			+\frac1K\sum_{k=1}^K
			\Bigl[
			\nabla g^k(\bar W_t)(v_t^k\otimes I_n)
			-\nabla g^k(\bar W_t)\bigl(\nabla_y f^k(g^k(\bar W_t))\otimes I_n\bigr)
			\Bigr]
		}_F^2\notag\\
		&\le
		\frac{2}{K}\sum_{k=1}^K
		\norm{
			\bigl(H_t^k-\nabla g^k(\bar W_t)\bigr)(v_t^k\otimes I_n)
		}_F^2+
		\frac{2}{K}\sum_{k=1}^K
		\norm{
			\nabla g^k(\bar W_t)
			\bigl((v_t^k-\nabla_y f^k(g^k(\bar W_t)))\otimes I_n\bigr)
		}_F^2\notag\\
		&\le
		\frac{2C_f^2}{K}\sum_{k=1}^K
		\|H_t^k-\nabla g^k(\bar W_t)\|_F^2
		+\frac{2C_g^2}{K}\sum_{k=1}^K
		\|v_t^k-\nabla_y f^k(g^k(\bar W_t))\|^2\notag\\
		&\le
		\frac{4C_f^2}{K}\sum_{k=1}^K
		\|H_t^k-\nabla g^k(W_t^k)\|_F^2
		+\frac{4C_f^2L_g^2}{K}\sum_{k=1}^K
		\|W_t^k-\bar W_t\|_F^2+
		\frac{4C_g^2}{K}\sum_{k=1}^K
		\|v_t^k-\nabla_y f^k(g^k(W_t^k))\|^2\notag\\
		&\qquad+
		\frac{4C_g^4L_f^2}{K}\sum_{k=1}^K
		\|W_t^k-\bar W_t\|_F^2 .
	\end{align}
	Averaging this display over $t$ and using Lemmas~\ref{lem:drift} and
	\ref{lem:trackers} gives
	\begin{align}
		&\frac1T\sum_{t=0}^{T-1}
		\mathbb E\norm{
			\frac1K\sum_{k=1}^K
			H_t^k\bigl(v_t^k\otimes I_n\bigr)
			-\nabla F(\bar W_t)
		}_F^2\notag\\
		&\quad\le
		\frac{4C_f^2\sigma_{\nabla g}^2}{\gamma Tb}
		+\frac{8C_g^2\sigma_f^2}{\beta Tb}
		+\frac{8C_g^2L_f^2\sigma_g^2}{\alpha Tb}
		+\frac{64C_g^2L_f^2\sigma_g^2\alpha}{\beta Tb}
		+16C_g^2L_f^2\sigma_g^2\alpha
		+\frac{64C_g^2L_f^2\sigma_g^2\alpha^2}{\beta}\notag\\
		&\qquad
		+\frac{128C_g^2L_f^2\sigma_g^2\alpha^3}{\beta}
		+16C_g^2\sigma_f^2\beta
		+8C_f^2\sigma_{\nabla g}^2\gamma
		+\frac{32C_g^4L_f^2n\eta^2\tau}{\alpha}
		+\frac{64C_g^4L_f^2n\eta^2\tau}{\beta}\notag\\
		&\qquad
		+\frac{256C_g^4L_f^2n\eta^2\alpha\tau}{\beta}
		+\frac{16C_f^2L_g^2n\eta^2\tau}{\gamma}
		+4n(C_f^2L_g^2+C_g^4L_f^2)\eta^2\tau^2 ,
	\end{align}
	where we used the projection bounds and
	Lemmas~\ref{lem:drift} and~\ref{lem:trackers}.
\end{proof}

\begin{lemma}[Client product disagreement under non-i.i.d. data]\label{lem:disagreement}
	Let \(F^k(W):=f^k(g^k(W))\) and
	\(F(W):=K^{-1}\sum_{k=1}^K F^k(W)\), and set
	\(L_F:=C_fL_g+C_g^2L_f\).
	\begin{align}
		&\frac1T\sum_{t=0}^{T-1}
		\frac1K\sum_{k=1}^K
		\mathbb E\norm{
			H_t^k(v_t^k\otimes I_n)
			-\frac1K\sum_{j=1}^K H_t^j(v_t^j\otimes I_n)
		}_F^2\notag\\
		&\quad\le
		\frac{6C_f^2\sigma_{\nabla g}^2}{\gamma Tb}
		+\frac{12C_g^2\sigma_f^2}{\beta Tb}
		+\frac{12C_g^2L_f^2\sigma_g^2}{\alpha Tb}
		+\frac{96C_g^2L_f^2\sigma_g^2\alpha}{\beta Tb}
		+24C_g^2L_f^2\sigma_g^2\alpha
		+\frac{96C_g^2L_f^2\sigma_g^2\alpha^2}{\beta}\notag\\
		&\qquad
		+\frac{192C_g^2L_f^2\sigma_g^2\alpha^3}{\beta}
		+24C_g^2\sigma_f^2\beta
		+12C_f^2\sigma_{\nabla g}^2\gamma
		+\frac{48C_g^4L_f^2n\eta^2\tau}{\alpha}
		+\frac{96C_g^4L_f^2n\eta^2\tau}{\beta}\notag\\
		&\qquad
		+\frac{384C_g^4L_f^2n\eta^2\alpha\tau}{\beta}
		+\frac{24C_f^2L_g^2n\eta^2\tau}{\gamma}
		+3L_F^2n\eta^2\tau^2
		+6C_f^2\Delta_{\nabla g}^2
		+12C_g^2\Delta_f^2
		+12C_g^2L_f^2\Delta_g^2 .
		\label{eq:app-vr-disagreement}
	\end{align}
\end{lemma}

\begin{proof}
	By Assumptions~\ref{assm:oracle} and~\ref{assm:bounded-grad},
	Jensen's inequality gives
	\[
	\|\nabla g^k(W)\|_F\le C_g,
	\qquad
	\|\nabla_y f^k(y)\|\le C_f.
	\]
	Moreover, Assumptions~\ref{assm:oracle}
	and~\ref{assm:vr-sample-smooth}, together with Jensen's inequality, imply
	\begin{equation}
		\begin{aligned}
			\|g^k(W)-g^k(W')\|
			&\le C_g\|W-W'\|_F,\\
			\|\nabla g^k(W)-\nabla g^k(W')\|_F
			&\le L_g\|W-W'\|_F,\\
			\|\nabla_y f^k(y)-\nabla_y f^k(y')\|
			&\le L_f\|y-y'\|.
		\end{aligned}
	\end{equation}
	Therefore, for any $W,W'\in\mathbb R^{m\times n}$,
	\begin{align}
		&\|\nabla F^k(W)-\nabla F^k(W')\|_F\notag\\
		&=
		\norm{
			\nabla g^k(W)\bigl(\nabla_y f^k(g^k(W))\otimes I_n\bigr)
			-\nabla g^k(W')\bigl(\nabla_y f^k(g^k(W'))\otimes I_n\bigr)
		}_F\notag\\
		&\le
		\|\nabla g^k(W)-\nabla g^k(W')\|_F
		\|\nabla_y f^k(g^k(W))\|+
		\|\nabla g^k(W')\|_F
		\|\nabla_y f^k(g^k(W))-\nabla_y f^k(g^k(W'))\|\notag\\
		&\le
		C_fL_g\|W-W'\|_F
		+
		C_gL_f\|g^k(W)-g^k(W')\|\notag\\
		&\le
		C_fL_g\|W-W'\|_F
		+
		C_g^2L_f\|W-W'\|_F
		=L_F\|W-W'\|_F ,
	\end{align}
	where we used Assumptions~\ref{assm:bounded-grad}
	and~\ref{assm:vr-sample-smooth}.
	Averaging over clients also gives
	\begin{equation}
		\|\nabla F(W)-\nabla F(W')\|_F
		\le
		\frac1K\sum_{k=1}^K
		\|\nabla F^k(W)-\nabla F^k(W')\|_F
		\le
		L_F\|W-W'\|_F .
	\end{equation}
	The non-i.i.d. term is controlled once for later substitution. For any
	$W\in\mathbb R^{m\times n}$ and $k,j\in[K]$,
	the above population bounds together with
	Assumption~\ref{assm:vr-hetero} give
	\begin{align}
		&\norm{
			\nabla F^k(W)-\nabla F^j(W)
		}_F^2\notag\\
		&\le
		2\norm{
			\bigl(\nabla g^k(W)-\nabla g^j(W)\bigr)
			\bigl(\nabla_y f^k(g^k(W))\otimes I_n\bigr)
		}_F^2+
		2\norm{
			\nabla g^j(W)
			\bigl((\nabla_y f^k(g^k(W))-\nabla_y f^j(g^j(W)))\otimes I_n\bigr)
		}_F^2\notag\\
		&\le
		2C_f^2\Delta_{\nabla g}^2
		+2C_g^2
		\|\nabla_y f^k(g^k(W))-\nabla_y f^j(g^j(W))\|^2\notag\\
		&\le
		2C_f^2\Delta_{\nabla g}^2
		+4C_g^2\Delta_f^2
		+4C_g^2L_f^2\Delta_g^2,
	\end{align}
	and Jensen's inequality yields
	\begin{align}
		&\frac1K\sum_{k=1}^K
		\|\nabla F^k(W)-\nabla F(W)\|_F^2\notag\\
		&=
		\frac1K\sum_{k=1}^K
		\norm{
			\frac1K\sum_{j=1}^K
			\bigl(\nabla F^k(W)-\nabla F^j(W)\bigr)
		}_F^2\notag\\
		&\le
		\frac1{K^2}\sum_{k=1}^K\sum_{j=1}^K
		\|\nabla F^k(W)-\nabla F^j(W)\|_F^2\notag\\
		&\le
		2C_f^2\Delta_{\nabla g}^2
		+4C_g^2\Delta_f^2
		+4C_g^2L_f^2\Delta_g^2 ,
	\end{align}
	where we used Assumption~\ref{assm:vr-hetero}, the $L_f$-smoothness,
	and the bounded-gradient assumptions.
	
	Let $Z_t^k=H_t^k(v_t^k\otimes I_n)$ and
	$\bar Z_t=K^{-1}\sum_{j=1}^K Z_t^j$. Since $\bar Z_t$ is the Euclidean
	mean of $\{Z_t^k\}_{k=1}^K$, the previous bounds imply, for every $t$,
	\begin{align}
		&\frac1K\sum_{k=1}^K
		\|Z_t^k-\bar Z_t\|_F^2\notag\\
		&\le
		\frac1K\sum_{k=1}^K
		\|Z_t^k-\nabla F(\bar W_t)\|_F^2\notag\\
		&\le
		\frac3K\sum_{k=1}^K
		\|Z_t^k-\nabla F^k(W_t^k)\|_F^2
		+\frac3K\sum_{k=1}^K
		\|\nabla F^k(W_t^k)-\nabla F^k(\bar W_t)\|_F^2\notag\\
		&\qquad
		+\frac3K\sum_{k=1}^K
		\|\nabla F^k(\bar W_t)-\nabla F(\bar W_t)\|_F^2\notag\\
		&\le
		\frac{6C_f^2}{K}\sum_{k=1}^K
		\|H_t^k-\nabla g^k(W_t^k)\|_F^2
		+\frac{6C_g^2}{K}\sum_{k=1}^K
		\|v_t^k-\nabla_y f^k(g^k(W_t^k))\|^2\notag\\
		&\qquad
		+\frac{3L_F^2}{K}\sum_{k=1}^K
		\|W_t^k-\bar W_t\|_F^2
		+6C_f^2\Delta_{\nabla g}^2
		+12C_g^2\Delta_f^2
		+12C_g^2L_f^2\Delta_g^2 ,
	\end{align}
	where we used the optimality of the Euclidean mean and the preceding
	smoothness and heterogeneity bounds.
	Averaging this inequality over $t$, taking expectations, and applying
	Lemmas~\ref{lem:drift} and~\ref{lem:trackers}, we obtain
	\begin{align}
		&\frac1T\sum_{t=0}^{T-1}
		\frac1K\sum_{k=1}^K
		\mathbb E\|Z_t^k-\bar Z_t\|_F^2\notag\\
		&\le
		6C_f^2
		\left(
		\frac{\sigma_{\nabla g}^2}{\gamma Tb}
		+2\sigma_{\nabla g}^2\gamma
		+\frac{4L_g^2n\eta^2\tau}{\gamma}
		\right)
		+6C_g^2
		\Bigg(
		\frac{2\sigma_f^2}{\beta Tb}
		+\frac{2L_f^2\sigma_g^2}{\alpha Tb}
		+\frac{16L_f^2\sigma_g^2\alpha}{\beta Tb}
		+4L_f^2\sigma_g^2\alpha\notag\\
		&\qquad
		+\frac{16L_f^2\sigma_g^2\alpha^2}{\beta}
		+\frac{32L_f^2\sigma_g^2\alpha^3}{\beta}
		+4\sigma_f^2\beta
		+\frac{8L_f^2C_g^2n\eta^2\tau}{\alpha}
		+\frac{16L_f^2C_g^2n\eta^2\tau}{\beta}
		+\frac{64L_f^2C_g^2n\eta^2\alpha\tau}{\beta}
		\Bigg)\notag\\
		&\qquad
		+3L_F^2n\eta^2\tau^2
		+6C_f^2\Delta_{\nabla g}^2
		+12C_g^2\Delta_f^2
		+12C_g^2L_f^2\Delta_g^2\notag\\
		&=
		\frac{6C_f^2\sigma_{\nabla g}^2}{\gamma Tb}
		+\frac{12C_g^2\sigma_f^2}{\beta Tb}
		+\frac{12C_g^2L_f^2\sigma_g^2}{\alpha Tb}
		+\frac{96C_g^2L_f^2\sigma_g^2\alpha}{\beta Tb}
		+24C_g^2L_f^2\sigma_g^2\alpha
		+\frac{96C_g^2L_f^2\sigma_g^2\alpha^2}{\beta}
		+\frac{192C_g^2L_f^2\sigma_g^2\alpha^3}{\beta}\notag\\
		&\qquad
		+24C_g^2\sigma_f^2\beta
		+12C_f^2\sigma_{\nabla g}^2\gamma
		+\frac{48C_g^4L_f^2n\eta^2\tau}{\alpha}
		+\frac{96C_g^4L_f^2n\eta^2\tau}{\beta}
		+\frac{384C_g^4L_f^2n\eta^2\alpha\tau}{\beta}\notag\\
		&\qquad
		+\frac{24C_f^2L_g^2n\eta^2\tau}{\gamma}
		+3L_F^2n\eta^2\tau^2
		+6C_f^2\Delta_{\nabla g}^2
		+12C_g^2\Delta_f^2
		+12C_g^2L_f^2\Delta_g^2 .
	\end{align}
	Substituting the definition of $Z_t^k$ gives the stated bound.
\end{proof}

\begin{lemma}[Average product tracker error]\label{lem:M-error-first-order}
	Assume $0<\rho<1$ and let
	$\bar M_t:=K^{-1}\sum_{k=1}^K M_t^k$. Then
	\begin{align}
		&\frac1T\sum_{t=0}^{T-1}
		\mathbb E\|\bar M_t-\nabla F(\bar W_t)\|_F
		\notag\\
		&\le
		\frac{4C_gC_f}{\rho T}
		+\frac{L_F\sqrt n\,\eta}{\rho}
		+
		\Bigg[
		\frac{4C_f^2\sigma_{\nabla g}^2}{\gamma Tb}
		+\frac{8C_g^2\sigma_f^2}{\beta Tb}
		+\frac{8C_g^2L_f^2\sigma_g^2}{\alpha Tb}
		+\frac{64C_g^2L_f^2\sigma_g^2\alpha}{\beta Tb}
		+16C_g^2L_f^2\sigma_g^2\alpha
		+\frac{64C_g^2L_f^2\sigma_g^2\alpha^2}{\beta}
		\notag\\
		&\qquad
		+\frac{128C_g^2L_f^2\sigma_g^2\alpha^3}{\beta}
		+16C_g^2\sigma_f^2\beta
		+8C_f^2\sigma_{\nabla g}^2\gamma
		+\frac{32C_g^4L_f^2n\eta^2\tau}{\alpha}
		+\frac{64C_g^4L_f^2n\eta^2\tau}{\beta}
		+\frac{256C_g^4L_f^2n\eta^2\alpha\tau}{\beta}
		\notag\\
		&\qquad
		+\frac{16C_f^2L_g^2n\eta^2\tau}{\gamma}
		+4n(C_f^2L_g^2+C_g^4L_f^2)\eta^2\tau^2
		\Bigg]^{1/2}.
		\label{eq:app-vr-average-M-error}
	\end{align}
\end{lemma}

\begin{proof}
	\begin{align}
		&\|\bar M_{t+1}-\nabla F(\bar W_{t+1})\|_F\notag\\
		&=
		\norm{
			(1-\rho)\bigl(\bar M_t-\nabla F(\bar W_t)\bigr)
			+\rho\left[
			\frac1K\sum_{k=1}^K H_{t+1}^k(v_{t+1}^k\otimes I_n)
			-\nabla F(\bar W_{t+1})
			\right]
			+(1-\rho)\bigl(\nabla F(\bar W_t)-\nabla F(\bar W_{t+1})\bigr)
		}_F\notag\\
		&\le
		(1-\rho)\|\bar M_t-\nabla F(\bar W_t)\|_F+
		\rho\norm{
			\frac1K\sum_{k=1}^K H_{t+1}^k(v_{t+1}^k\otimes I_n)
			-\nabla F(\bar W_{t+1})
		}_F
		+
		(1-\rho)\|\nabla F(\bar W_t)-\nabla F(\bar W_{t+1})\|_F\notag\\
		&\le
		(1-\rho)\|\bar M_t-\nabla F(\bar W_t)\|_F+
		\rho\norm{
			\frac1K\sum_{k=1}^K H_{t+1}^k(v_{t+1}^k\otimes I_n)
			-\nabla F(\bar W_{t+1})
		}_F
		+L_F\|\bar W_{t+1}-\bar W_t\|_F\notag\\
		&\le
		(1-\rho)\|\bar M_t-\nabla F(\bar W_t)\|_F+
		\rho\norm{
			\frac1K\sum_{k=1}^K H_{t+1}^k(v_{t+1}^k\otimes I_n)
			-\nabla F(\bar W_{t+1})
		}_F+
		\frac{L_F\eta}{K}\sum_{k=1}^K\|Q(M_t^k)\|_F\notag\\
		&\le
		(1-\rho)\|\bar M_t-\nabla F(\bar W_t)\|_F+
		\rho\norm{
			\frac1K\sum_{k=1}^K H_{t+1}^k(v_{t+1}^k\otimes I_n)
			-\nabla F(\bar W_{t+1})
		}_F
		+L_F\eta\sqrt n ,
	\end{align}
	where we used the $L_F$-smoothness of $F$, the triangle inequality,
	and $\|U_t^k(V_t^k)^\top\|_F\le\sqrt n$.
	Consequently,
	\begin{align}
		&\frac1T\sum_{t=0}^{T-1}
		\mathbb E\|\bar M_t-\nabla F(\bar W_t)\|_F\notag\\
		&\le
		\frac1{\rho T}\sum_{t=0}^{T-1}
		\mathbb E\|\bar M_t-\nabla F(\bar W_t)\|_F
		-\frac1{\rho T}\sum_{t=0}^{T-1}
		\mathbb E\|\bar M_{t+1}-\nabla F(\bar W_{t+1})\|_F\notag\\
		&\quad+
		\frac1T\sum_{t=0}^{T-1}
		\mathbb E\norm{
			\frac1K\sum_{k=1}^K H_{t+1}^k(v_{t+1}^k\otimes I_n)
			-\nabla F(\bar W_{t+1})
		}_F
		+\frac{L_F\eta\sqrt n}{\rho}\notag\\
		&\le
		\frac{\mathbb E\|\bar M_0-\nabla F(W_0)\|_F}{\rho T}
		+\frac1T\sum_{t=1}^{T}
		\mathbb E\norm{
			\frac1K\sum_{k=1}^K H_t^k(v_t^k\otimes I_n)
			-\nabla F(\bar W_t)
		}_F
		+\frac{L_F\eta\sqrt n}{\rho}\notag\\
		&\le
		\frac{2C_gC_f}{\rho T}
		+\frac{2C_gC_f}{T}
		+\frac1T\sum_{t=0}^{T-1}
		\mathbb E\norm{
			\frac1K\sum_{k=1}^K H_t^k(v_t^k\otimes I_n)
			-\nabla F(\bar W_t)
		}_F
		+\frac{L_F\eta\sqrt n}{\rho}\notag\\
		&\le
		\frac{4C_gC_f}{\rho T}
		+\frac{L_F\eta\sqrt n}{\rho}
		+\frac1T\sum_{t=0}^{T-1}
		\mathbb E\norm{
			\frac1K\sum_{k=1}^K H_t^k(v_t^k\otimes I_n)
			-\nabla F(\bar W_t)
		}_F .
	\end{align}
	By Cauchy--Schwarz and Lemma~\ref{lem:product}, the remaining average satisfies
	\begin{align}
		&\frac1T\sum_{t=0}^{T-1}
		\mathbb E
		\norm{
			\frac1K\sum_{k=1}^K
			H_t^k(v_t^k\otimes I_n)
			-
			\nabla F(\bar W_t)
		}_F
		\notag\\
		\le{}&
		\Bigg[
		\frac{4C_f^2\sigma_{\nabla g}^2}{\gamma Tb}
		+\frac{8C_g^2\sigma_f^2}{\beta Tb}
		+\frac{8C_g^2L_f^2\sigma_g^2}{\alpha Tb}
		+\frac{64C_g^2L_f^2\sigma_g^2\alpha}{\beta Tb}
		+16C_g^2L_f^2\sigma_g^2\alpha
		\notag\\
		&\qquad
		+\frac{64C_g^2L_f^2\sigma_g^2\alpha^2}{\beta}
		+\frac{128C_g^2L_f^2\sigma_g^2\alpha^3}{\beta}
		+16C_g^2\sigma_f^2\beta
		+8C_f^2\sigma_{\nabla g}^2\gamma
		+\frac{32C_g^4L_f^2n\eta^2\tau}{\alpha}
		\notag\\
		&\qquad
		+\frac{64C_g^4L_f^2n\eta^2\tau}{\beta}
		+\frac{256C_g^4L_f^2n\eta^2\alpha\tau}{\beta}
		+\frac{16C_f^2L_g^2n\eta^2\tau}{\gamma}
		+4n(C_f^2L_g^2+C_g^4L_f^2)\eta^2\tau^2
		\Bigg]^{1/2}.
	\end{align}
	Combining the preceding two displays proves
	\eqref{eq:app-vr-average-M-error}.
\end{proof}

\begin{lemma}[Client momentum disagreement]
	\label{lem:M-disagreement-first-order}
	Assume $0<\rho<1$. Let
	$\bar M_t:=K^{-1}\sum_{k=1}^K M_t^k$ and
	$s(t):=\tau\lfloor t/\tau\rfloor$. Then
	\begin{align}
		&\frac1{KT}\sum_{t=0}^{T-1}\sum_{k=1}^K
		\mathbb E\|M_t^k-\bar M_t\|_F
		\notag\\
		&\le
		\rho\tau
		\Bigg[
		\frac{6C_f^2\sigma_{\nabla g}^2}{\gamma Tb}
		+\frac{12C_g^2\sigma_f^2}{\beta Tb}
		+\frac{12C_g^2L_f^2\sigma_g^2}{\alpha Tb}
		\notag\\
		&\qquad
		+\frac{96C_g^2L_f^2\sigma_g^2\alpha}{\beta Tb}
		+24C_g^2L_f^2\sigma_g^2\alpha
		+\frac{96C_g^2L_f^2\sigma_g^2\alpha^2}{\beta}
		+\frac{192C_g^2L_f^2\sigma_g^2\alpha^3}{\beta}
		+24C_g^2\sigma_f^2\beta
		+12C_f^2\sigma_{\nabla g}^2\gamma
		\notag\\
		&\qquad
		+\frac{48C_g^4L_f^2n\eta^2\tau}{\alpha}
		+\frac{96C_g^4L_f^2n\eta^2\tau}{\beta}
		+\frac{384C_g^4L_f^2n\eta^2\alpha\tau}{\beta}
		+\frac{24C_f^2L_g^2n\eta^2\tau}{\gamma}
		+3L_F^2n\eta^2\tau^2
		\notag\\
		&\qquad
		+6C_f^2\Delta_{\nabla g}^2
		+12C_g^2\Delta_f^2
		+12C_g^2L_f^2\Delta_g^2
		\Bigg]^{1/2}.
		\label{eq:app-vr-local-M-disagreement}
	\end{align}
\end{lemma}

\begin{proof}
	
	Averaging the momentum recursion over the clients, subtracting the resulting
	identity from the local recursion, and using the synchronization condition
	$M_{s(t)}^k=\bar M_{s(t)}$, we obtain
	\begin{align}
		&\frac1{KT}\sum_{t=0}^{T-1}\sum_{k=1}^K
		\mathbb E\|M_t^k-\bar M_t\|_F\notag\\
		&=
		\frac1{KT}\sum_{t=0}^{T-1}\sum_{k=1}^K
		\mathbb E\norm{
			\sum_{\ell=s(t)+1}^{t}
			\rho(1-\rho)^{t-\ell}\left[
			H_\ell^k(v_\ell^k\otimes I_n)
			-\frac1K\sum_{j=1}^K H_\ell^j(v_\ell^j\otimes I_n)
			\right]
		}_F\notag\\
		&\le
		\frac{\rho}{KT}\sum_{t=0}^{T-1}\sum_{k=1}^K
		\sum_{\ell=s(t)+1}^{t}(1-\rho)^{t-\ell}
		\mathbb E\norm{
			H_\ell^k(v_\ell^k\otimes I_n)
			-\frac1K\sum_{j=1}^K H_\ell^j(v_\ell^j\otimes I_n)
		}_F\notag\\
		&\le
		\frac{\rho}{T}\sum_{t=0}^{T-1}
		\sum_{\ell=s(t)+1}^{t}
		\frac1K\sum_{k=1}^K
		\mathbb E\norm{
			H_\ell^k(v_\ell^k\otimes I_n)
			-\frac1K\sum_{j=1}^K H_\ell^j(v_\ell^j\otimes I_n)
		}_F\notag\\
		&\le
		\rho\tau\frac1T\sum_{\ell=0}^{T-1}
		\frac1K\sum_{k=1}^K
		\mathbb E\norm{
			H_\ell^k(v_\ell^k\otimes I_n)
			-\frac1K\sum_{j=1}^K H_\ell^j(v_\ell^j\otimes I_n)
		}_F .
	\end{align}
	The inequalities above use the triangle inequality,
	$(1-\rho)^{t-\ell}\le1$, and the fact that each $\ell$ occurs at most
	$\tau$ times. Lemma~C.4 then bounds the remaining term:
	\begin{align}
		&\frac1T\sum_{\ell=0}^{T-1}
		\frac1K\sum_{k=1}^K
		\mathbb E
		\norm{
			H_\ell^k(v_\ell^k\otimes I_n)
			-
			\frac1K\sum_{j=1}^K
			H_\ell^j(v_\ell^j\otimes I_n)
		}_F
		\notag\\
		\le{}&
		\Bigg[
		\frac{6C_f^2\sigma_{\nabla g}^2}{\gamma Tb}
		+\frac{12C_g^2\sigma_f^2}{\beta Tb}
		+\frac{12C_g^2L_f^2\sigma_g^2}{\alpha Tb}
		+\frac{96C_g^2L_f^2\sigma_g^2\alpha}{\beta Tb}
		+24C_g^2L_f^2\sigma_g^2\alpha
		\notag\\
		&\quad
		+\frac{96C_g^2L_f^2\sigma_g^2\alpha^2}{\beta}
		+\frac{192C_g^2L_f^2\sigma_g^2\alpha^3}{\beta}
		+24C_g^2\sigma_f^2\beta
		+12C_f^2\sigma_{\nabla g}^2\gamma
		+\frac{48C_g^4L_f^2n\eta^2\tau}{\alpha}
		+\frac{96C_g^4L_f^2n\eta^2\tau}{\beta}
		\notag\\
		&\quad
		+\frac{384C_g^4L_f^2n\eta^2\alpha\tau}{\beta}
		+\frac{24C_f^2L_g^2n\eta^2\tau}{\gamma}
		+3L_F^2n\eta^2\tau^2
		+6C_f^2\Delta_{\nabla g}^2
		+12C_g^2\Delta_f^2
		+12C_g^2L_f^2\Delta_g^2
		\Bigg]^{1/2}.
	\end{align}
	Combining the last two displays proves
	\eqref{eq:app-vr-local-M-disagreement}.
\end{proof}

\begin{lemma}[Muon descent]\label{lem:descent}
	For every $t$,
	\begin{equation}
		\label{eq:app-vr-descent}
		\begin{aligned}
			\mathbb E\|\nabla F(\bar W_t)\|_F
			&\le
			\frac{\mathbb E[F(\bar W_t)-F(\bar W_{t+1})]}{\eta}+\frac{2\sqrt n}{K}\sum_{k=1}^K
			\mathbb E\|M_t^k-\bar M_t\|_F\\
			&\quad
			+2\sqrt n\,\mathbb E\|\bar M_t-\nabla F(\bar W_t)\|_F+\frac{L_F n\eta}{2}.
		\end{aligned}
	\end{equation}
\end{lemma}

\begin{proof}
	By the $L_F$-smoothness of $F$,
	\begin{align}
		F(\bar W_{t+1})
		&\le
		F(\bar W_t)
		-\eta
		\left\langle
		\nabla F(\bar W_t),
		\frac1K\sum_{k=1}^K U_t^k(V_t^k)^\top
		\right\rangle
		+\frac{L_F\eta^2}{2}
		\norm{
			\frac1K\sum_{k=1}^K U_t^k(V_t^k)^\top
		}_F^2\notag\\
		&=
		F(\bar W_t)
		-\frac{\eta}{K}\sum_{k=1}^K
		\left\langle \nabla F(\bar W_t), U_t^k(V_t^k)^\top\right\rangle
		+\frac{L_F\eta^2}{2}
		\norm{
			\frac1K\sum_{k=1}^K U_t^k(V_t^k)^\top
		}_F^2\notag\\
		&=
		F(\bar W_t)
		-\frac{\eta}{K}\sum_{k=1}^K
		\left\langle M_t^k,U_t^k(V_t^k)^\top\right\rangle
		-\frac{\eta}{K}\sum_{k=1}^K
		\left\langle
		\nabla F(\bar W_t)-M_t^k,U_t^k(V_t^k)^\top
		\right\rangle\notag+
		\frac{L_F\eta^2}{2}
		\norm{
			\frac1K\sum_{k=1}^K U_t^k(V_t^k)^\top
		}_F^2\notag\\
		&\le
		F(\bar W_t)
		-\frac{\eta}{K}\sum_{k=1}^K
		\|M_t^k\|_*
		+\frac{\eta}{K}\sum_{k=1}^K
		\|\nabla F(\bar W_t)-M_t^k\|_*+
		\frac{L_F\eta^2}{2K}
		\sum_{k=1}^K\|U_t^k(V_t^k)^\top\|_F^2\notag\\
		&\le
		F(\bar W_t)
		-\frac{\eta}{K}\sum_{k=1}^K
		\left(
		\|\nabla F(\bar W_t)\|_*-\|\nabla F(\bar W_t)-M_t^k\|_*
		\right)
		\notag\\
		&\quad
		+\frac{\eta}{K}\sum_{k=1}^K
		\|\nabla F(\bar W_t)-M_t^k\|_*+
		\frac{L_F\eta^2}{2K}
		\sum_{k=1}^K\|U_t^k(V_t^k)^\top\|_F^2\notag\\
		&\le
		F(\bar W_t)
		-\eta\|\nabla F(\bar W_t)\|_*
		+\frac{2\eta}{K}\sum_{k=1}^K
		\|M_t^k-\nabla F(\bar W_t)\|_*
		+\frac{L_Fn\eta^2}{2}\notag\\
		&\le
		F(\bar W_t)
		-\eta\|\nabla F(\bar W_t)\|_F
		+\frac{2\eta\sqrt n}{K}\sum_{k=1}^K
		\|M_t^k-\bar M_t\|_F
		+2\eta\sqrt n\,\|\bar M_t-\nabla F(\bar W_t)\|_F
		+\frac{L_Fn\eta^2}{2},
	\end{align}
	where we used Lemma~\ref{lem:muon-properties}, the triangle inequality,
	and the standard nuclear--Frobenius norm relations.
	Rearranging and taking expectation proves the claim.
\end{proof}

\begin{theorem}[Convergence of FedCoMuon-VR with first-order product tracker]
	\label{thm:rate-first-order}
	Suppose Assumptions~\ref{assm:lower}--\ref{assm:bounded-grad},
	\ref{assm:vr-sample-smooth},
	and~\ref{assm:vr-hetero} hold. Let $L_F=C_fL_g+C_g^2L_f$.
	Assume $\eta>0$, $0<\alpha,\beta,\gamma,\rho<1$, $b\ge1$, and $\tau>0$.
	Then, for any $T\ge1$,
	\begin{align}
		\label{eq:app-vr-fedcomuon-vr-convergence}
		&\frac1T\sum_{t=0}^{T-1}
		\mathbb E\|\nabla F(\bar W_t)\|_F
		\notag\\
		&\le
		\frac{F(\bar W_0)-F_*}{\eta T}
		+\frac{L_Fn\eta}{2}
		+2\sqrt n
		\Bigg[
		\frac{4C_gC_f}{\rho T}
		+\frac{L_F\sqrt n\,\eta}{\rho}
		+(2+\sqrt6\,\rho\tau)
		\notag\\
		&\qquad\quad\times\Bigg(
		C_f^2\sigma_{\nabla g}^2
		\left(\frac1{\gamma Tb}+2\gamma\right)
		+
		2C_g^2\sigma_f^2
		\left(\frac1{\beta Tb}+2\beta\right)
		\notag\\
		&\quad
		+
		2C_g^2L_f^2\sigma_g^2
		\left[
		\frac1{\alpha Tb}
		+2\alpha
		\right.
		\notag\\
		&\qquad\qquad\left.
		+\frac{8\alpha}{\beta Tb}
		+\frac{8\alpha^2}{\beta}(1+2\alpha)
		\right]
		+
		4n\eta^2\tau
		\left(
		\frac{2C_g^4L_f^2}{\alpha}
		+
		\frac{4C_g^4L_f^2}{\beta}
		+
		\frac{16C_g^4L_f^2\alpha}{\beta}
		+
		\frac{C_f^2L_g^2}{\gamma}
		\right)
		\notag\\
		&\qquad\qquad
		+
		n\eta^2\tau^2
		\left(
		C_f^2L_g^2
		+
		C_g^4L_f^2
		+
		\frac{L_F^2}{2}
		\right)
		\Bigg)^{1/2}
		+
		\sqrt6\,\rho\tau
		\left(
		C_f^2\Delta_{\nabla g}^2
		+
		2C_g^2\Delta_f^2
		+
		2C_g^2L_f^2\Delta_g^2
		\right)^{1/2}
		\Bigg].
	\end{align}
\end{theorem}

\begin{proof}
	By Lemma~\ref{lem:descent}, for every $t$,
	\begin{align}
		\mathbb E\|\nabla F(\bar W_t)\|_F
		&\le
		\frac{\mathbb E[F(\bar W_t)-F(\bar W_{t+1})]}{\eta}
		+\frac{2\sqrt n}{K}\sum_{k=1}^K
		\mathbb E\|M_t^k-\bar M_t\|_F\notag\\
		&\quad+2\sqrt n\,\mathbb E\|\bar M_t-\nabla F(\bar W_t)\|_F
		+
		\frac{L_Fn\eta}{2}.
	\end{align}
	After summing this inequality over $t=0,\ldots,T-1$, dividing by $T$, and using
	$F(\bar W_T)\ge F_*$, it remains to bound the two averaged error terms.
	Applying Lemmas~\ref{lem:M-error-first-order}
	and~\ref{lem:M-disagreement-first-order} gives
	\begin{align}
		&\frac{2\sqrt n}{KT}\sum_{t=0}^{T-1}\sum_{k=1}^K
		\mathbb E\|M_t^k-\bar M_t\|_F
		+\frac{2\sqrt n}{T}\sum_{t=0}^{T-1}
		\mathbb E\|\bar M_t-\nabla F(\bar W_t)\|_F
		\notag\\
		\le{}&
		2\sqrt n
		\Bigg[
		\frac{4C_gC_f}{\rho T}
		+\frac{L_F\sqrt n\,\eta}{\rho}
		+\Bigg[
		4C_f^2\sigma_{\nabla g}^2\left(\frac1{\gamma Tb}+2\gamma\right)
		+8C_g^2\sigma_f^2\left(\frac1{\beta Tb}+2\beta\right)
		\notag\\
		&\quad
		+8C_g^2L_f^2\sigma_g^2\left[
		\frac1{\alpha Tb}+2\alpha
		+\frac{8\alpha}{\beta Tb}
		+\frac{8\alpha^2}{\beta}(1+2\alpha)
		\right]
		\notag\\
		&\quad
		+16n\eta^2\tau\left(
		\frac{2C_g^4L_f^2}{\alpha}
		+\frac{4C_g^4L_f^2}{\beta}
		+\frac{16C_g^4L_f^2\alpha}{\beta}
		+\frac{C_f^2L_g^2}{\gamma}
		\right)
		+4n(C_f^2L_g^2+C_g^4L_f^2)\eta^2\tau^2
		\Bigg]^{1/2}
		{}+\rho\tau\Bigg[
		\notag\\
		&\quad
		6C_f^2\sigma_{\nabla g}^2\left(\frac1{\gamma Tb}+2\gamma\right)
		+12C_g^2\sigma_f^2\left(\frac1{\beta Tb}+2\beta\right)
		+12C_g^2L_f^2\sigma_g^2\left[
		\frac1{\alpha Tb}+2\alpha
		+\frac{8\alpha}{\beta Tb}
		+\frac{8\alpha^2}{\beta}(1+2\alpha)
		\right]
		\notag\\
		&\quad
		+24n\eta^2\tau\left(
		\frac{2C_g^4L_f^2}{\alpha}
		+\frac{4C_g^4L_f^2}{\beta}
		+\frac{16C_g^4L_f^2\alpha}{\beta}
		+\frac{C_f^2L_g^2}{\gamma}
		\right)
		+3L_F^2n\eta^2\tau^2
		+6C_f^2\Delta_{\nabla g}^2
		+12C_g^2\Delta_f^2
		+12C_g^2L_f^2\Delta_g^2
		\Bigg]^{1/2}
		\Bigg]
		\notag\\
		\le{}&
		2\sqrt n
		\Bigg[
		\frac{4C_gC_f}{\rho T}
		+\frac{L_F\sqrt n\,\eta}{\rho}
		+
		\Bigg[
		4
		\Bigg(
		C_f^2\sigma_{\nabla g}^2
		\left(\frac1{\gamma Tb}+2\gamma\right)
		+
		2C_g^2\sigma_f^2
		\left(\frac1{\beta Tb}+2\beta\right)
		\notag\\
		&\quad
		+
		2C_g^2L_f^2\sigma_g^2
		\left[
		\frac1{\alpha Tb}
		+2\alpha
		+
		\frac{8\alpha}{\beta Tb}
		+\frac{8\alpha^2}{\beta}(1+2\alpha)
		\right]
		+
		4n\eta^2\tau
		\left(
		\frac{2C_g^4L_f^2}{\alpha}
		+
		\frac{4C_g^4L_f^2}{\beta}
		+
		\frac{16C_g^4L_f^2\alpha}{\beta}
		+
		\frac{C_f^2L_g^2}{\gamma}
		\right)
		\notag\\
		&\quad
		+
		n\eta^2\tau^2
		\left(
		C_f^2L_g^2
		+
		C_g^4L_f^2
		+
		\frac{L_F^2}{2}
		\right)
		\Bigg)
		\Bigg]^{1/2}
		{}+\rho\tau
		\Bigg[
		6
		\Bigg(
		C_f^2\sigma_{\nabla g}^2
		\left(\frac1{\gamma Tb}+2\gamma\right)
		+
		2C_g^2\sigma_f^2
		\left(\frac1{\beta Tb}+2\beta\right)
		\notag\\
		&\quad
		+
		2C_g^2L_f^2\sigma_g^2
		\left[
		\frac1{\alpha Tb}
		+2\alpha
		+
		\frac{8\alpha}{\beta Tb}
		+\frac{8\alpha^2}{\beta}(1+2\alpha)
		\right]
		+
		4n\eta^2\tau
		\left(
		\frac{2C_g^4L_f^2}{\alpha}
		+
		\frac{4C_g^4L_f^2}{\beta}
		+
		\frac{16C_g^4L_f^2\alpha}{\beta}
		+
		\frac{C_f^2L_g^2}{\gamma}
		\right)
		\notag\\
		&\quad
		+
		n\eta^2\tau^2
		\left(
		C_f^2L_g^2
		+
		C_g^4L_f^2
		+
		\frac{L_F^2}{2}
		\right)
		\Bigg)
		+
		6
		\left(
		C_f^2\Delta_{\nabla g}^2
		+
		2C_g^2\Delta_f^2
		+
		2C_g^2L_f^2\Delta_g^2
		\right)
		\Bigg]^{1/2}
		\Bigg]
		\notag\\
		\le{}&
		2\sqrt n
		\Bigg[
		\frac{4C_gC_f}{\rho T}
		+\frac{L_F\sqrt n\,\eta}{\rho}
		\notag\\
		&\quad
		+
		(2+\sqrt6\,\rho\tau)
		\Bigg(
		C_f^2\sigma_{\nabla g}^2
		\left(\frac1{\gamma Tb}+2\gamma\right)
		+
		2C_g^2\sigma_f^2
		\left(\frac1{\beta Tb}+2\beta\right)
		\notag\\
		&\qquad
		+
		2C_g^2L_f^2\sigma_g^2
		\left[
		\frac1{\alpha Tb}
		+2\alpha
		+
		\frac{8\alpha}{\beta Tb}
		+\frac{8\alpha^2}{\beta}(1+2\alpha)
		\right]
		+
		4n\eta^2\tau
		\left(
		\frac{2C_g^4L_f^2}{\alpha}
		+
		\frac{4C_g^4L_f^2}{\beta}
		+
		\frac{16C_g^4L_f^2\alpha}{\beta}
		+
		\frac{C_f^2L_g^2}{\gamma}
		\right)
		\notag\\
		&\qquad
		+
		n\eta^2\tau^2
		\left(
		C_f^2L_g^2
		+
		C_g^4L_f^2
		+
		\frac{L_F^2}{2}
		\right)
		\Bigg)^{1/2}
		+
		\sqrt6\,\rho\tau
		\left(
		C_f^2\Delta_{\nabla g}^2
		+
		2C_g^2\Delta_f^2
		+
		2C_g^2L_f^2\Delta_g^2
		\right)^{1/2}
		\Bigg].
	\end{align}
	Substituting the last display into the averaged descent inequality gives
	\begin{align}
		&\frac1T\sum_{t=0}^{T-1}
		\mathbb E\|\nabla F(\bar W_t)\|_F
		\notag\\
		&\le
		\frac{F(\bar W_0)-F_*}{\eta T}
		+\frac{L_Fn\eta}{2}
		+2\sqrt n
		\Bigg[
		\frac{4C_gC_f}{\rho T}
		+\frac{L_F\sqrt n\,\eta}{\rho}
		+(2+\sqrt6\,\rho\tau)
		\notag\\
		&\qquad\quad\times\Bigg(
		C_f^2\sigma_{\nabla g}^2
		\left(\frac1{\gamma Tb}+2\gamma\right)
		+
		2C_g^2\sigma_f^2
		\left(\frac1{\beta Tb}+2\beta\right)
		+
		2C_g^2L_f^2\sigma_g^2
		\left[
		\frac1{\alpha Tb}
		+2\alpha
		\right.
		\notag\\
		&\qquad\qquad\left.
		+\frac{8\alpha}{\beta Tb}
		+\frac{8\alpha^2}{\beta}(1+2\alpha)
		\right]
		+
		4n\eta^2\tau
		\left(
		\frac{2C_g^4L_f^2}{\alpha}
		+
		\frac{4C_g^4L_f^2}{\beta}
		+
		\frac{16C_g^4L_f^2\alpha}{\beta}
		+
		\frac{C_f^2L_g^2}{\gamma}
		\right)
		\notag\\
		&\qquad\qquad
		+
		n\eta^2\tau^2
		\left(
		C_f^2L_g^2
		+
		C_g^4L_f^2
		+
		\frac{L_F^2}{2}
		\right)
		\Bigg)^{1/2}
		+
		\sqrt6\,\rho\tau
		\left(
		C_f^2\Delta_{\nabla g}^2
		+
		2C_g^2\Delta_f^2
		+
		2C_g^2L_f^2\Delta_g^2
		\right)^{1/2}
		\Bigg].
	\end{align}
	This proves the theorem.
\end{proof}

For $\eta=\alpha=\beta=\gamma=T^{-2/3}$,
$\rho=T^{-1/3}$, $b=T^{2/3}$, and $\tau=O(1)$, the non-square-root
terms in Theorem~\ref{thm:rate-first-order} are at most $O(T^{-1/3})$,
while every term inside the square root is at most $O(T^{-2/3})$.
Moreover, $2+\sqrt6\,\rho\tau=O(1)$ and
$\rho\tau=O(T^{-1/3})$. Hence,
$T^{-1}\sum_{t=0}^{T-1}\mathbb E\|\nabla F(\bar W_t)\|_F
=O(T^{-1/3})$.

\section{Detailed Experimental Settings}
\label{app:experimental-settings}

All experiments were repeated using two random seeds, 42 and 43, and
we report the average results across the two runs.
For all Muon-based methods, we use five Newton--Schulz iterations to
approximate the orthogonalization of matrix-valued parameters. All learning
rates and method-specific hyperparameters are selected through grid search
based on validation performance.

\subsection{Robust Federated Learning}
\label{app:robust-settings}

\subsubsection{Image Classification on MNIST}
\label{app:robust-mnist-settings}

The CNN architecture used for MNIST is summarized in
Table~\ref{tab:four-layer-cnn}.

\begin{table}[H]
	\centering
	\begin{tabular}{lcccc}
		\toprule
		\textbf{Layer Type} & \textbf{Output Size} & \textbf{Kernel Size}
		& \textbf{Stride} & \textbf{Activation} \\
		\midrule
		Input       & $28\times28\times1$ & --         & -- & --      \\
		Convolution & $24\times24\times6$ & $5\times5$ & 1  & ReLU    \\
		Max Pooling & $12\times12\times6$ & $2\times2$ & 2  & --      \\
		Convolution & $8\times8\times16$  & $5\times5$ & 1  & ReLU    \\
		Max Pooling & $4\times4\times16$  & $2\times2$ & 2  & --      \\
		Flatten     & 256                  & --         & -- & --      \\
		Dense       & 120                  & --         & -- & ReLU    \\
		Output      & 10                   & --         & -- & --      \\
		\bottomrule
	\end{tabular}
	\caption{CNN architecture used for MNIST.}
	\label{tab:four-layer-cnn}
\end{table}

For the MNIST experiments, we search the learning rate over
$\{0.005,0.01,0.02,0.03,0.05,0.1,0.2,0.5,1\}$. We set the batch size
to 20 and $\lambda=0.5$ for all algorithms. FedAvg and ComFedL use a
learning rate of $0.02$, while all other methods use a learning rate of
$0.01$. For FedMuon, we set $\beta=0.1$. FedMuon-LGA uses $\alpha=0.2$
and $\beta=0.8$, while FedMuon-BC uses $\alpha=0.1$. For Local-SCGDM,
we set $\alpha=0.2$ and $\gamma=0.3$. FedCoMuon uses $\alpha=0.2$ and
$\beta=0.1$. For FedCoMuon-VR, we set $\alpha=0.2$, $\beta=0.8$,
$\gamma=0.9$, and $\rho=0.2$.

\subsubsection{Language Modeling on WikiText-2}
\label{app:robust-wikitext-settings}

We use an 8-layer Transformer encoder with an embedding dimension of 768 and
eight attention heads per layer. Each Transformer block employs a
feed-forward network with a hidden dimension of 1024, together with sinusoidal
positional encodings. A dropout rate of $0.1$ is applied throughout the
network. The final output layer projects the hidden representations to the
vocabulary size for next-token prediction.

The learning rates are searched over
$\{0.005,0.01,0.02,0.03,0.05,0.1,0.2,0.5,1\}$.
In the experiments, we set the batch size to 32 for all algorithms. We set the
learning rate to $0.1$ for FedAvg. FedMuon and FedMuon-BC use a learning rate
of $0.03$, with $\beta=0.3$ for FedMuon and $\alpha=0.1$ for FedMuon-BC.
FedMuon-LGA uses a learning rate of $0.01$, with $\alpha=0.5$ and
$\beta=0.8$. ComFedL and Local-SCGDM use a learning rate of $0.2$, while
Local-SCGDM additionally uses $\alpha=0.2$ and $\gamma=0.1$. For FedCoMuon,
we set the learning rate to $0.02$ and use $\alpha=\beta=0.2$. For
FedCoMuon-VR, we set the learning rate to $0.03$, with $\alpha=0.3$,
$\beta=0.8$, $\gamma=0.6$, and $\rho=0.2$.

\subsection{Task-Distributed Meta Learning}
\label{app:maml-settings}

\subsubsection{CNN-Based Meta Learning}
\label{app:maml-cnn-settings}

The architecture of the 7-layer CNN used for CIFAR-10 is summarized in
Table~\ref{tab:seven-layer-cnn}.

\begin{table}[H]
	\centering
	\begin{tabular}{lcccc}
		\toprule
		\textbf{Layer Type} & \textbf{Output Size} & \textbf{Kernel Size}
		& \textbf{Stride} & \textbf{Activation} \\
		\midrule
		Input       & $32\times32\times3$    & --         & -- & --      \\
		Convolution & $30\times30\times96$   & $3\times3$ & 1  & ReLU    \\
		Convolution & $14\times14\times96$   & $3\times3$ & 2  & ReLU    \\
		Convolution & $14\times14\times196$  & $1\times1$ & 1  & ReLU    \\
		Convolution & $14\times14\times10$   & $1\times1$ & 1  & ReLU    \\
		Flatten     & $1{,}960$              & --         & -- & --      \\
		Dense       & $1{,}000$              & --         & -- & ReLU    \\
		Dense       & $1{,}000$              & --         & -- & ReLU    \\
		Output      & $10$                   & --         & -- & --      \\
		\bottomrule
	\end{tabular}
	\caption{Architecture of the 7-layer CNN used for CIFAR-10.}
	\label{tab:seven-layer-cnn}
\end{table}

Both the inner and outer learning rates are searched over
$\{0.005,0.01,0.02,0.03,0.05,0.1,0.2,0.5,1\}$.
We set the batch size to 64 for all algorithms. We search $\lambda$ over
$\{0.2,0.5,1,2\}$ and select $\lambda=0.5$. The same hyperparameter search
spaces are used for all heterogeneity levels.

For the 7-layer CNN experiments on CIFAR-10, the outer learning rate is set
to $0.1$ for all methods. FedMAML uses an inner learning rate of $0.03$ for
$\chi=0.3$ and $0.5$, and $0.05$ for $\chi=0.7$. FedMuon, FedMuon-LGA,
FedMuon-BC, and ComFedL use an inner learning rate of $0.05$, whereas
Local-SCGDM, FedCoMuon, and FedCoMuon-VR use $0.01$, $0.03$, and $0.03$,
respectively. For FedMuon, we set $\beta=0.4$. For FedMuon-LGA, we set
$\alpha=0.5$ and $\beta=0.8$, while FedMuon-BC uses $\alpha=0.1$.
Local-SCGDM uses $\alpha=\gamma=0.9$. For FedCoMuon, we set
$\alpha=\beta=0.7$ across all heterogeneity settings. For FedCoMuon-VR,
we set $\alpha=0.1$, $\beta=0.8$, and $\rho=0.2$ for all heterogeneity
settings, while $\gamma$ is set to $0.7$, $0.7$, and $0.9$ for
$\chi=0.3$, $0.5$, and $0.7$, respectively.

\subsubsection{ViT-Tiny-Based Meta Learning}
\label{app:maml-vit-settings}

For ViT-Tiny-based meta learning, we conduct experiments on CIFAR-10 using a
ViT-Tiny model with 12 Transformer blocks, a hidden dimension of 192, three
attention heads, and a patch size of 4. We adopt the same federated setting
and dominant-class data partition as in the CNN experiments and set
$\chi=0.3$.

We set the batch size to 32 for all algorithms.
Both the inner and outer learning rates are searched over
$\{0.001,0.003,0.005,0.01,0.02,0.05,0.1,0.2,0.5\}$.
FedMAML uses inner and outer learning rates of $0.05$ and $0.1$, respectively.
FedMuon, FedMuon-LGA, and FedMuon-BC use an inner learning rate of $0.005$
and an outer learning rate of $0.01$. We set $\beta=0.3$ for FedMuon,
$\alpha=0.5$ and $\beta=0.8$ for FedMuon-LGA, and $\alpha=0.1$ for
FedMuon-BC. ComFedL uses inner and outer learning rates of $0.003$ and $0.1$,
respectively. Local-SCGDM uses inner and outer learning rates of $0.003$ and
$0.05$, with $\alpha=0.3$ and $\gamma=0.9$. FedCoMuon uses inner and outer
learning rates of $0.005$ and $0.01$, with $\alpha=0.9$ and $\beta=0.3$.
FedCoMuon-VR also uses inner and outer learning rates of $0.005$ and $0.01$,
respectively, with $\alpha=0.8$, $\beta=0.8$, $\gamma=0.9$, and
$\rho=0.2$.

\end{document}